\documentclass[journal]{IEEEtran}
\usepackage{amsmath,amsfonts,amssymb,mathrsfs,amsthm}
	\usepackage{bm}
	\usepackage{bbold}
	\usepackage{pgfplots}
	\usepackage{tikz}
	\usepackage{hyperref}
	\usepackage{epstopdf}
	\pgfplotsset{compat=newest}
	
	\include{new_commands}
	\newcommand{\tr}{\mathop{\rm tr}}
	\newcommand{\E}{\mathop{\mathbb{E}}}
	\newcommand{\diag}{\mathop{\rm diag}}

	\newcommand{\asto}{\to_{a.s.}}
	\newcommand{\probto}{\to_{prob.}}
	
	\newtheorem{assumption}{{\bf Assumption.}}
	\newtheorem{theorem}{{\bf Theorem.}}
	\newtheorem{remark}{{\bf Remark.}}
	\newtheorem{corollary}{{\bf Corollary.}}
	\newtheorem{proposition}{{\bf Proposition.}}
	\newtheorem{lemma}{{\bf Lemma.}}
\makeatletter
\let\l@ENGLISH\l@english
\makeatother

\begin{document}%
\title{A Large Dimensional Study of Regularized Discriminant Analysis}
\author{Khalil Elkhalil,~\IEEEmembership{Student Member,~IEEE,} Abla Kammoun,~\IEEEmembership{Member,~IEEE}, Romain Couillet,~\IEEEmembership{Senior Member,~IEEE},  Tareq~Y.~Al-Naffouri,~\IEEEmembership{Member,~IEEE}, and Mohamed-Slim Alouini,~\IEEEmembership{Fellow,~IEEE}

\thanks{
	Part of this work related to the derivation of an asymptotic equivalent for the R-QDA probability of misclassification has been accepted for publication in the IEEE MLSP workshop 2017 \cite{khalil_mlsp}. \par K. Elkhalil, A. Kammoun, T. Y. Al-Naffouri and M.-S. Alouini are with the Electrical Engineering Program, King Abdullah University of Science and Technology, Thuwal, Saudi Arabia; e-mails: \{khalil.elkhalil, abla.kammoun, tareq.alnaffouri, slim.alouini\}@kaust.edu.sa.\par  R. Couillet is with the CentraleSupélec, Paris, 91190 Châtenay-Malabry,
	France; e-mail: romain.couillet@centralesupelec.fr.
}
%\thanks{Mohammed Eltayeb and Hamid Reza Bahrami are with the Department
%of Electrical and Computer Engineering, The University of Akron, Ohio, USA; e-mails: \{me33, hrb\}@uakron.edu. Khalil Elkhalil and Tareq Y. Al-Naffouri are with the Electrical Engineering Department, King Abdullah University of Science and Technology, Thuwal, Saudi Arabia; e-mails: \{khalil.elkhalil, tareq.alnaffouri\}@kaust.edu.sa. Tareq Y. Al-Naffouri is also associated with the Department of Electrical Engineering, King Fahd University of Petroleum
%and Minerals, Dhahran 31261, Kingdom of Saudi Arabia.}% <-this % stops 

}

\maketitle
%\markboth{IEEE TRANSACTIONS ON WIRELESS COMMUNICATIONS ,~Vol.~14, No.~8, August~2016}%
\vspace{-15mm}
\begin{abstract}
    In this paper, we conduct a large dimensional study of regularized discriminant analysis classifiers with its two popular variants known as regularized LDA and regularized QDA. The analysis is based on the assumption that the data samples are drawn from a Gaussian mixture model with different means and covariances and relies on tools from random matrix theory (RMT). We consider the regime in which both the data dimension and training size within each class
    tends to infinity with fixed ratio. Under mild assumptions, we show that the probability of misclassification converges to a deterministic quantity that describes in closed form the performance of these classifiers in terms of the class statistics as well as the problem dimension. The result allows for a better understanding of the underlying classification algorithms in terms of their performances in practical large
    but finite dimensions. Further exploitation of the results permits to optimally tune the regularization parameter with the aim of minimizing the probability of misclassification. The analysis is validated with numerical results involving synthetic as well as real data from the USPS dataset yielding a high accuracy in predicting the performances and hence making an interesting connection between theory and practice. 
\end{abstract}
\section{Introduction}
\label{introduction}

Linear Discriminant analysis (LDA) is an old concept that dates back to Fisher that generalizes the Fisher discriminant \cite{stat_patern_recog, pattern_classification}. Given two statistically defined datasets, or classes, the Fisher discriminant analysis is designed to maximize the ratio of the variance between classes to the variance within classes and is useful for both classification and dimensionality reduction \cite{bishop,tibshirani}. LDA, on the other hand,
relying merely on the concept of model based
classification  \cite{bishop}, is conceived so that the misclassification rate is minimized under a Gaussian assumption for the data. Interestingly, both ideas lead to the same
classifier when the data of both classes share the same covariance matrix. Maintaining the Gaussian assumption but considering the general case of distinct covariance matrices, quadratic discriminant analysis (QDA) becomes the optimal classifier in terms of the minimization of the misclassification rate when both statistical means and covariances of the classes are known. \par 
In practice,  these parameters are rarely given and only estimated based on training data.
%
%can be learnt from the training data set for which the class label associated with each observation  is assumed  known.  
Assuming the number of training samples is high enough, QDA and LDA should remain asymptotically optimal.
% in terms of the minimization
%of the misclassification rate. % as the maximum likelihood estimated from the available training data set.   
% Such a finding is, however, ensured only when the number of training samples is quite larger than the dimensions of the observations. This becomes all the more problematic, 
It is however often the case in practice that the data dimension is large, if not larger, than the number of observations. In such circumstances, the covariance matrix estimate becomes ill-conditioned or even non invertible, which leads to poor classification performance. 
%Based on the previous discussion on the design of LDA and QDA, it is clear that both classifiers are designed relying on the entire knowledge of the data statistics that can be summarized in the statistical means and the covariance matrices for the Gaussian distribution. From a practical standpoint, these parameters can be learned using maximum likelihood estimation from the available training set as it is the case in any supervised learning setting. However, it is quite known that this learning procedure is only effective as long as the number of training samples is quite larger than the data' dimension since the estimated statistics would converge in some sens to the true ones. In the general case, the training size might be even less than the data dimension resulting in the covariance matrix estimate being ill-conditioned leading to poor classification performance.

To overcome this difficulty, many techniques can be considered. One can resort to dimensionality reduction so as to embed the data in a low-dimensional space that retains most of the useful information from a classification point of view \cite{ghodsi_tuto, survey_dim_reduction}. This ensures a higher number of training samples than the effective data size. Questions as to which dimensions
to be selected or to what extent dimension should be reduced remain open. Another alternative involves the regularized versions of LDA and QDA denoted, throughout this paper, by R-LDA and R-QDA \cite{tibshirani,zollanvari}. Both approaches constitute the main focus of the article.  
% To overcome this limitation, many techniques can be considered: for instance one can resort to dimensionalty reduction techniques \citep{ghodsi_tuto, survey_dim_reduction} in the aim to embed the data in a low-dimensional space retaining most of the \emph{useful} information. By doing so, we enjoy a bigger number of training samples and as such we would have better estimates of the statistics. However, dimensionality reduction techniques are effective only if the following fundamental questions are answered: what is the \emph{optimal} dimension of the projected space? For a given dimension of the projected space which subspace should we project on?. A close study of the literature behind theses techniques reveals that answering these questions is not an easy task and it is as such out of the scope of this paper. Another way to deal with the low training size is to consider regularized versions of LDA and QDA respectively denoted by R-LDA and R-QDA throughout the paper \cite{tibshirani, zollanvari}. This constitutes the main focus of the current paper. \par 
\par 
There exist many works on the performance analysis of discriminant analysis classifiers. In \cite{qda_stoachastic1}, an exact analysis of QDA is made by relying on properties of Wishart matrices. This allows for exact expressions of the
probability misclassification rate for all sample size $n$ and dimension $p$. 
%Many approaches have been deployed in order to carry out the performance analysis of such classifiers. One interesting approach is by McFarland  where an exact analysis has been made relying on some fundamental properties of Wishart matrices \citep{qda_stoachastic1}. The most interesting feature of the analysis is that it permit to derive exact expressions of the probability of misclassifcation. 
This analysis is however only valid as long as $n \geq p$ . Generalizing this analysis to regularized versions is however beyond analytical reach. This motivated further studies to consider asymptotic regimes. In \cite{slda, sqda} the authors consider the large $p$ asymptotics and observe that LDA and QDA fall short even when the exact covariance matrix is known. \cite{slda} thus proposed improved LDA and PCA that exploit sparsity assumptions on the statistical means and covariances, however not necessarily met in practice. This leads us to consider in the present work the \emph{double asymptotic regime} in which both $p$ and $n$ tend to infinity with fixed ratio. 
%More precisely, under the assumption that $p \gg n$ and relying on the sparsity assumption of the statistical means and covariance matrices, the asymptotic peformance of LDA and QDA has been conducted \cite{}. This asymptotic regime is interesting because it treats the case of high dimensional data and permits to encompass the difficulty of low sample size. However, the technical assumptions that  have been made are somewhat strong, even unrealistic, especially the assumption of the sparse statistics. Also, the asymptotic regime $p \gg n$ does not permit to treat scenarios where $p$ and $n$ are of the same order of magnitude with either $p > n$ or $p < n$. This lead us to consider what is known as the \emph{double asymptotic regime} or the \emph{random matrix theory regime} in which both dimensions $p$ and $n$ tends to infinity at the same pace, i.e., $p/n \to constant$. 
This regime leverages results from random matrix theory \cite{girko, silverstein, walid_new, couillet, gram_mixture}. For LDA analysis, this regime was first considered in \cite{raudys} under the assumption of equal covariance matrices. It was  extended  to the analysis of R-LDA in \cite{zollanvari} and to the Euclidean distance discriminant rule in \cite{euclidean_DA}. To the best of the authors' knowledge, the general case in which the covariances across
classes are different was never treated. As shown in the course of the paper, a major difficulty for the analysis resides in choosing the   assumptions governing the growth rate of means and covariances to avoid nontrivial asymptotic classification performances. 
\par 
This motivates the present work. Particularly, we  propose a large dimensional analysis of both R-LDA and R-QDA in the double asymptotic regime for general Gaussian assumptions. Precisely, under technical, yet mild, assumptions controlling the distances between the class means and covariances, we prove that the probability of misclassification converges to a non-trivial deterministic quantity that only depends on the class statistics as well as the ratio $p/n$. 
Interestingly, R-LDA and R-QDA require different growth regimes, reflecting a fundamental difference in the way they leverage the information about the means and covariances. Notably, R-QDA requires a minimal distance between class means of order $O(\sqrt{p})$ while R-LDA necessitates a difference in means of  order $O(1)$. However, R-LDA does not seem to leverage the information about the distance in covariance matrices. The results of \cite{zollanvari} are in particular recovered when the spectral norm of the difference of covariance matrices is small. These findings lead to  insights into when LDA or QDA should be preferred in practical scenarios. 
%In particular, our findings show some counter-intuitive insights on the asymptotic behavior of R-LDA and R-QDA. Surprisingly, we show that with the same assumptions on the distance between the means considered for example in \cite{zollanvari}, R-QDA is asymptotically independent of this distance which means that the distance between the covariances represent the discriminant metric that assymptotically matters. To see the impact of the means on the performance of R-QDA, one would relax the previous assumption on the distance between the means allowing for larger order of magnitude. As we will show in the paper, this permits to see a tangible effect of the means on the R-QDA performance, however it results in R-LDA achieving asymptotic perfect classification since the classes are well separated.
%This can be understood as R-QDA being subject to strong noise induced by the fact that R-QDA estimates individual covariance matrices each with a fraction of the total sample size as opposed to R-LDA which conceptually assumes a common covariance matrix and thus uses the whole samples to estimate one covariance matrix, namely the pooled covariance matrix. Another important insight that we infer from the analysis is that in order to see the role of covariance matrices in the classification task, the mutual distance between the covariances should scale larger than a certain order of magnitude that we fully characterize. This is particularly means for example that if these conditions are not satisfied we would basically recover previous results as in \cite{zollanvari}.
\par
To sum up, our main results are as follows: 
\begin{itemize}
	\item Under mild assumptions, we establish the
	convergence of the misclassification rate for both R-LDA and R-QDA classifiers to a
	deterministic error as a function of the statistical parameters associated with each class. 
	\item  We design a consistent estimator for the misclassification rate for both R-LDA and R-QDA classifiers allowing for a pre-estimate of the optimal regularization parameter.
	\item We validate our theoretical findings on both synthetic and real data drawn
	from the USPS dataset and illustrate the good accuracy of our results in both
	settings.
\end{itemize}
The remainder is organized as follows.  We give an overview of discriminant analysis for binary classification in Section \ref{section:binary_classification}. The main results are presented in Section \ref{results}, the proofs of which are deferred to the Appendix. In Section \ref{section:G_estimator}, we design a consistent estimator of the misclassification error rate. We validate our analysis for real data in Section \ref{section:experiments} and conclude the article in Section \ref{section:conclusion}. \\ 
\textbf{Notations:} 
Scalars, vectors and matrices are respectively denoted by non-boldface, boldface lowercase and boldface uppercase characters. $\mathbf{0}_{p\times n}$ and $\mathbf{1}_{p\times n}$ are respectively the matrix of zeros and ones of size $p \times n$, $\mathbf{I}_p$ denotes the $p\times p$ identity matrix. The notation $\|.\|$ stands for the Euclidean norm for vectors and the spectral norm for matrices. $\left(.\right)^T$, $\tr\left(.\right)$ and $|.|$ stands for the transpose, the trace and the determinant of a matrix respectively. For two functionals $f$ and $g$, we say that $f = O\left(g\right)$, if $\exists \:0<M < \infty$ such that $|f| \leq M g$. $\mathbb{P}\left(.\right)$, $\to_d$, $\to_{prob.}$ and $\to_{a.s.}$ respectively denote the probability measure, the convergence in distribution, the convergence in probability and the almost sure convergence of random variables. $\Phi\left(.\right)$ denotes the cumulative density function (CDF) of the standard normal distribution, i.e. $\Phi\left(x\right) = \int_{-\infty}^{x} \frac{1}{\sqrt{2\pi}}e^{-\frac{t^2}{2}}dt$.
%Prior works, for example \cite{zollanvari} studied the performance of the LDA classifier in the asymptotic regime. Our analysis is mainly based on fundamental results on the convergence of certain functionals of random matrices \cite{walid_new} \\
%to the best of our knowledge this is the first paper that establishes the convergence of QDA in the double asymptotic regime with fewer assumptions on $\bm{\mu}_i$ and $\mathbf{\Sigma}_i$.
%	Probabilistic inference has become a core technology in AI,
%	largely due to developments in graph-theoretic methods for the
%	representation and manipulation of complex probability
%	distributions~\citep{pearl:88}.  Whether in their guise as
%	directed graphs (Bayesian networks) or as undirected graphs (Markov
%	random fields), \emph{probabilistic graphical models} have a number
%	of virtues as representations of uncertainty and as inference engines.
%	Graphical models allow a separation between qualitative, structural
%	aspects of uncertain knowledge and the quantitative, parametric aspects
%	of uncertainty...\\
%
%	{\noindent \em Remainder omitted in this sample. See http://www.jmlr.org/papers/ for full paper.}

% Acknowledgements should go at the end, before appendices and references

% Manual newpage inserted to improve layout of sample file - not
% needed in general before appendices/bibliography.

\section{Discriminant Analysis for Binary Classification}
\label{section:binary_classification}
This paper studies  binary discriminant analysis techniques which employs a discriminant rule to assign for an input data vector the class to which it most likely belongs. The discriminant rule is designed based on $n$ available training data with known class labels. In this paper, we consider the case in which a bayesian discriminant rule is employed. Hence, we assume that observations from class $\mathcal{C}_i$, $i \in \{0,1\}$ are independent and are sampled from a multivariate Gaussian distribution with mean
$\bm{\mu}_i \in \mathbb{R}^{p \times 1}$ and non-negative covariance matrix $\mathbf{\Sigma}_i \in \mathbb{R}^{p \times p}$.  Formally speaking, an observation vector  $\mathbf{x} \in \mathbb{R}^{p \times p}$ is classified to $\mathcal{C}_i$, $i \in \{0,1\}$, if %Of interest in this paper is the case where a bayesian discriminant rule is employed.   %decide whether an input data vector belongs to  class $\mathcal{C}_0$ or class $\mathcal{C}_1$. 
%In this paper, discriminant analysis (DA) is studied for binary classification where an input data is either belongs to class $\mathcal{C}_0$ or class $\mathcal{C}_1$. As stated in the introduction, in DA, we assume a certain model for a class where for particular interest in this paper, we assume that a class $\mathcal{C}_i$, $i \in \{0,1\}$ is modeled by a multivariate Gaussian distribution with a given mean $\bm{\mu}_i \in \mathbb{R}^{p \times 1}$ and a non-negative definite  covariance matrix $\mathbf{\Sigma}_i \in \mathbb{R}^{p \times p}$. More precisely, let $\mathbf{x} \in \mathbb{R}^{p \times p}$ be an input data containing $p$ features, then $\mathbf{x}$ is classified to $\mathcal{C}_i$, $i\in \{0,1\}$, if $\exists$ $\bm{\omega} \sim \mathcal{N}\left(\mathbf{0},\mathbf{I}_p\right)$ such that
\begin{equation}
\label{gaussian_model}
\mathbf{x} = \bm{\mu}_i + \bm{\Sigma}_i^{1/2} \bm{z}, \quad \bm{z}\sim \mathcal{N}\left(\mathbf{0},\mathbf{I}_p\right).
\end{equation}
Let \begin{equation} 
\begin{split}
\label{discriminant_qda}
& W^{QDA}\left(\mathbf{x}\right)   = -\frac{1}{2}\log \frac{|\mathbf{\Sigma}_0|}{|\mathbf{\Sigma}_1|}-\frac{1}{2} {\bf x}^{T}\left(\mathbf{\Sigma}_0^{-1}-\mathbf{\Sigma}_1^{-1}\right){\bf x} \\ & +{\bf x}^{T}\mathbf{\Sigma}_0^{-1}\boldsymbol{\mu}_0  -{\bf
	x}^{T}\mathbf{\Sigma}_1^{-1}\boldsymbol{\mu}_1
-\frac{1}{2}\boldsymbol{\mu}_0^{T}\mathbf{\Sigma}_0^{-1}\boldsymbol{\mu}_0+\frac{1}{2}\boldsymbol{\mu}_1^{T}\mathbf{\Sigma}_1^{-1}\boldsymbol{\mu}_1 \\ & -\log\frac{\pi_1}{\pi_0}.%\left(\mathbf{x}-\bm{\mu}_i\right)^T\mathbf{\Sigma}_i^{-1}\left(\mathbf{x}-\bm{\mu}_i\right) + \log \pi_i,
\end{split}
\end{equation}
As stated in \cite{mclachlan}, for distinct covariance matrices $\mathbf{
	\Sigma}_0$ and $\mathbf{\Sigma}_1$, the discriminant rule is summarized as follows %The corresponding classification rule is given by
\begin{align}
\label{qda_classification_rule}
\left\{\begin{matrix}
\mathbf{x} \in \mathcal{C}_0 & \: \quad\quad\quad\quad \textnormal{if} \quad W^{QDA} \left(\mathbf{x}\right)>0. \\
\mathbf{x} \in \mathcal{C}_1& \textnormal{otherwise}.
\end{matrix}\right.
\end{align}
%	In this paper, discriminant analysis is studied for binary classification where each input data either belongs to class $\mathcal{C}_0$ or class $\mathcal{C}_1$. We assume that class $\mathcal{C}_i$, $i \in \{0,1\}$ is modeled by a multivariate Gaussian distribution with a given mean $\bm{\mu}_i \in \mathbb{R}^{p \times 1}$ and a non-negative definite  covariance matrix $\mathbf{\Sigma}_i \in \mathbb{R}^{p \times p}$. More precisely, letting $\mathbf{x} \in \mathbb{R}^{p \times 1}$ be an input data, $\mathbf{x}$ is classified to $\mathcal{C}_i$, $i\in \{0,1\}$, if
%		\begin{equation}
%		\label{gaussian_model}
%		\mathbf{x} = \bm{\mu}_i + \bm{\Sigma}_i^{1/2} \bm{\omega}, \: \text{for} \: \: \bm{\omega} \sim \mathcal{N}\left(\mathbf{0},\mathbf{I}_p\right).
%		\end{equation}
When the considered classes have the same covariance matrix, i.e., $\mathbf{\Sigma}_0=\mathbf{\Sigma}_1$, the discriminant  function simplifies to \cite{tibshirani,bishop,zollanvari}
%	As stated in \citep{mclachlan}, for distinct covariance matrices $\mathbf{\Sigma}_0$ and $\mathbf{\Sigma}_1$, denoting
%	\begin{align}
%	\label{discriminant_qda}
%	W^{QDA}_i\left(\mathbf{x}\right) = -\frac{1}{2}\log |\mathbf{\Sigma}_i|-\frac{1}{2} \left(\mathbf{x}-\bm{\mu}_i\right)^T\mathbf{\Sigma}_i^{-1}\left(\mathbf{x}-\bm{\mu}_i\right) + \log \pi_i,
%	\end{align}
%	where $\pi_i$ is the class' prior probability for $i \in \{0,1\}$, the corresponding maximum likelihood classification rule is given by
%	\begin{align}
%	\label{qda_classification_rule}
%	\left\{\begin{matrix}
%	\mathbf{x} \in \mathcal{C}_0 & \: \quad\quad\quad\quad \text{if} \quad W^{QDA}_0\left(\mathbf{x}\right) > W^{QDA}_1\left(\mathbf{x}\right) \\
%	\mathbf{x} \in \mathcal{C}_1&  \text{otherwise}.
%	\end{matrix}\right.
%	\end{align}
%	If both classes have a common covariance matrix $\mathbf{\Sigma}$, then the discriminant $W^{QDA}_0\left(\mathbf{x}\right)  - W^{QDA}_1\left(\mathbf{x}\right) $ reduces to the LDA's discriminant function given by \citep{tibshirani,bishop,zollanvari}
\begin{align}
\label{discriminant_lda}
W^{LDA}\left(\mathbf{x}\right) = \left(\mathbf{x}-\frac{\bm{\mu}_0+\bm{\mu}_1}{2}\right)^T \mathbf{\Sigma}^{-1}\left(\bm{\mu}_0-\bm{\mu}_1\right) - \log\frac{\pi_1}{\pi_0}.
\end{align}
Classification obeys hence the following rule:
%	Similarly, the classification rule for LDA is determined by
\begin{align}
\label{lda_classification_rule}
\left\{\begin{matrix}
\mathbf{x} \in \mathcal{C}_0 & \: \quad\quad\quad\quad \text{if} \quad  W^{LDA}\left(\mathbf{x}\right) > 0 \\
\mathbf{x} \in \mathcal{C}_1& \textnormal{otherwise}.
\end{matrix}\right.
\end{align}
Since $W^{LDA}$ is linear in ${\bf x}$, the corresponding classification method is referred to as linear discriminant analysis.   
As can be seen from  (\ref{qda_classification_rule}) and (\ref{lda_classification_rule}), the classification rules assume  the knowedge of the class statistics, namely their associated covariance matrices and  mean vectors. In practice, these statistics can be estimated using the available training data.   
%A look at the classification rules in (\ref{qda_classification_rule}), (\ref{lda_classification_rule}) reveals that they inherently assume  the knowedge of the class statistics, namely their associated covariance matrices and  mean vectors. In practice, these statistics can be estimated using the available training data.   
%	Clearly from (\ref{qda_classification_rule}), (\ref{lda_classification_rule}), the classification rules for both LDA and QDA are based on the knowledge of the class' statistics. Therefore, such statistics have to be estimated in the training period. 
As such,  we assume that $n_i$, $i \in \{0,1\}$ independent training samples $\mathcal{T}_0=\left \{ \mathbf{x}_l \in \mathcal{C}_0 \right \}_{l=1}^{n_0}$ and $\mathcal{T}_1=\left \{ \mathbf{x}_l \in \mathcal{C}_1 \right \}_{l=n_0+1}^{n_0+n_1=n}$
% where $n_0=n_1=n$\footnote{This assumption is for technical reasons to guarantee convergence of certain random quantities as will be shown in the course of the paper.}
are respectively available to estimate the mean and the covariance matrix of each class $i$\footnote{We assume that $\frac{n_i}{n_0+n_1} \to \pi_i$ for $i \in \{0,1\}$ which is valid under random sapmling. Therefore, we do not consider the problem of separate sampling when $n_0$ and $n_1$ are not chosen following the priors $\pi_0$ and $\pi_1$.}. For that, we consider the following sample estimates
\begin{align*}
\hat{\boldsymbol{\mu}}_i&  = \frac{1}{n_i} \sum_{l \in \mathcal{T}_i} \mathbf{x}_l,   \quad i \in \{0,1\}\\
\widehat{\mathbf{\Sigma}}_i & = \frac{1}{n_i-1} \sum_{l \in \mathcal{T}_i} \left(\mathbf{x}_l-\hat{\boldsymbol{\mu}}_i\right)\left(\mathbf{x}_l-\hat{\boldsymbol{\mu}}_i\right)^T,   \quad i \in \{0,1\}\\
\widehat{\mathbf{\Sigma}}&= \frac{\left(n_0-1\right)\widehat{\mathbf{\Sigma}}_0 +\left(n_1-1\right)\widehat{\mathbf{\Sigma}}_1 }{n-2},
\end{align*}
where $\widehat{\mathbf{\Sigma}}$ is the pooled sample covariance matrix for both classes.
To avoid singularity issues when $n_i < p$, we use the  ridge estimator of the inverse of the covariance matrix  \cite{tibshirani}   %more robust estimator of the inverse of the covariance matrix known as the ridge estimator of the covariance matrix given by \citep{tibshirani}
\begin{align}
\mathbf{H} & = \left(\mathbf{I}_p + \gamma \widehat{\mathbf{\Sigma}}\right)^{-1} \label{eq:H} ,\\
\mathbf{H}_i & = \left(\mathbf{I}_p + \gamma \widehat{\mathbf{\Sigma}}_i\right)^{-1},\label{eq:Hi} \quad i \in \{0,1\}
\end{align}
where $\gamma > 0$ is a regularization parameter. Replacing $\boldsymbol{\Sigma}^{-1}$ and $\boldsymbol{\Sigma}_i$ for $i\in\{0,1\}$ by  \eqref{eq:H} and \eqref{eq:Hi} into \eqref{discriminant_lda} and \eqref{qda_classification_rule}, we obtain the following discriminant rules
\begin{equation}
\label{est_discriminant_lda}
\widehat{W}^{R-LDA}\left(\mathbf{x}\right) = \left(\mathbf{x}-\frac{\hat{\boldsymbol{\mu}}_0+\hat{\boldsymbol{\mu}}_1}{2}\right)^T \mathbf{H}\left(\hat{\boldsymbol{\mu}}_0-\hat{\boldsymbol{\mu}}_1\right) - \log\frac{\pi_1}{\pi_0}.
\end{equation}
\begin{equation}
\label{est_discriminant_qda}
\begin{split}
\widehat{W}^{R-QDA}\left(\mathbf{x}\right) & = \frac{1}{2}\log \frac{|{\bf H}_0|}{|{\bf H}_1|}-\frac{1}{2}\left({\bf x}-\hat{\boldsymbol{\mu}}_0\right)^{T}{\bf H}_0\left({\bf x}-\hat{\boldsymbol{\mu}}_0\right) \\ & +\frac{1}{2}\left({\bf x}-\hat{\boldsymbol{\mu}}_1\right)^{T}{\bf H}_1\left({\bf x}-\hat{\boldsymbol{\mu}}_1\right)-\log\frac{\pi_1}{\pi_0}.
\end{split}
%\frac{1}{2}\log |\mathbf{H}_i|-\frac{1}{2} \left(\mathbf{x}-\overline{\mathbf{x}}_i\right)^T\mathbf{H}_i\left(\mathbf{x}-\overline{\mathbf{x}}_i\right) + \log \pi_i.
\end{equation}
%        where  the subscript ${R}$ is used above in reference to the employed regularization parameter $\gamma$. 
The corresponding classification methods will be denoted respectively by R-LDA and R-QDA.   
Conditioned on the training samples $\mathcal{T}_i$, $i \in \{0,1\}$, the classification errors associated with R-LDA and R-QDA when ${\bf x}$ belongs to class  $\mathcal{C}_i$ are given by %the classification error contributed by class $\mathcal{C}_i$ is computed respectively as
\begin{align}
\label{lda_conditional_error}
\epsilon_i^{R-LDA}& = \mathbb{P}\left[\left(-1\right)^i\widehat{W}^{R-LDA}\left(\mathbf{x}\right) < 0 \: | \mathbf{x} \in \mathcal{C}_i, \mathcal{T}_0,\mathcal{T}_1\right]. \\
\label{qda_conditional_error}
\epsilon_i^{R-QDA} &= \mathbb{P}\left[\left(-1\right)^i\widehat{W}^{R-QDA}({\bf x})<0 \:| \mathbf{x} \in \mathcal{C}_i,\mathcal{T}_0,\mathcal{T}_1 \right].
\end{align}
The total classification errors are respectively given by
\begin{align*}
\epsilon^{R-LDA}&=\pi_0\epsilon_0^{R-LDA}+\pi_1\epsilon_1^{R-LDA}.\\
\epsilon^{R-QDA}&=\pi_0\epsilon_0^{R-QDA}+\pi_1\epsilon_1^{R-QDA}.
\end{align*}
In the following, we propose to analyze the asymptotic classication errors of both R-LDA and R-QDA  when $p,n_i$ grow large at the same rate. For R-LDA, our results cover a more general setting than the one studied in \cite{zollanvari}, in that they apply to the case where both classes have distinct covariance matrices.%For LDA, we consider a more general setting than that studied in \citep{zollanvari}, where we allow both classes to have distinct covariance matrices.
\section{Main Results}
\label{results}
The main contributions of the present work are two fold. First, we carry out an asymptotic analysis of the classification error rate for both R-LDA and R-QDA, showing that they converge to some deterministic quantities that depend solely on the observations statistics associated with each class. Such a result allows a better understanding of the impact of these parameters on the performances. Second, we build consistent estimates of the asymptotic misclassification error rates for both
estimators. An estimator of the misclassification error rate has been provided in \cite{zollanvari} but for the R-LDA when the classes are assumed to have identical covariance matrices. Our results regarding R-LDA in this respect extends the one in \cite{zollanvari} when the covariance matrices are not equal. The treatment of R-QDA is however new and constitute the main contribution of the present work.         %%In this section, we present the main contributions of the article. In a first part, we carry out asymptotic analysis of the classification error for both R-LDA and R-QDA. More precisely, under some mild assumptions, we show that the classification errors converge to some deterministic quantities that depend solely on the observations' statistics, namely the mean and covariances within each class. The interest of such results lie in that they establish a mathematical relation

\subsection{Asymptotic Performance of R-LDA with Distinct Covariance Matrices}
\label{LDA_distinct}
In this section, we present an asymptotic analysis of the R-LDA classifier. Our analysis is mainly based on recent results from RMT concerning some properties of Gram matrices of mixture models \cite{gram_mixture}. We recall that \cite{zollanvari} made a similar analysis of R-LDA in the double asymptotic regime when both classes have a common covariance matrix, thereby not requiring these advanced tools. As such, our results can be viewed as a generalization of
\cite{zollanvari} when both classes have distinct covariance matrices. This permits to evaluate the performance of R-LDA in practical scenarios when the assumption of common covariance matrices cannot always be guaranteed. To allow derivations, we shall consider the following growth rate assumptions%Before deriving the deterministic equivalent for R-LDA, we consider the double asymptotic regime in which $n_i,$ $p$ $\to \infty$ for $i \in \{0,1\}$ such that the following assumptions hold
\begin{assumption}[Data scaling]
	\label{As:1L}
	$\frac{p}{n} \to c \in \left(0,\infty\right)$.
\end{assumption}
\begin{assumption}[Class scaling]
	\label{As:2L}
	$\frac{n_i}{n} \to c_i\in \left(0,\infty\right)$, \textnormal{for} $i \in \{0,1\}$.
\end{assumption}
%\begin{assumption}
%	\label{As:2}
%	$\frac{n_i}{n} \to c_i \in \left(0,\infty\right)$.
%\end{assumption}
\begin{assumption}[Covariance scaling]
	\label{As:4L}
	$\lim\sup_{p}\left \| \mathbf{\Sigma}_i \right \|<\infty$, \textnormal{for} $i \in \{0,1\}$.
\end{assumption}		
\begin{assumption}[Mean scaling]
	\label{As:3L} Let $\boldsymbol{\mu}=\bm{\mu}_0 - \bm{\mu}_1$. Then,
	$\lim\sup_p \|\boldsymbol{\mu}\|=\lim\sup_p\| \bm{\mu}_0 - \bm{\mu}_1\| <\infty$.
\end{assumption}
These assumptions are mainly considered to achieve an asymptotically non-trivial classification error. Assumption \ref{As:4L} is frequently met within the framework of random matrix theory \cite{gram_mixture}. Under the setting of Assumption \ref{As:4L},  Assumption~\ref{As:3L} ensures that a nontrivial classification rate is obtained: If $\left \| \bm{\mu}_0 - \bm{\mu}_1 \right \|$ scales faster than $O\left(1\right)$, then perfect asymptotic classification is achieved; however, if
$\left \| \bm{\mu}_0 - \bm{\mu}_1 \right \|$ scales slower than $O\left(1\right)$, classification is asymptotically impossible. Assumptions \ref{As:1L} and \ref{As:2L}
respectively control the growth rate in the data and the training. %Assumption \ref{As:4L} is essential to be able to exploit the machinery in \citep{gram_mixture} and also to guarantee a non-trivial misclassification rate.
%	\begin{assumption}
%		\label{As:5L}
%		$\lim \sup \frac{1}{\sqrt{p}}\tr \mathbf{A}\left(\mathbf{\Sigma}_0-\mathbf{\Sigma}_1\right)= \bigo\left(1\right)$, for all $\mathbf{A} \in \mathbb{R}^{p\times p}$ satisfying $\left\|\mathbf{A}\right\| = \bigo\left(1\right)$.
%	\end{assumption}
\subsubsection{Deterministic Equivalent}
We are in a position to derive a deterministic equivalent of the  misclassification error rate of the R-LDA. Indeed,
conditioned on the training data $\mathbf{x}_1,\cdots,\mathbf{x}_n$, the probability of misclassification is given by: \cite{zollanvari}
\begin{equation}
\label{error_lda}
\epsilon^{R-LDA}_i = \Phi\left(\frac{\left(-1\right)^{i+1}G\left(\bm{\mu}_i,\hat{\bm{\mu}}_0,\hat{{\bm{\mu}}}_1,\mathbf{H}\right) + \left(-1\right)^i \log \frac{\pi_1}{\pi_0}}  {\sqrt{D\left({\hat{\bm{\mu}}}_0,\hat{{\bm{\mu}}}_1,\mathbf{H},\mathbf{\Sigma}_i\right)}}   \right),
\end{equation}
where \begin{equation}
G\left(\bm{\mu}_i,\hat{{\bm{\mu}}}_0,\hat{{\bm{\mu}}}_1,\mathbf{H}\right)  = \left(\bm{\mu}_i-\frac{\hat{\bm{\mu}}_0+\hat{{\bm{\mu}}}_1}{2}\right)^T\mathbf{H}\left(\hat{{\bm{\mu}}}_0-\hat{{\bm{\mu}}}_1\right).
\end{equation}

\begin{equation}
D\left(\hat{{\bm{\mu}}}_0,\hat{{\bm{\mu}}}_1,\mathbf{H},\mathbf{\Sigma}_i\right)
= \left(\hat{{\bm{\mu}}}_0-\hat{{\bm{\mu}}}_1\right)^T \mathbf{H}\mathbf{\Sigma}_i\mathbf{H}\left(\hat{{\bm{\mu}}}_0-\hat{{\bm{\mu}}}_1\right).
\end{equation}
The total misclassification probability is thus given by
\begin{equation}
\epsilon^{R-LDA} = \pi_0 \epsilon^{R-LDA}_0 + \pi_1 \epsilon^{R-LDA}_1.
\end{equation}
Prior to stating the main result concerning R-LDA, we shall introduce the following quantities, which naturally appear, as a result of applying \cite{gram_mixture}. Let $\bar{\mathbf{Q}}\left(z\right)$ be the  matrix defined as follows
%	\begin{equation}
%		Q\left(z\right) = \left(\mathbf{Y}_0\mathbf{Y}_0^T + \mathbf{Y}_1\mathbf{Y}_1^T - z \mathbf{I}_p \right)^{-1}
%	\end{equation}
%	Then, $Q\left(z\right)$ is equivalent to a deterministic matrix $\bar{Q}\left(z\right)$ or $Q\left(z\right)\leftrightarrow \bar{Q}\left(z\right)$ in the sens that we have convergence in probability $\frac{1}{n} \tr \mathbf{M}\left(Q\left(z\right)-\bar{Q}\left(z\right)\right) \to 0$ and $\mathbf{u}^T \left(Q\left(z\right)-\bar{Q}\left(z\right)\right) \mathbf{v} \to 0$ for all deterministic matrices $\mathbf{M}$ of bounded spectral norms and all deterministic vectors $\mathbf{u}$ and $\mathbf{v}$ of bounded euclidean norms.
%
%	 $\bar{Q}\left(z\right)$ is given by
\begin{equation}
\bar{\mathbf{Q}}\left(z\right) \triangleq -\frac{1}{z} \left(\mathbf{I}_p+c_0g_0(z)\mathbf{\Sigma}_0+c_1g_1(z)\mathbf{\Sigma}_1\right)^{-1}, \: z \in \mathbb{C}.
\end{equation}
where
$g_i(z)$, $i\in \{0,1\}$, satisfies the following fixed point equations
\begin{equation}
\label{eq:fixed_point_LDA}
\frac{p}{n}g_i(z) = -\frac{1}{z} \frac{1}{1+\tilde{g}_i(z)}, \:\: \tilde{g}_i(z) =  \frac{1}{p} \tr \boldsymbol{\Sigma}_i \bar{\mathbf{Q}}\left(z\right).
\end{equation}
Also define ${\bf A}_i=\boldsymbol{\Sigma}_i\bar{\bf Q}(z)$ and $\widetilde{\mathbf{Q}}_i\left(z\right)$ as
\begin{equation}
\widetilde{\mathbf{Q}}_i\left(z\right) \triangleq \bar{\bf Q}(z)\left({\bf A}_i+R_0{\bf A}_0+R_1{\bf A}_1\right),%\bar{\mathbf{Q}}\left(z\right) \mathbf{\Sigma}_i \bar{\mathbf{Q}}\left(z\right) + R_{0}\bar{\mathbf{Q}}\left(z\right) \mathbf{\Sigma}_0 \bar{\mathbf{Q}}\left(z\right) + R_{1}\bar{\mathbf{Q}}\left(z\right) \mathbf{\Sigma}_1 \bar{\mathbf{Q}}\left(z\right), 
\label{eq:Qtildei}
\end{equation}
where
\begin{align}
R_{i} = \frac{ z^2 c_i g_i^2(z) \frac{1}{n} \tr {\bf A}_0{\bf A}_1}{1- \left(z^2 c_0 g_0^2(z) + z^2c_1g_1^2(z)\right)\frac{1}{n}\tr {\bf A}_0{\bf A}_1 }, \ \ i\in\left\{0,1\right\}.
\end{align}
%		\begin{equation}
%			R_{ij} = \frac{c_i}{c_j} \left[ \left(\mathbf{I}_2-\mathbf{\Omega}\left(z\right)\right)^{-1}\mathbf{\Omega}\left(z\right)\right]_{ij}, \: \mathbf{\Omega}\left(z\right)_{ij} = \frac{p}{n}c_j z^2g_i^2(z)\frac{1}{p}\tr \mathbf{\Sigma}_0\bar{\mathbf{Q}}\left(z\right) \mathbf{\Sigma}_1\bar{Q}\left(z\right).
%		\end{equation}
%	Then, we have $Q(z) \mathbf{\Sigma}_a Q(z) \leftrightarrow \widetilde{Q}_a\left(z\right)$.
The quantities in (\ref{eq:fixed_point_LDA}) can be computed in an iterative fashion where convergence is guaranteed after few iterations (see \cite{gram_mixture} for more details).
Moreover, define
\begin{align}
\overline{G}_i\left(z\right) & \triangleq \frac{\left(-1\right)^{i+1}}{2} z \bm{\mu}^T \bar{\mathbf{Q}}\left(z\right) \bm{\mu} + \frac{z}{2n_0} \tr {\bf A}_0 - \frac{z}{2n_1} \tr {\bf A}_1 .
\end{align}
\begin{align}
\overline{D}_i\left(z\right) & \triangleq z^2 \bm{\mu}^T\widetilde{\mathbf{Q}}_i\left(z\right) \bm{\mu}
+ \frac{z^2}{n_0} \tr \mathbf{\Sigma}_0\widetilde{\mathbf{Q}}_i\left(z\right) + \frac{z^2}{n_1} \tr \mathbf{\Sigma}_1\widetilde{\mathbf{Q}}_i\left(z\right),
\end{align}
where $\bm{\mu} = \bm{\mu}_0 - \bm{\mu}_1$.
With these definitions at hand, we   state the following theorem
\begin{theorem}
	\label{DE_LDA}
	Under Assumptions \ref{As:1L}-\ref{As:3L}, we have
	\begin{align}
	\label{G_i}
	G\left(\bm{\mu}_i,\hat{\bm{\mu}}_0,\hat{\bm{\mu}}_1,\mathbf{H}\right) - \overline{G}_i\left(-\frac{1}{c\gamma}\right) & \to_{a.s.} 0.
	\end{align}
	\begin{align}
	\label{D_i}
	D\left(\hat{\bm{\mu}}_0,\hat{\bm{\mu}}_1,\mathbf{H},\mathbf{\Sigma}_i\right) - \overline{D}_i\left(-\frac{1}{c\gamma}\right) & \asto 0.
	\end{align}
	%		D\left(\bar{\mathbf{x}}_0,\bar{\mathbf{x}}_1,\mathbf{H},\mathbf{\Sigma}_i\right) \asymp z^2 \bm{\mu}^T\widetilde{Q}_i\left(z\right) \bm{\mu}
	%		+ \frac{z^2}{n_0} \tr \mathbf{\Sigma}_0\widetilde{Q}_i\left(z\right) + \frac{z^2}{n_1} \tr \mathbf{\Sigma}_1\widetilde{Q}_i\left(z\right)
	As a consequence, the conditional misclassification probability converges almost surely to a deterministic quantity $\bar{\epsilon}_i^{R-LDA}$
	\begin{align}
	\epsilon_i^{R-LDA} - \bar{\epsilon}_i^{R-LDA} \asto 0,
	\end{align}
	where
	\begin{align}
	\bar{\epsilon}_i^{R-LDA} = \Phi\left(\frac{\left(-1\right)^{i+1}\overline{G}_i + \left(-1\right)^i \log\left(\frac{\pi_0}{\pi_1}\right)}{\sqrt{\overline{D}_i}}\right).
	\end{align}
	%\begin{equation}
	%	G\left(\bm{\mu}_i,\bar{\mathbf{x}}_0,\bar{\mathbf{x}}_1,\mathbf{H}\right) \asymp \frac{\left(-1\right)^{i+1}}{2} z \bm{\mu}^T \bar{Q}\left(z\right) \bm{\mu} + \frac{z}{2n_0} \tr \mathbf{\Sigma}_0\bar{Q}\left(z\right) - \frac{z}{2n_1} \tr \mathbf{\Sigma}_1\bar{Q}\left(z\right)
	%\end{equation}
\end{theorem}
\begin{proof}
	See Appendix 	\ref{appendix:theorem1}.
\end{proof}
\begin{remark}
	As stated earlier, if $\|\bm{\mu}\|$ scales faster than $O\left(1\right)$, perfect asymptotic classification is achieved. %regardless of the choice of $\mathbf{\Sigma}_0$ and $\mathbf{\Sigma}_1$ satisfying assumption \ref{As:4L}. 
	This can be  seen by noticing that $\frac{\overline{G}_i\left(-\frac{1}{c\gamma}\right)}{\sqrt{\overline{D}_i\left(-\frac{1}{c\gamma}\right)}}$ would grow indefinitely large with $p$, thereby making the conditional error rates vanish.  %in the expressions of $\overline{G}_i\left(-\frac{1}{c\gamma}\right)$ and $\overline{D}_i\left(-\frac{1}{c\gamma}\right)$. %with the terms involving $\bm{\mu}$ are the dominants ones allowing the total misclassification probability to approach zero.
\end{remark}
When $\left\|\boldsymbol{\Sigma}_0-\boldsymbol{\Sigma}_1\right\|$ converges to zero, the asymptotic misclassification error rate of each class coincides with the one derived in \cite{zollanvari} obtained when $\mathbf{\Sigma}_0=\mathbf{\Sigma}_1$.
\begin{corollary}
	\label{common_cov_remark} In the case where $\|\mathbf{\Sigma}_0-\mathbf{\Sigma}_1\| = o\left(1\right)$ (including the common covariance case where $\mathbf{\Sigma}_0=\mathbf{\Sigma}_1=\mathbf{\Sigma}$), the conditional misclassification error rate converges almost surely to $\overline{\overline{\epsilon}}_{i}^{R-LDA}$
	$$
	\overline{\overline{\epsilon}}_{i}^{R-LDA}=\Phi\left(\frac{(-1)^{i+1}\overline{\overline{G}}_i(-\frac{1}{c\gamma})+(-1)^{i}\log\frac{\pi_0}{\pi_1}}{\sqrt{\overline{\overline{D}}_i(-\frac{1}{c\gamma})}}\right),
	$$
	where
	\begin{align*}
	\overline{\overline{G}}_i\left(-\frac{1}{c\gamma}\right) &=  \frac{\left(-1\right)^i}{2} \bm{\mu}^T \left(\mathbf{I}_p + \frac{\gamma}{1+\gamma \delta} \mathbf{\Sigma}\right)^{-1}\bm{\mu} - \frac{n \delta}{2}\left(\frac{1}{n_0}-\frac{1}{n_1}\right)
	\end{align*}
	\begin{align*}
	 &\overline{\overline{D}}_i\left(-\frac{1}{c\gamma}\right) \\ &  = \frac{\left[\bm{\mu}^T \mathbf{\Sigma}\left(\mathbf{I}_p + \frac{\gamma}{1+\gamma \delta} \mathbf{\Sigma}\right)^{-2}\bm{\mu} + \left(\frac{1}{n_0} + \frac{1}{n_1}\right) \tr \mathbf{\Sigma}^2\left(\mathbf{I}_p + \frac{\gamma}{1+\gamma \delta} \mathbf{\Sigma}\right)^{-2}\right]}{1 -\frac{\gamma^2}{n\left(1+\gamma \delta\right)^2} \tr \mathbf{\Sigma}^2\left(\mathbf{I}_p + \frac{\gamma}{1+\gamma \delta} \mathbf{\Sigma}\right)^{-2}}.
	%	& \times \left[\bm{\mu}^T \mathbf{\Sigma}\left(\mathbf{I}_p + \frac{\gamma}{1+\gamma \delta} \mathbf{\Sigma}\right)^{-2}\bm{\mu} + \left(\frac{1}{n_0} + \frac{1}{n_1}\right) \tr \mathbf{\Sigma}^2\left(\mathbf{I}_p + \frac{\gamma}{1+\gamma \delta} \mathbf{\Sigma}\right)^{-2}\right]
	\end{align*}
	where $\boldsymbol{\Sigma}=\boldsymbol{\Sigma}_0$ or $\boldsymbol{\Sigma}=\boldsymbol{\Sigma}_1$, and $\delta$ is the unique positive solution to the following equation:
	$$
	\delta=\frac{1}{n}\tr \boldsymbol{\Sigma}\left({\bf I}_p+\frac{\gamma}{1+\gamma\delta}\boldsymbol{\Sigma}\right)^{-1}.
	$$
	% is any of $\left\{\boldsymbol{\Sigma}_0,\boldsymbol{\Sigma}_1\right\}$.
	%the conditional misclassification probability converges almost surely to  $\overline{\epsilon}_i^{R-LDA}$:
	%$$
	%\epsilon_i^{R-LDA}-\overline{}
	%$$
	%it is possible starting from the results of Theorem \ref{DE_LDA} and after some basic manipulations to recover the same results as in \citep[Theorem 1]{zollanvari}.
\end{corollary}
\begin{proof}
	%		If $\lim \sup_{\alpha , \beta} \frac{1}{p} \tr \left(\mathbf{\Sigma}_0-\mathbf{\Sigma}_1\right)\left(\mathbf{I} + \alpha \mathbf{\Sigma}_0 + \beta \mathbf{\Sigma}_1\right)^{-1} = o\left(1\right)$
	When $\|\mathbf{\Sigma}_0-\mathbf{\Sigma}_1\| = o\left(1\right)$, we first prove that up to an error  $o(1)$, the key deterministic equivalents can be simplified to depend only on $\boldsymbol{\Sigma}_0$ (or $\mathbf{\Sigma}_1$).
	In the sequel, we take $\boldsymbol{\Sigma}=\boldsymbol{\Sigma}_0$.    As $\|\mathbf{\Sigma}_0-\mathbf{\Sigma}_1\| = o\left(1\right)$, we have
	$$\tilde{g}_i\left(z\right) = \underbrace{\frac{1}{p} \tr \mathbf{\Sigma}\bar{\mathbf{Q}}\left(z\right)}_{\tilde{g}\left(z\right)} + o(1) , \: \forall i \in \{0,1\}.$$ 	
	It follows that $g_i\left(z\right) = g\left(z\right) + o(1)$ where $g\left(z\right) = - \frac{1}{z} \frac{n}{p} \frac{1}{1+\tilde{g}\left(z\right) }$. The above relations allow  to simplify functionals involving matrix $\bar{\mathbf{Q}}$. To see that, we decompose $\bar{\mathbf{Q}}$ as
	%We can further decompose $\bar{\mathbf{Q}}$ as:			
	\begin{align*}
	& \bar{\mathbf{Q}}\left(z\right)  \\ & = -z^{-1} \left(\mathbf{I} + g\left(z\right) \mathbf{\Sigma}_0 + c_1 g\left(z\right)\left(\mathbf{\Sigma}_1-\mathbf{\Sigma}_0\right) \right)^{-1} +o_{\|.\|}(1)\\
	& = -z^{-1} \left(\mathbf{I} + g\left(z\right) \mathbf{\Sigma}_0\right)^{-1} - z^{-1} \Biggl[\left(\mathbf{I} + g\left(z\right) \mathbf{\Sigma}_0 + c_1 g\left(z\right)\left(\mathbf{\Sigma}_1-\mathbf{\Sigma}_0\right) \right)^{-1}  \\ & - \left(\mathbf{I} + g\left(z\right) \mathbf{\Sigma}_0\right)^{-1} \Biggr]
	+o_{\|.\|}(1) \\
	& \overset{(a)}{=}  -z^{-1} \left(\mathbf{I} + g\left(z\right) \mathbf{\Sigma}_0\right)^{-1}  \\ \scriptscriptstyle
	& - z^{-1} \left(\mathbf{I} + g\left(z\right) \mathbf{\Sigma}_0 + c_1 g\left(z\right)\left(\mathbf{\Sigma}_1-\mathbf{\Sigma}_0\right) \right)^{-1}  c_1g\left(z\right) \left(\mathbf{\Sigma}_1-\mathbf{\Sigma}_0\right) \\ & \times \left(\mathbf{I} + g\left(z\right) \mathbf{\Sigma}_0\right)^{-1} +o_{\|.\|}(1)
	\end{align*}
	where $(a)$ follows from the resolvent identity and $o_{\|.\|}(1)$ denotes a matrix with spectral norm converging to zero. Define 
	\begin{align*}
	\mathbf{\Psi} & = z^{-1} \left(\mathbf{I} + g\left(z\right) \mathbf{\Sigma}_0 + c_1 g\left(z\right)\left(\mathbf{\Sigma}_1-\mathbf{\Sigma}_0\right) \right)^{-1}  c_1g\left(z\right) \left(\mathbf{\Sigma}_1-\mathbf{\Sigma}_0\right) \\ & \times \left(\mathbf{I} + g\left(z\right) \mathbf{\Sigma}_0\right)^{-1}.
			\end{align*}
	Then, it can be shown using the inequality $\|\mathbf{AB}\|\leq \|{\bf A}\|\|{\bf B}\|$ for ${\bf A}$ and ${\bf B}$ two matrices in $\mathbb{R}^{p\times p}$ that:
	\begin{align*}
& 	\|  \mathbf{\Psi}\|  \overset{(b)}{\leq} z^{-1}c_1 g\left(z\right)\|\mathbf{\Sigma}_0-\mathbf{\Sigma}_1\| \\ & \times \left\|\left(\mathbf{I} + g\left(z\right) \mathbf{\Sigma}_0 + c_1 g\left(z\right)\left(\mathbf{\Sigma}_1-\mathbf{\Sigma}_0\right) \right)^{-1} \left(\mathbf{I} + g\left(z\right) \mathbf{\Sigma}_0\right)^{-1} \right\|  \\ & = o\left(1\right).
	\end{align*}
	Hence, for ${\bf a},{\bf b}\in \mathbb{R}^{p}$, 
	$$
	\boldsymbol{a}^{T}\bar{\mathbf{Q}}\left(z\right)\boldsymbol{b}=-z^{-1}\boldsymbol{a}^{T}\left(\mathbf{I} + g\left(z\right) \mathbf{\Sigma}\right)^{-1}\boldsymbol{b}+o(1),
	$$
	and $\frac{1}{p}\tr {\bf A}\bar{\mathbf{Q}}=\frac{-z^{-1}}{p}\tr {\bf A}\left(\mathbf{I} + g\left(z\right) \mathbf{\Sigma}\right)^{-1}+o(1)$.
	Using the same notations as in \cite{zollanvari} we have in particular for $z = -\frac{1}{c\gamma}$,
	$\tilde{g}\left(z\right) = \delta+o(1)$ and $g\left(z\right) = \frac{\gamma}{1+\gamma \delta}+o(1)$,                     where $\delta$ is the fixed-point solution in \cite[Proposition 1]{zollanvari}.
	Moreover,
	\begin{align*}
	& \frac{z}{2n_0} \tr \mathbf{\Sigma}_0\bar{\mathbf{Q}}\left(z\right) - \frac{z}{2n_1} \tr \mathbf{\Sigma}_1\bar{\mathbf{Q}}\left(z\right)  \\ &  = \frac{z \tr \mathbf{\Sigma}_0 \bar{\mathbf{Q}}\left(z\right)}{2} \left(\frac{1}{n_0}-\frac{1}{n_1}\right)  + \underbrace{\frac{z}{2 n_1} \tr \left(\mathbf{\Sigma}_0-\mathbf{\Sigma}_1\right)\bar{\mathbf{Q}}\left(z\right)}_{\leq \|\mathbf{\Sigma}_0-\mathbf{\Sigma}_1\| \frac{z}{2n_1} \tr \bar{\mathbf{Q}}\left(z\right)} \\
	& = \frac{z \tr \mathbf{\Sigma} \bar{\mathbf{Q}}\left(z\right)}{2} \left(\frac{1}{n_0}-\frac{1}{n_1}\right) + o\left(1\right) \\
	& = \frac{- \tr \mathbf{\Sigma} \left(\mathbf{I} + g\left(z\right) \mathbf{\Sigma}\right)^{-1}}{2}\left(\frac{1}{n_0}-\frac{1}{n_1}\right) + o\left(1\right).
	\end{align*}
	It follows that
	\begin{align*}
	\overline{G}_i\left(-\frac{1}{c\gamma}\right) & =  \frac{\left(-1\right)^i}{2} \bm{\mu}^T \left(\mathbf{I}_p + \frac{\gamma}{1+\gamma \delta} \mathbf{\Sigma}\right)^{-1}\bm{\mu} \\ & - \frac{n \delta}{2}\left(\frac{1}{n_0}-\frac{1}{n_1}\right)+ o \left(1\right).
	\end{align*}					
	Using the same arguments, we can show that 				
	\begin{align*}
& 	\overline{D}_i\left(-\frac{1}{c\gamma}\right) = \left[1 -\frac{\gamma^2}{n\left(1+\gamma \delta\right)^2} \tr \mathbf{\Sigma}^2\left(\mathbf{I}_p + \frac{\gamma}{1+\gamma \delta} \mathbf{\Sigma}\right)^{-2}\right]^{-1} \\ & \times \Biggl[\bm{\mu}^T \mathbf{\Sigma}\left(\mathbf{I}_p + \frac{\gamma}{1+\gamma \delta} \mathbf{\Sigma}\right)^{-2}\bm{\mu}  \\ & + \left(\frac{1}{n_0} + \frac{1}{n_1}\right) \tr \mathbf{\Sigma}^2\left(\mathbf{I}_p + \frac{\gamma}{1+\gamma \delta} \mathbf{\Sigma}\right)^{-2}\Biggr] + o(1).
	\end{align*}										
\end{proof}
%		Corollary \ref{common_cov_remark} brings the question: is there a condition that governs the distance between $\mathbf{\Sigma}_0$ and $\mathbf{\Sigma}_1$ (in some sens) such that the performance of RLDA with common covariance matrices derived in \citep[Theorem 1]{zollanvari} will be asymptotically equivalent to the one derived in Theorem \ref{DE_LDA}?. This is illustrated in the following remark.
%		\begin{remark}
%			\label{remark_LDA_distance}
%			In the case where $\|\mathbf{\Sigma}_0-\mathbf{\Sigma}_1\| = o\left(1\right)$, it can be shown that the classification error asymptotically converges to that of R-LDA with common covariance matrices.
%		\end{remark}
Corollary \ref{common_cov_remark} is useful because it allows to specify the range of applications of Theorem \ref{DE_LDA} in which the information on  the covariance matrix is essential for the classification task.
%	 \footnote{We believe that weaker assumptions can be made on $\mathbf{\Sigma}_0$ and $\mathbf{\Sigma}_1$ for Remark \ref{remark_LDA_distance} to hold.}
Also, it shows how R-LDA is robust against small perturbations in the covariance matrix. Similar observations have been made in \cite{DA_unequal} where it was shown via a Monte Carlo study that LDA is robust against the modeling assumptions. 
%		It is worth mentioning however that the assumption $\|\mathbf{\Sigma}_0 - \mathbf{\Sigma}_1\| = o\left(1\right)$ can be relaxed to a weaker assumption on the distance between $\mathbf{\Sigma}_0$ and $\mathbf{\Sigma}_1$ and still
%        have the output of Corollary \ref{common_cov_remark}. \textcolor{red}{Abla: Pas juste meme si tu prend l'assumption sur tout A, tu n'a pas la trace normalise dans le terme $\mu^{T}\psi\mu$}A weaker assumption would be for instance $\frac{1}{p} \tr \mathbf{A} \left(\mathbf{\Sigma}_0 - \mathbf{\Sigma}_1\right) = o\left(1\right)$, for all $\|\mathbf{A}\| = \bigo\left(1\right)$. In this case, we may still have $\|\mathbf{\Sigma}_0 - \mathbf{\Sigma}_1\| = \bigo\left(1\right)$ but the performance is asymptotically the same as having equal covariance matrices.
\\
\subsection{Asymptotic Performance of R-QDA}
\label{QDA_section}
In this part, we state the main results regarding the derivation of deterministic approximations of the R-QDA classification error.  Such results have been obtained by considering some specific assumptions, carefully chosen such that an asymptotically non-trivial classification error (i.e., neither $0$ nor $1$) is achieved. We particularly highlight how the provided asymptotic approximations depend  on such statistical parameters as the means and covariances within
classes, thus allowing a better understanding of the performance of the R-QDA classifier. Ultimately, these results can be exploited in order to improve the performances by allowing optimal setting of  the regularization parameter.
%In this section, we state the main results of the paper related to the derivation of a deterministic equivalent of the QDA's classification error under certain assumptions. These assumptions are taken such that an asymptotically non-trivial classification error (i.e., neither 0 nor 1) can be achieved. As such, the objective of the current paper is to analyze the performance of QDA and reveal the relation between the asymptotic classification error with the QDA's components such as the classes' statistical means and the covariance matrices. In doing so, we allow a more rigorous understanding of the QDA classifier in the double asymptotic regime. As a result, it is possible to exploit this understanding and permit a better design, where for example we could optimize the regularization parameter to minimize the classification error.
\subsubsection{Technical Assumptions}
We consider the following double asymptotic regime in which $n_i,$ $p$ $\to \infty$ for $i \in \{0,1\}$ with the following assumptions met
\begin{assumption}[Data scaling]
	%\textcolor{red}{Abla: Will theoretical results depend on d? I am not sure that this is required. }
	\label{As:1} $n_0-n_1 = o\left(1\right)$ and
	$\frac{p}{n} \to c \in \left(0,\infty\right)$.
\end{assumption}
%\begin{assumption}
%	\label{As:2}
%	$\frac{n_i}{n} \to c_i \in \left(0,\infty\right)$.
%\end{assumption}
\begin{assumption}[Mean scaling]
	\label{As:2}
	$\left \| \bm{\mu}_0 - \bm{\mu}_1 \right \|^2 = O\left(\sqrt{p}\right)$.
\end{assumption}
\begin{assumption}[Covariance scaling]
	\label{As:3}
	$\left \| \mathbf{\Sigma}_i \right \|=O\left(1\right)$.
\end{assumption}
\begin{assumption}
	\label{As:4}
	Matrix $\boldsymbol{\Sigma}_0-\boldsymbol{\Sigma}_1$ has exactly $O(\sqrt{p})$ eigenvalues of order $O(1)$. The remaining eigenvalues are of order $O(\frac{1}{\sqrt{p}})$.
	%Let $f$ be a bivariate smooth function. Then,
	% Let $\mathcal{P}_2$ be the set of the bivariate polynomials whose coefficients are bounded by $1$.
	%	$$ \frac{1}{\sqrt{p}}\tr f(\boldsymbol{\Sigma}_1,\boldsymbol{\Sigma}_2)\left(\mathbf{\Sigma}_0-\mathbf{\Sigma}_1\right)  = \bigo\left(1\right),$$
	% where $f$ is a bi-variate polynomial.
\end{assumption}
%$\max_{i,j \in \{0,1\}} \sup_{\alpha,\beta , \lambda , \nu \geq  0} \frac{1}{\sqrt{p}} \tr \left(\mathbf{I}_p+\alpha \mathbf{\Sigma}_i\right)^{-1} \left(\mathbf{I}_p + \lambda \mathbf{\Sigma}_k\right) \left(\mathbf{\Sigma}_0-\mathbf{\Sigma}_1\right) \left(\mathbf{I}_p + \nu \mathbf{\Sigma}_l\right) \left(\mathbf{I}_p+\beta \mathbf{\Sigma}_j\right)^{-1} = \bigo\left(1\right)$
%\begin{assumption}
%	\label{As:4}
%	$\lim \sup \frac{1}{\sqrt{p}}\tr \mathbf{A}\left(\mathbf{\Sigma}_0-\mathbf{\Sigma}_1\right)= \bigo\left(1\right)$, for all $\mathbf{A} \in \mathbb{R}^{p\times p}$ satisfying $\left\|\mathbf{A}\right\| = \bigo\left(1\right)$.
%\end{assumption}
%\begin{itemize}
%	\item[As1]\label{As:1} $\frac{p}{n} \to c \in \left(0,\infty\right)$.
%	\item[As2]\label{As1} $\frac{n_i}{n} \to c_i \in \left(0,\infty\right)$.
%	\item[As3]\label{As1} $\left \| \bm{\mu}_0 - \bm{\mu}_1 \right \| = \bigo\left(1\right)$.
%	\item[As4]\label{As1} $\left \| \mathbf{\Sigma}_i \right \|=\bigo\left(1\right)$.
%	 \item[As5]\label{As1} $\sup_{\left\|\mathbf{A}\right\| = \bigo\left(1\right)} \frac{1}{\sqrt{p}}\tr \mathbf{A}\left(\mathbf{\Sigma}_0-\mathbf{\Sigma}_1\right)= \bigo\left(1\right).$
%\end{itemize}
Assumption \ref{As:1} implies also that $\pi_i\to \frac{1}{2}$ for $i\in\left\{0,1\right\}$. As we shall see later, if this is not satisfied, the R-QDA  perform asymptotically as the  classifier that  assigns all observations to the same class.  %the  class with the highest prior probability $\pi_i$.
%The first assumption state that both the number of features and the number of training samples have the same order of magnitudes. This allows to make the problem more tractable and to exploit important results from the theory of large random matrices. Moreover, with assumption \ref{As:1}, for $i \in \{0.,1\}$, $\pi_i \to \frac{1}{2}$ as $p \to \infty$.
The second assumption governs the distance between the two classes in terms of the Euclidean distance  between the means. This is mandatory in order to avoid asymptotic perfect classification. This is a much stronger assumption than Assumption \ref{As:2L} in R-LDA since we allow larger values for $\|\bm{\mu}_0-\bm{\mu}_1\|$. This can be understood as R-QDA being subject to strong noise induced when estimating $\mathbf{\Sigma}_i$, $i \in \{0,1\}$ which requires
a large value $\|\bm{\mu}_0-\bm{\mu}_1\|$ so that it can play a role in classification. A similar assumption is required to control the distance between the covariance matrices. Particularly, the spectral norm of the covariance matrices are required to be bounded as stated in Assumption \ref{As:3} while their difference should satisfy Assumption \ref{As:4}.% This means that we need the covariance matrices to be far apart at most by a factor of $\bigo\left(\sqrt{p}\right)$ in terms of trace. 
The latter assumption implies that for any matrix ${\bf A}$ of bounded spectral norm,
$$
\frac{1}{\sqrt{p}}\tr {\bf A}(\boldsymbol{\Sigma}_0-\boldsymbol{\Sigma}_1)=O(1)
.$$
\subsubsection{Central Limit Theorem (CLT)}
It can be easily shown that the R-QDA conditional classification error in \eqref{qda_conditional_error} can be expressed  as
\begin{align}
\label{simplified_cond_error}
\epsilon^{R-QDA}_i = \mathbb{P}\left[\bm{z}^T \mathbf{B}_i\bm{z}+2\bm{z}^T\mathbf{r}_i < \xi_i | \bm{z} \sim \mathcal{N}\left(\mathbf{0},\mathbf{I}_p\right), \mathcal{T}_0,\mathcal{T}_1\right],
\end{align}
where
\begin{align*}
\mathbf{B}_i & =\mathbf{\Sigma}_i^{1/2} \left(\mathbf{H}_1-\mathbf{H}_0\right) \mathbf{\Sigma}_i^{1/2}, \\ \mathbf{r}_i& =\mathbf{\Sigma}_i^{1/2}\left[\mathbf{H}_1\left(\bm{\mu}_i-\hat{\bm{\mu}}_1\right)-\mathbf{H}_0\left(\bm{\mu}_i-\hat{\bm{\mu}}_0\right)\right], \\  \xi_i &= -\log \left(\frac{|\mathbf{H}_0|}{|\mathbf{H}_1|}\right) + \left(\bm{\mu}_i-\hat{\bm{\mu}}_0\right)^T \mathbf{H}_0 \left(\bm{\mu}_i-\hat{\bm{\mu}}_0\right)  \\ & -
\left(\bm{\mu}_i-\hat{\bm{\mu}}_1\right)^T \mathbf{H}_1 \left(\bm{\mu}_i-\hat{\bm{\mu}}_1\right) + 2 \log \frac{\pi_1}{\pi_0}.
\end{align*}
Computing $ \epsilon^{R-QDA}_i$ amounts  to the cumulative distribution function (CDF) of quadratic forms of Gaussian random vectors, and hence cannot be derived in closed form in general. However, it can be still approximated by considering asymptotic regimes that allow to exploit results about central limit theorem involving quadratic forms. 
Under Assumptions \ref{As:1}-\ref{As:4}, a central limit theorem (CLT) on the random variable $\bm{z}^T \mathbf{B}_i\bm{z}+2\bm{z}^T\mathbf{r}_i$ when $\bm{z} \sim \mathcal{N}\left(\mathbf{0},\mathbf{I}_p\right)$ is established.
%		This result is essential to evaluate the asymptotic approximation of the misclassification rate and is stated as follows:
%The corresponding result is stated in the following proposition:%This is illustrated in the following proposition
\begin{proposition}[CLT]
	\label{CLT_quadratic}
	Assume that assumptions \ref{As:1}-\ref{As:4} hold true. Assume also that for $i \in \{0,1\}$
	\begin{equation}
	\lim_{p\to\infty} \frac{60\tr {\bf B}_i^2+240\tr {\bf B}_i^2\|{\bf r}_i\|_{2}^{2}+48 \|{\bf r}_i\|_{2}^4}{\left(2\tr {\bf B}_i^2+4\|{\bf r}_i\|_{2}^2\right)^2}\to 0. 
	\label{cond}
	\end{equation}
	Then,
	%	Under assumptions \ref{As:1}-\ref{As:4} and with the help of Theorem \ref{DE_QDA}, the following convergence holds
	\begin{align}
	\label{clt_quad} \frac{\bm{z}^T \mathbf{B}_i \bm{z} + 2 \bm{z}^T \mathbf{r}_i - \tr \mathbf{B}_i}{\sqrt{2 \tr \mathbf{B}_i^2 + 4 \mathbf{r}_i^T\mathbf{r}_i}} \to_d \mathcal{N}\left(0,1\right).
	\end{align}
	\label{prop:clt}
\end{proposition}
\begin{proof}
	The proof is mainly based on the application of the Lyapunov's CLT for the sum of independent but non identically distributed random variables \cite{Billingsley}. The detailed proof is postponed to Appendix \ref{appendix:clt}.
\end{proof}
The condition in\eqref{cond} will be proven to hold almost surely. 
Hence, 	as a by-product of the above Proposition, we obtain the following expression for the conditional classification error $\epsilon_i$
\begin{corollary}
	Under the setting of Proposition \ref{prop:clt}	, the  conditional classification error in (\ref{qda_conditional_error}) satisfies
	\begin{align}
	\epsilon^{R-QDA}_i - \Phi\left(\left(-1\right)^i \frac{\xi_i-\tr \mathbf{B}_i}{\sqrt{2 \tr \mathbf{B}_i^2 + 4 \mathbf{r}_i^T\mathbf{r}_i}}\right) \asto 0.
	\end{align}
\end{corollary}
As such an asymptotic equivalent of the conditional classification error can be derived. This is the subject of the next subsection.
\subsubsection{Deterministic Equivalents}
This part is devoted to the derivation of deterministic equivalents of some random quantities involved in the R-QDA conditional classification error. Before that, we shall introduce the following notations which basically arise as a result of applying standard results from random matrix theory. We define for $i\in \{0,1\}$, $\delta_i$ as the unique positive solution to the following fixed point equation\footnote{Mathematical details treating the existence and uniqueness of $\delta_i$ can be found in \cite{walid_new}.}
$$
\delta_i=\frac{1}{n_i}\tr \boldsymbol{\Sigma}_i\left({\bf I}_p+\frac{\gamma}{1+\gamma\delta_i}\boldsymbol{\Sigma}_i\right)^{-1}.
$$
Define ${\bf T}_i$ as
$$
{\bf T}_i=\left({\bf I}_p+\frac{\gamma}{1+\gamma\delta_i}\boldsymbol{\Sigma}_i\right)^{-1},
$$
and the scalar ${\phi}_i$ and $\tilde{\phi}_i$ as
$$
\phi_i=\frac{1}{n_i}\tr \boldsymbol{\Sigma}_i^2{\bf T}_i^2, \hspace{0.5cm} \tilde{\phi}_i=\frac{1}{(1+\gamma\delta_i)^2}.
$$
Define $\overline{\xi}_i$, $\overline{b}_i$ and $\overline{B}_i$ as
\begin{equation}
\label{eq:xi}
\begin{split}
\overline{\xi}_i & \triangleq  \frac{1}{\sqrt{p}}\Biggl[-  \log \frac{|\mathbf{T}_0|}{|\mathbf{T}_1|} + \log \frac{\left(1+\gamma \delta_0\right)^{n_0}}{\left(1+\gamma \delta_1\right)^{n_1}}  \\ & + \gamma \left(\frac{n_1 \delta_1}{1+\gamma \delta_1}-\frac{n_0 \delta_0}{1+\gamma \delta_0}\right)  +\left(-1\right)^{i+1}\bm{\mu}^T \mathbf{T}_{1-i}\bm{\mu} \biggr].
\end{split}
\end{equation}
\begin{align}
\overline{b}_i  = \frac{1}{\sqrt{p}} \tr \mathbf{\Sigma}_i \left(\mathbf{T}_1-\mathbf{T}_0\right). \label{eq:bi}
%	\overline{B_i} & \triangleq c \left[ \frac{   \phi_0 }{1-\gamma^2 \phi_0 \tilde{\phi}_0}  + \frac{   \phi_1 }{1-\gamma^2 \phi_1 \tilde{\phi}_1} \right] + \textcolor{blue}{\frac{1}{p} \tr \left(\mathbf{\Sigma}_i^2-\mathbf{\Sigma}_{1-i}^2\right)\mathbf{T}_{1-i}^2} -  \frac{2}{p}\tr \mathbf{\Sigma}_i\mathbf{T}_1\mathbf{\Sigma}_i\mathbf{T}_0. \\
 %\frac{c \phi_i}{1-\gamma^2 \phi_i \tilde{\phi}_i} + \frac{1}{p}\tr \mathbf{\Sigma}_i^2 \mathbf{T}_{1-i}^2 + \frac{c \gamma^2 \tilde{\phi}_{1-i}}{1-\gamma^2 \phi_{1-i} \tilde{\phi}_{1-i}}\left(\frac{1}{n_i}\tr \mathbf{\Sigma}_i\mathbf{\Sigma}_{1-i}\mathbf{T}_{1-i}^2\right)^2 - \frac{2}{p}\tr \mathbf{\Sigma}_i\mathbf{T}_1 \mathbf{\Sigma}_i\mathbf{T}_0. 
\end{align}
\begin{equation}
\label{eq:Bi}
\begin{split}
\overline{B_i} & \triangleq \frac{\phi_i}{1-\gamma^2\phi_i\tilde{\phi}_i}\frac{n_i}{p}+\frac{1}{p}\tr \boldsymbol{\Sigma}_i^2{\bf T}_{1-i}^2 \\ & +\frac{n_{i}}{p}\frac{\gamma^2\tilde{\phi}_{1-i}}{1-\gamma^2{\phi}_{1-i}\tilde{\phi}_{1-i}}\left(\frac{1}{n_i}\tr \mathbf{\Sigma}_i\mathbf{\Sigma}_{1-1}{\bf T}_{1-i}^{2}\right)^2 \\ & -\frac{2}{p}\tr \boldsymbol{\Sigma}_i{\bf T}_1\boldsymbol{\Sigma}_i{\bf T}_0.
\end{split}
\end{equation}
%		\begin{theorem}
%			\label{DE_th_QDA}
%			Let Assumptions \ref{As:1},\ref{As:2}, \ref{As:3} and \ref{As:4} hold. Then, we have the following convergences
%			\begin{align}
%			&	\label{xi}
%			\xi_i - \overline{\xi}_i  \xrightarrow{p}0.
%			\end{align}
%			\begin{align}
%			&\label{Bi}
%			\frac{1}{\sqrt{p}}\tr \mathbf{B}_i - \overline{b_i} \xrightarrow{p}0.
%			\end{align}
%			\begin{align}
%			\label{B2}
%			&\frac{1}{p}\tr \mathbf{B}_i^2 - 	\overline{B_i} \xrightarrow{p}0.
%			\end{align}
%			\begin{align}
%			\label{yi}
%			&\frac{1}{p} \mathbf{y}_i^T\mathbf{y}_i \xrightarrow{p}0.
%			\end{align}
%		\end{theorem}
%As stated earlier, the conditioning is on the training samples given by $\mathcal{T}_i$, $i=0,1$. The next step consists on showing that based on the afro-mentioned assumptions, the conditional classification error converges in probability to a deterministic quantity that only depends on the class' statistics and the problem's dimensions. In the following, we provide a deterministic equivalent of the conditional classification error and provide some insights.
As shall be shown in Appendix \ref{appendix:DE_QDA}, these quantities are deterministic approximations in probability of $\xi$, $b_i$ and $B_i$. We therefore get
\begin{theorem}
	\label{DE_QDA}
	Under assumptions \ref{As:1}-\ref{As:4}, the following convergence holds for $i \in \{0,1\}$
	\begin{align*}
	\epsilon^{R-QDA}_i - \Phi\left(\left(-1\right)^i \frac{\overline{\xi}_i-\overline{b_i}}{\sqrt{2 \overline{B_i}}}\right) \probto 0.
	\end{align*}
\end{theorem}
\begin{proof}
	The proof is postponed to Appendix 	\ref{appendix:DE_QDA}.
\end{proof}
At first sight, quantity $\overline{\xi}_i-\overline{b}_i$ appears to be of order $O(\sqrt{p})$, since $\frac{1}{\sqrt{p}}\log |{\bf T}_i|$ and $\frac{1}{\sqrt{p}}\tr \boldsymbol{\Sigma}_i{\bf T}_i$ are $O(\sqrt{p})$. Following this line of thought, the asymptotic misclassification probability error is expected to converge to a trivial misclassification error. This statement is, hopefully false. Assumption \ref{As:4} and \ref{As:1} were                carefully 
designed so that  $\frac{1}{\sqrt{p}}\log |{\bf T}_1|-\frac{1}{\sqrt{p}}\log|{\bf T}_0|$ and $\frac{1}{\sqrt{p}}(n_1\delta_1-n_0\delta_0)$ are of order $O(1)$. In particular, the following is proven in Appendix \ref{app:order}
\begin{proposition}
	Under Assumption \ref{As:1}-\ref{As:4} The deterministic quantities $\overline{\xi}_i$ and $\overline{b}_i$ are uniformly bounded when $p$ grows to infinity.
	\label{prop:order}
\end{proposition}
\begin{proof}
	The proof is deferred to Appendix \ref{app:order}
\end{proof}
\begin{remark}
	The results of       Theorem \ref{DE_QDA} along with proposition \ref{prop:order} show that the classification error converges to a non-trivial deterministic quantity that  depends only on the  statistical means and covariances within each class.
	The major importance of this result is that it allows to find a good choice of the regularization $\gamma$ as the value that minimizes the asymptotic classification error. While it seems to be elusive for such value to possess a closed-form expression, it can be numerically approximated by using a simple one-dimensional line search algorithm.
\end{remark}
\begin{remark}
	Using Assumption \ref{As:4}, it can be shown that $\overline{B}_i$ can asymptotically simplified to
	\begin{equation}
	\overline{B}_i= \frac{1}{c}\frac{\phi^2\tilde{\phi}}{1-\gamma^2\phi\tilde{\phi}}+o(1).
	\label{eq:BI}
	\end{equation}
	where $\phi=\phi_{0}$ or $\phi=\phi_1$. The above relation comes from the fact that, up to an error of order  $o(1)$, matrices $\boldsymbol{\Sigma}_1$ or $\boldsymbol{\Sigma}_0$ can be used interchangeably in $\phi_0$ or $\phi_1$ and in the terms involved in $\overline{B}_i$. This, in particular, implies that $\overline{B}_0$ and $\overline{B}_1$ are the same up to a vanishing error. It is noteworthy to see that the same artifice could not work for the terms $\overline{\xi}_i$ and
	$\overline{b}_i$ because the normalization, being with $\frac{1}{\sqrt{p}}$, is not sufficient to provide vanishing terms. 
	We should also mention that, although \eqref{eq:BI} takes a simpler form, we chose to work  in the simulations and when computing the consistent estimates of $\overline{B}_i$ with the expression \eqref{eq:Bi} since we found that it provides the highest accuracy.  
\end{remark}

%Therefore, this result presents a suitable way to find the optimal $\gamma$ that minimizes the asymptotic classification error. Although, it is difficult to obtain the optimal $\gamma$ in closed form, it can be determined numerically using a simple one-dimensional line search algorithm.
\subsubsection{Some Special cases} 
	a)  It is important to note that we could have considered $\left \| \bm{\mu}_0 - \bm{\mu}_1 \right \|=O(1)$. In this case, the classification error rate would still converge to a non trivial limit but would not asymptotically depend on the difference  $\left \| \bm{\mu}_0 - \bm{\mu}_1 \right \|$. This is because in this case, the difference in covariance matrices dominate that of the means and as such represent the discriminant metric that  asymptotically matters. \\
	b) Another interesting case to highlight is the one in which $\left\|\mathbf{\Sigma}_0-\mathbf{\Sigma}_1\right\|=o\left(p^{-\frac{1}{2}} \right)$. From Theorem \ref{DE_QDA} and using \eqref{eq:BI}, it is easy to show that the total classification error converges as
	%\begin{align}
	%		\epsilon - \Phi\left( - \frac{\bm{\mu}^T \mathbf{T}\bm{\mu}}{2}    \sqrt{\frac{1-\gamma^2 \phi \tilde{\phi}}{cp\gamma^2 \phi^2\tilde{\phi} + \bm{\mu}^T \mathbf{\Sigma T}^2\bm{\mu}}}\right) \xrightarrow{p}0,
	%\end{align}
	\begin{align}
	\label{RQDA_equal_cov}
	\epsilon^{R-QDA} - \Phi\left( - \frac{\bm{\mu}^T \mathbf{T}\bm{\mu}}{2\sqrt{p}}    \sqrt{\frac{c(1-\gamma^2 \phi \tilde{\phi})}{\gamma^2 \phi^2\tilde{\phi} }}\right) \probto 0,
	\end{align}
	where $\phi$, $\tilde{\phi}$ and $\mathbf{T}$ have respectively the same definitions as $\phi_i$, $\tilde{\phi}_i$ and $\mathbf{T}_i$ upon dropping the class index $i$, since quantities associated with class $0$ or class $1$ can be used interchangeably in the asymptotic regime.
	It is easy to see that in this case if  $\left \| \bm{\mu}_0 - \bm{\mu}_1 \right \|^2$ scales slower than $O\left(\sqrt{p}\right)$, classification is asymptotically impossible. This must be contrasted with the results of R-LDA, which provides non-vanishing misclassification rates for $\left \| \bm{\mu}_0 - \bm{\mu}_1 \right \|=O(1)$. This means that in this particular setting, R-QDA is asymptotically beaten by R-LDA which achieves perfect classification. \\ 
	%			In Figure \ref{fig:RQDA_equal_cov} we show the performance of R-QDA which converges to the limit predicted by equation \eqref{RQDA_equal_cov}.
	%Some hints on when LDA or QDA should be used can be then drawn.  If the difference in mean $\left \| \bm{\mu}_0 - \bm{\mu}_1
	%\right \|^2$  is less than $O(\sqrt{p})$ and the difference in spectral norm of the covariance matrices is less than $O(p^{-\frac{1}{2}}) $.
	c) When  $\left\|\mathbf{\Sigma}_0-\mathbf{\Sigma}_1\right\|_{F}=O(1)$ occurring for instance when  $\left\|\mathbf{\Sigma}_0-\mathbf{\Sigma}_1\right\|_1=O(p^{-\frac{1}{2}})$ or $\mathbf{\Sigma}_0-\mathbf{\Sigma}_1$ is of finite rank,  and $\left \| \bm{\mu}_0 - \bm{\mu}_1 \right \|^2=O(1)$,  then $\overline{b}_i\to b$ where $b$ does not depend on $i$ and as such the misclassification error probability associated with both classes converge  respectively to $1-\eta$ and $\eta$ with $\eta$ some
	probability depending solely on the statistics. The total misclassification error associated with R-QDA converges to $0.5$. \\  %starting from the limit in \eqref{RQDA_equal_cov} we can easily prove that the misclassification error probability approaches $1/2$
	% \textcolor{red}{Abla: IT IS NOT ABOUT REMOVING DETAILS FROM THE PREVIOUS VERSION, but ADD DETAILS HERE IN ORDER TO EXPLAIN TO ROMAIN WHY WE DO HAVE THIS/}
	%			
	%			associated with each class converges respectively to $1-\eta$ and $\eta$ with $\eta$ some probability depending solely on the statistics.
	%			Hence, the total mis-classification error probability associated with QDA converges to
	%			$$
	%			\epsilon \to 0.5.
	%			$$
	d) When $n_1-n_0\to\infty$, quantities $\overline{\xi}_i$ and $\overline{b}_i$  grow unboundedly as the dimension increases. This unveils that asymptotically, the discriminant score of R-QDA will keep the same sign for all observations. The classifier would thus return the same class regardless of the observation under consideration. \par 
The above remarks should help to draw some hints on when R-LDA or R-QDA should be used.  Particularly, if the Frobenius norm of $\boldsymbol{\Sigma}_0-\boldsymbol{\Sigma}_1$ is  $O(1)$, using the information on the difference between the class covariance matrices is not recommended. We should rather rely on using the information on the difference between the classes' means, or in other words favoring the use of R-LDA against  R-QDA.
%		\begin{figure}
%			\centering
%			\includegraphics[scale=.7]{Figures_Journal/RQDA_equal_cov_updated.pdf}
%			\caption{Performance in terms of the testing classification error  of the regularized QDA classifier with equal training ($n_0=n_1$), $\gamma=1$ and $\left[\mathbf{\Sigma}_0\right]_{i,j} = 0.6^{|i-j|}$ and $\mathbf{\Sigma}_1= \mathbf{\Sigma}_0 + \frac{2}{p} \mathbf{I}_p$. $\bm{\mu}_0 = \left[1,\mathbf{0}_{1\times \left(p-1\right)}\right]^T$ and $\bm{\mu}_1 = \bm{\mu}_0 + \frac{0.8}{p^{1/4}} \mathbf{1}_{p \times 1}$. The testing error is evaluated over a testing set of size $1000$ samples for both classes and averaged over $500$ realizations.}
%			\label{fig:RQDA_equal_cov}
%		\end{figure}
%\pagebreak
\section{General Consistent Estimator of the Testing Error}
\label{section:G_estimator}
In the machine learning field,  evaluating the performances of algorithms is a crucial step that not only serves to ensure their efficacy but also  to properly set the parameters involved in the design thereof, a process known in the machine learning parlance as model selection. The traditional way to evaluate performances consists in devoting a part of the training data to the design of the underlying method whereas performances are tested on the remaining data called testing
data, treated as unseen data since they do not intervene in the design step. Among the many existing computational methods that are built on these ideas are the   cross-validation \cite{Geisser,Lachenbruch}  and the bootstrap \cite{efron_83, efron_tibshirani} techniques. Despite being widely used in the machine learning community, these methods have the drawback of being computationally expensive and most importantly of relying on mere computations, which does not
lead to gain a better understanding  of the performances of the underlying algorithm. As far as LDA and QDA classifiers are considered, the results of the previous section allow to gain a deeper understanding of the classification performances with respect to the covariances and  means associated with both classes. However, as these results are expressed in terms of the unknown covariances and means, they could not be relied upon to assess the classification performances. In this section, we address this
question and provide consistent estimators of the classification performances for  both R-LDA and R-QDA classifiers that approximate in probability their asymptotic expressions.      %it is an important step to develop computational methods that evaluate the performances of algorithms.  	
%       Cross-validation \citep{Geisser,Lachenbruch}  and bootstrapping \citep{efron_83, efron_tibshirani} are among the most common estimation techniques used in classification to estimate the testing error based on the training data. Being of general use in any classification problem, they do not provide explicit relation between the estimated testing error and the problem's parameters (class' statistics and shrinkage parameter $\gamma$) apart from being computationally expensive. This calls for  the construction of an error estimator that needs not only to be accurate but also to be explicit in terms of the problem's parameters.
%	In this section, we construct a consistent estimator of the true error for both R-LDA and R-QDA classifiers under the multivariate Gaussian assumption. The estimators are designed in such a way to converge in probability to the true classification error.
\subsection{R-LDA}
The following theorem provides the expression of the class-conditional true error estimator $\epsilon_i^{R-LDA}$, for $i \in \{0,1\}$.
\begin{theorem}
	\label{LDA_estimator} Under Assumptions \ref{As:1L}-\ref{As:3L}, denote
	\begin{equation}
	\widehat{\epsilon}_i^{R-LDA}= \Phi\left(\frac{\left(-1\right)^{i+1}G\left(\hat{\bm{\mu}}_i,\hat{\bm{\mu}}_0,\hat{\bm{\mu}}_1,\mathbf{H}\right) +  \widehat{\theta}_i + \left(-1\right)^i \log \frac{\pi_1}{\pi_0}}  {\widehat{\psi}_i\sqrt{ D\left({\hat{\bm{\mu}}}_0,\hat{\bm{\mu}}_1,\mathbf{H},\widehat{\mathbf{\Sigma}}_i\right)}}   \right),
	\end{equation}
	where
	\begin{align}
	\widehat{\theta}_i & = \frac{\frac{1}{n_i} \tr \widehat{\mathbf{\Sigma}}_i \mathbf{H} }{1-\frac{\gamma}{n-2}\tr \widehat{\mathbf{\Sigma}}_i \mathbf{H} }, \\
	\widehat{\psi}_i &  = \frac{1}{1-\frac{\gamma}{n-2} \tr \widehat{\mathbf{\Sigma}}_i \mathbf{H}}.
	\end{align}
	Then,
	\begin{align*}
	\epsilon_i^{R-LDA} - \widehat{\epsilon}_i^{R-LDA} \to_{a.s.}0.
	\end{align*}
\end{theorem}
\begin{proof}
	The proof is postponed to Appendix \ref{appendix:GE_LDA}.
\end{proof}
\begin{remark}
	From Theorem \ref{LDA_estimator}, it is easy to recover the general consistent estimator of the  conditional classification error constructed in \cite{zollanvari}. In particular, in the case where $\mathbf{\Sigma}_0=\mathbf{\Sigma}_1 = \mathbf{\Sigma}$, we have the following
	\begin{align*}
	\frac{1}{n_i} \tr \widehat{\mathbf{\Sigma}} \mathbf{H} = \frac{1}{\gamma} \left(\frac{p}{n_i}-\frac{1}{n_i}\tr \mathbf{H}\right).
	\end{align*}
	Thus, upon dropping the class index $i$, $\widehat{\theta}$ is equivalent to $\widehat{\delta}$ used in \cite{zollanvari}.
\end{remark}
\subsection{R-QDA}
Based on the deterministic equivalent of the conditional classification error derived in Theorem \ref{DE_QDA} , we construct a general consistent estimator of $\epsilon_i^{R-QDA}$ denoted by $\widehat{\epsilon}_i^{R-QDA}$. The general consistent estimator of the R-QDA misclassification error is given by the following Theorem. 
%		Before providing the expression of $\widehat{\epsilon}_i^{RQDA}$ , we shall introduce some useful quantities.
%		\begin{align*}
%		\widehat{\delta}_i & = \frac{1}{\gamma} \frac{\frac{p}{n_i}-\frac{1}{n_i}\tr \mathbf{H}_i}{1-\frac{p}{n_i}+\frac{1}{n_i}\tr \mathbf{H}_i}. \\
%		\widehat{\widetilde{\delta}}_i & = \frac{1}{1+\gamma \widehat{\delta}_i} \\
%		\widehat{\delta'}_i & = -\frac{\widehat{\delta}_i}{\gamma} + \frac{1}{\gamma} \frac{\frac{1}{n_i}\tr \mathbf{H}_i \widehat{\mathbf{\Sigma}}_i\mathbf{H}_{i}}{\left(1-\frac{p}{n_i}+\frac{1}{n_i}\tr \mathbf{H}_i\right)^2}. \\
%		\widehat{\widetilde{\delta}'}_i& = - \left(	\widehat{\delta}_i +\gamma \widehat{\delta'}_i \right)\widehat{\widetilde{\delta}}_i ^2.
%		%	 	 \left(\mathbf{I}_p + \mathbf{H}_0 \mathbf{H}_1  -\mathbf{H}_0 - \mathbf{H}_1 \right)
%		\\
%		\widehat{\phi}_i & = - \frac{\widehat{\delta'}_i}{\widehat{\widetilde{\delta}}_i + \gamma 	\widehat{\widetilde{\delta}'}_i}
%		\end{align*}
%		We are now in position to state the following theorem.
\begin{theorem}
	\label{QDA_estimator}
	Under Assumptions \ref{As:1}-\ref{As:4}, define
	\begin{equation}
	\widehat{\epsilon}_i^{R-QDA} = \Phi \left(\left(-1\right)^i \frac{\widehat{\xi}_i - \widehat{b}_i}{\sqrt{2 \widehat{B}_i}}\right),
	\end{equation}
	Then,
	\begin{align*}
	\widehat{\epsilon}_i^{R-QDA} - \epsilon_i^{R-QDA}  \probto 0.
	\end{align*}
	where
	\begin{align*}
	\widehat{\xi}_i & = -\frac{1}{\sqrt{p}}\log \frac{|\mathbf{H}_0|}{|\mathbf{H}_1|} + \frac{\left(-1\right)^{i+1}}{\sqrt{p}} \left(\hat{\bm{\mu}}_0-\hat{\bm{\mu}}_1\right)^T \mathbf{H}_{1-i} \left(\hat{\bm{\mu}}_0-\hat{\bm{\mu}}_1\right) .\\
	\widehat{\delta}_i & = \frac{1}{\gamma} \frac{\frac{p}{n_i}-\frac{1}{n_i}\tr \mathbf{H}_i}{1-\frac{p}{n_i}+\frac{1}{n_i}\tr \mathbf{H}_i}. \\ 
	\widehat{b}_i & = \frac{\left(-1\right)^i}{\sqrt{p}} \tr \widehat{\mathbf{\Sigma}}_i\mathbf{H}_{1-i} + \frac{\left(-1\right)^{i+1}n_i}{\sqrt{p}} \widehat{\delta}_i. \\
	%			\widehat{B}_i & = c \left(\frac{\widehat{\phi}_0}{1-\gamma^2\widehat{\phi}_0 \widehat{\widetilde{\phi}}_0}   + \frac{\widehat{\phi}_1}{1-\gamma^2\widehat{\phi}_1 \widehat{\widetilde{\phi}}_1} \right) + 	c. f\left(\widehat{\mathbf{\Sigma}}_i^2, \mathbf{H}_{1-i}\right) - \frac{1}{n_i} \tr \widehat{\mathbf{\Sigma}}_if\left(\widehat{\mathbf{\Sigma}}_i, \mathbf{H}_{1-i}\right)  \\ & +  \frac{c}{1-\widehat{\phi}_{1-i}\widehat{\tilde{\phi}}_{1-i}} \left(\frac{1}{n_i}\tr \widehat{\mathbf{\Sigma}}_i\mathbf{H}_{1-i}-f\left(\widehat{\mathbf{\Sigma}}_i,\mathbf{H}_{1-i}\right)\right)^2
	%			- \frac{2}{p \gamma \widehat{\widetilde{\delta}}_i} \left[ \tr \mathbf{H}_{1-i}\widehat{\mathbf{\Sigma}}_i - \frac{1}{\gamma \widehat{\widetilde{\delta}}_i} \tr \mathbf{H}_{1-i} \left(\mathbf{I}_p-\mathbf{H}_{i}\right) \right]
	\widehat{B}_i & = \left(1+\gamma \widehat{\delta}_i\right)^4 \frac{1}{p}\tr \widehat{\mathbf{\Sigma}}_i \mathbf{H}_i\widehat{\mathbf{\Sigma}}_i \mathbf{H}_i - \frac{n_i}{p}\widehat{\delta}_i^2 \left(1+\gamma \widehat{\delta}_i\right)^2  \\ & + \frac{1}{p}\tr \widehat{\mathbf{\Sigma}}_i \mathbf{H}_{1-i}\widehat{\mathbf{\Sigma}}_i \mathbf{H}_{1-i} - \frac{n_i}{p} \left(\frac{1}{n_i}\tr \widehat{\mathbf{\Sigma}}_i \mathbf{H}_{1-i}\right)^2 \\
	& -2 \left(1+\gamma \widehat{\delta}_i\right)^2 \frac{1}{p} \tr \widehat{\mathbf{\Sigma}}_i \mathbf{H}_i \widehat{\mathbf{\Sigma}}_i \mathbf{H}_{1-i} +  \widehat{\delta}_i \left(1+\gamma \widehat{\delta}_i\right) \frac{2}{p} \tr \widehat{\mathbf{\Sigma}}_i \mathbf{H}_{1-i}.
	\end{align*}
	%			with
	%			\begin{align*}
	%			f\left(\mathbf{A},\mathbf{H}_i\right) = \frac{\frac{2}{n_i} \tr \mathbf{A}\mathbf{H}_i^2 + \frac{1}{n_i} \tr \mathbf{AH}_i \left( 1 - \sqrt{1 + \frac{4}{1-\gamma^2 \widehat{\phi}_i \widehat{\tilde{\phi}}_i} \left(\frac{1}{n_i}\tr \mathbf{H}_i^2 - \frac{1}{n_i} \tr \mathbf{H}_i\right)}\right)}{1 + \sqrt{1 + \frac{4}{1-\gamma^2 \widehat{\phi}_i \widehat{\tilde{\phi}}_i} \left(\frac{1}{n_i}\tr \mathbf{H}_i^2 - \frac{1}{n_i} \tr \mathbf{H}_i\right)}}.
	%			\end{align*}
\end{theorem}
\begin{proof}
	See Appendix \ref{appendix:QDA_estimator}.
\end{proof}
%		\begin{remark}
%			It is worth mentioning that deriving the deterministic equivalents in Theorem \ref{DE_QDA}  played a key role in obtaining a general consistent estimator for the R-QDA misclassification rate, unlike R-LDA where both the derived general estimator and the determinsitic equivalent of the misclassification rate were obtained independently as shown in the proof.
%		\end{remark}
\subsection{Validation with synthetic data}	
We validate the results of Theorems \ref{LDA_estimator} and \ref{QDA_estimator} by examining the accuracy of the proposed general consistent estimators in terms of the RMS defined as follows\footnote{Since both synthetic and real data are of finite dimensions, we keep vanishing parts of the estimator when implementing the code in our simulations.}
\begin{equation}
\text{RMS} \left(\widehat{\epsilon}\right) = \sqrt{\text{Bias}\left(\widehat{\epsilon}\right)^2 + \text{var}\left(\widehat{\epsilon}-\epsilon\right)},
\end{equation}
where 
\begin{equation}
\label{bias_rms}
\text{Bias}\left(\widehat{\epsilon}\right) = \mathbb{E} \left[\widehat{\epsilon}-\epsilon \right].
\end{equation}

We also compare the proposed general consistent estimator (that we denote by the G-estimator) for both R-LDA and R-QDA with the following benchmark estimation techniques fully described in \cite{Dougherty}
\begin{itemize}
	%	\item Resubstitution (Training error).
	\item 5-fold cross-validation with 5 repetitions (5-CV).
	\item 0.632 bootstrap (B632).
	\item 0.632+ bootstrap (B632+).
	\item Plugin estimator consisting of replacing the statistics in the deterministic equivalents by their corresponding sample estimates.
	%			\item General consistent estimator proposed in \cite{zollanvari} for R-LDA with common covariance matrices.
\end{itemize}
%        It is noteworthy to mention that the above estimators require testing data to predict the performances, unlike the proposed G-estimators that estimate the performances using solely the training data.  
\begin{figure}
	\centering
		\includegraphics[scale=.4]{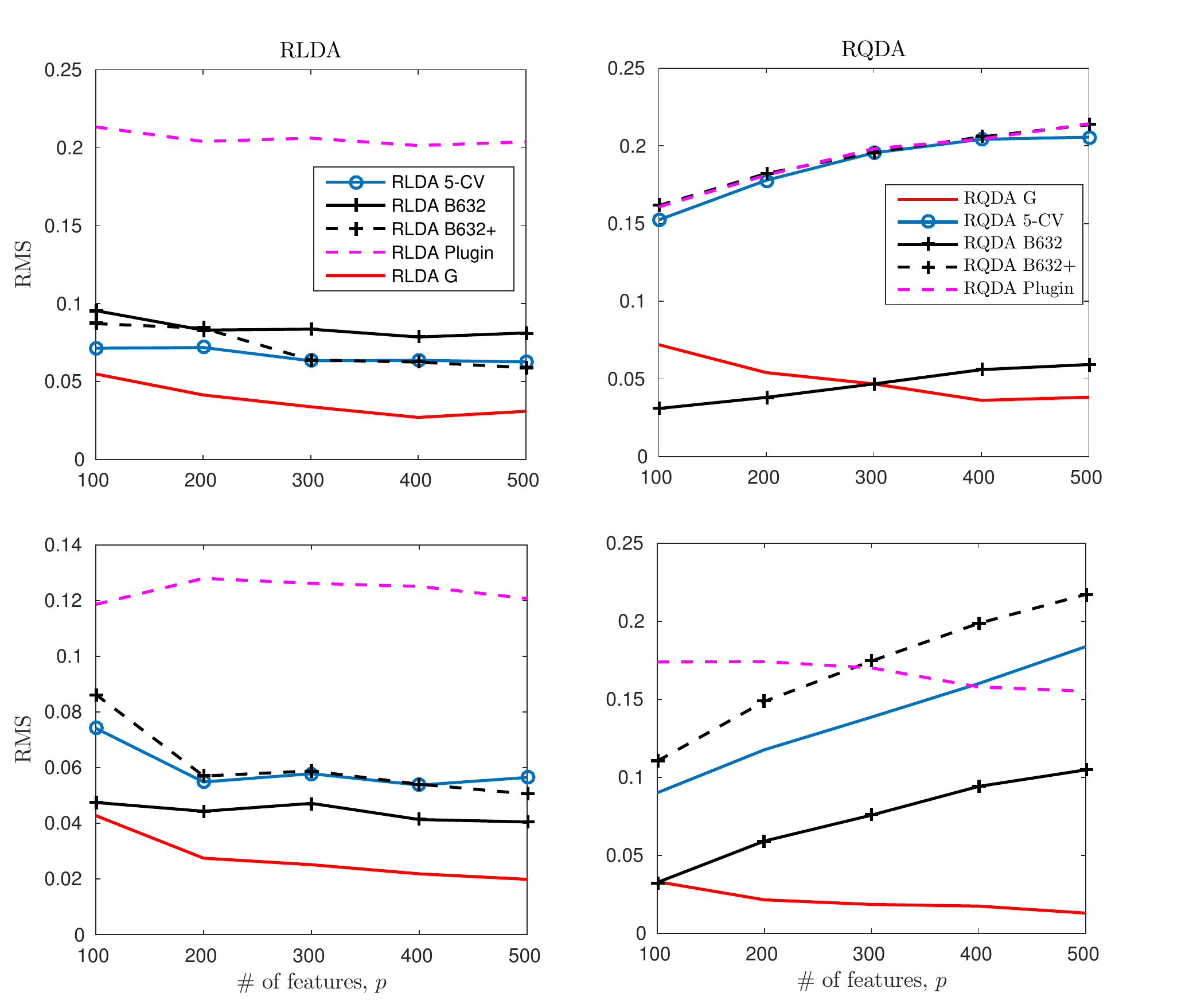}
	\caption{RMS performance of the proposed general consistent estimators (RLDA G and RQDA G) compared with the benchmark estimation techniques. We consider equal training size ($n_0=n_1$), $\gamma=1$ and $\left[\mathbf{\Sigma}_0\right]_{i,j} = 0.6^{|i-j|}$, $\mathbf{\Sigma}_1= \mathbf{\Sigma}_0 + 3 \mathbf{S}_p$, $\bm{\mu}_0 = \left[1,\mathbf{0}_{1\times \left(p-1\right)}\right]^T$ and $\bm{\mu}_1 = \bm{\mu}_0 + \frac{0.8}{\sqrt{p}} \mathbf{1}_{p \times 1}$. The first row treats the
		case where $n_0 = p/2$ whereas the second row treats the case $n_0 = p$. The testing error is evaluated over a testing set of size $1000$ samples for both classes and averaged over $1000$ realizations.}
	\label{fig:RDA_RMS}
\end{figure}
\begin{figure}
	\centering
	\includegraphics[scale=.45]{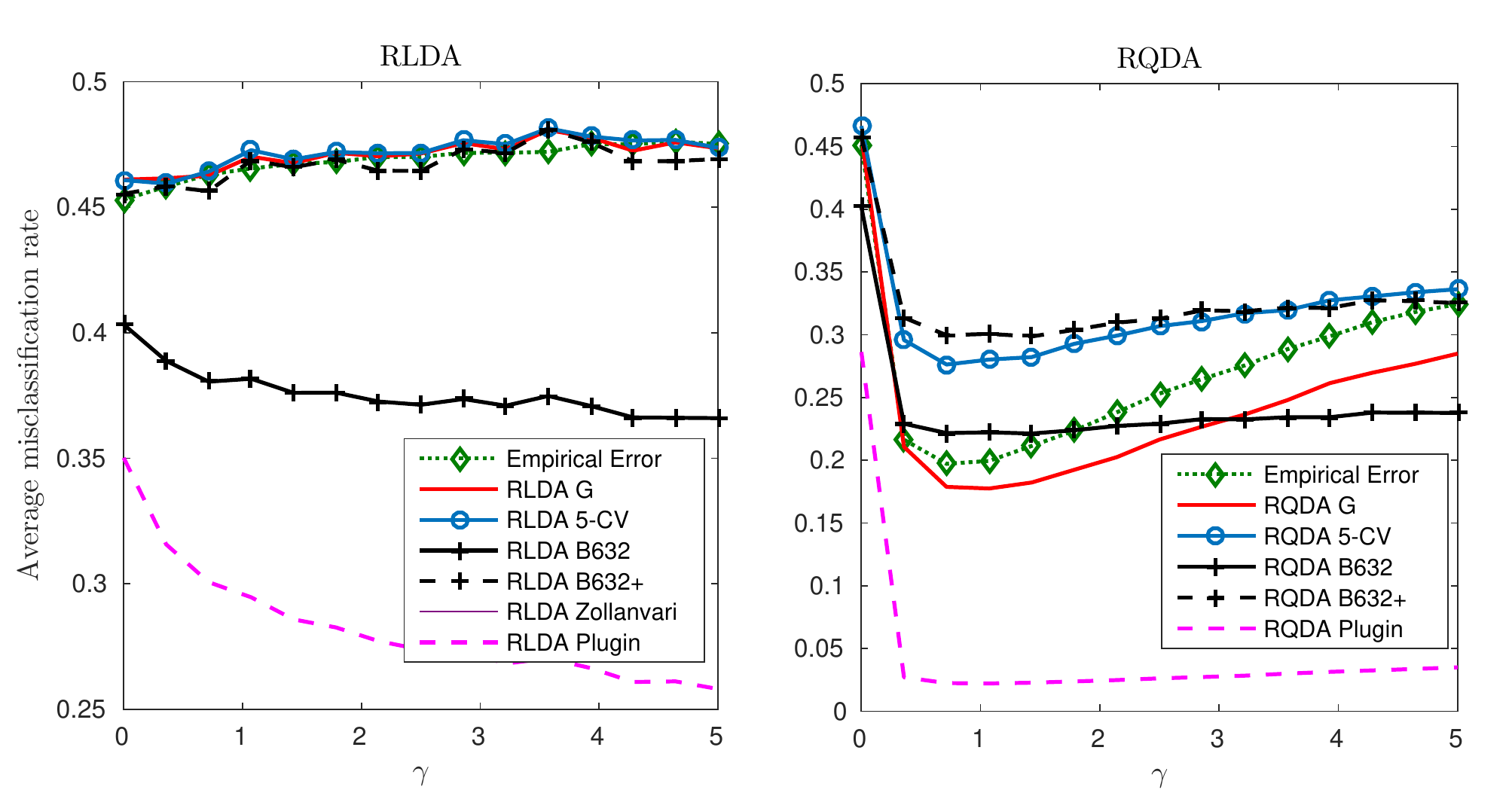}
	\caption{Average misclassification rate versus the regularization parameter $\gamma$. We consider $p=100$ features with equal training size ($n_0=n_1=p$), $\left[\mathbf{\Sigma}_0\right]_{i,j} = 0.6^{|i-j|}$, $\mathbf{\Sigma}_1= \mathbf{\Sigma}_0 + 3 \mathbf{S}_p$, $\bm{\mu}_0 = \left[1,\mathbf{0}_{1\times \left(p-1\right)}\right]^T$ and $\bm{\mu}_1 = \bm{\mu}_0 + \frac{0.8}{\sqrt{p}} \mathbf{1}_{p \times 1}$. The testing error is evaluated over a testing set of size $1000$ samples for both classes and averaged over $1000$ realizations.}
	\label{fig:estmates_vs_gamma}
\end{figure}
In Figure \ref{fig:RDA_RMS}, we observe that the naive plugin estimator has the worst RMS performance for both classifiers in most cases. This is simply explained by the fact that when $p$ and $n_i$ have the same order of magnitude, the sample estimates are inaccurate which leads to a medicocre RMS performance.  %does not converge to the theoretical limit which lead to a mediocre RMS performance.
On another front, it is clear for both settings ($n_i = p/2$ and $n_i = p$) that the proposed G-estimator achieves a suitable RMS performance beating 5-fold cross validation and the bootstrap. 
%        Another important observation is that the G-etimator proposed in \citep{zollanvari} is close to our proposed G-estimator for R-LDA. This is in agreement with the result of Corollary \ref{common_cov_remark}  since  we have chosen both covariance matrices to be close in terms of spectral norm.   However, it is worth mentioning that in the general setting ($\|\mathbf{\Sigma}_0-\mathbf{\Sigma}_1\| = \bigo\left(1\right)$), the proposed G estimator in \cite{zollanvari} might result in a bad RMS performance since it is build on the assumption of equal covariance matrices.\\  
In Figure \ref{fig:estmates_vs_gamma}, we examine the performance of the different error estimators against the regularization parameter. As shown in the Figure \ref{fig:estmates_vs_gamma}, R-LDA is less vulnerable to the choice of $\gamma$ as compared to R-QDA where the choice of $\gamma$ tends to have a higher influence on the performance. Also, for both classifiers, the proposed G-estimator is able to track the empirical error and thus permits to predict the optimal regularizer with high accuracy.
\section{Experiments with real data}
\label{section:experiments}
\begin{figure}
	\centering
	\includegraphics[scale=.45]{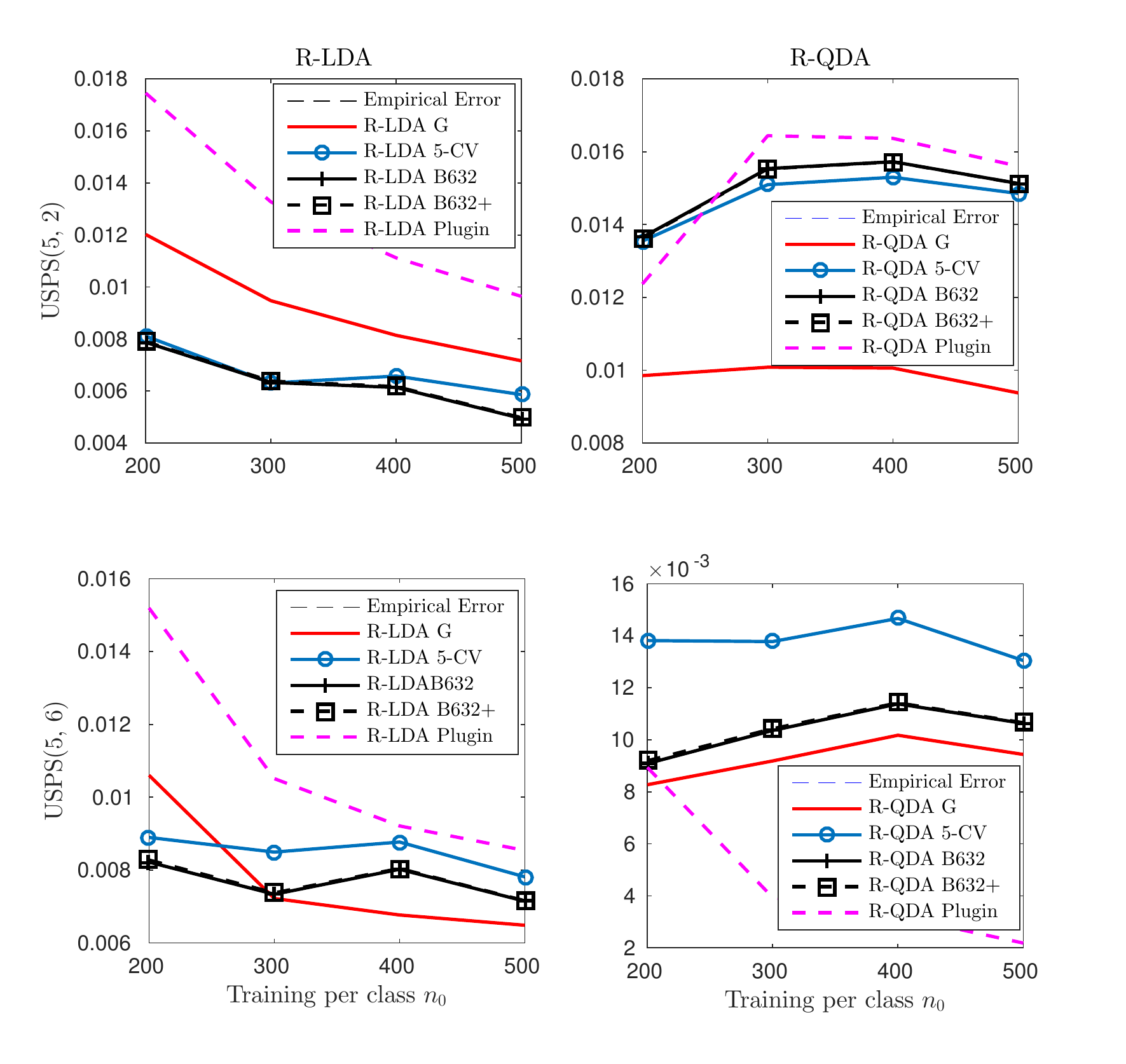}
	\caption{RMS performance of the proposed general consistent estimators (R-LDA G and R-QDA G) compared with the benchmark estimation techniques. We consider equal training size ($n_0=n_1$)  and $\gamma=1$. The first row gives the performance for the USPS data with digits (5, 2) whereas the second row considers the digits (5, 6).}
	\label{fig:USPS_RMS}
\end{figure}
In this section, we examine the performance of the proposed G estimator on the public USPS dataset of handwritten digits \cite{mnist}. The dataset consists of 7291 training samples of $16 \times 16$ grayscale images ($p=256$ features) and 2007 testing images \url{http://www.csie.ntu.edu.tw/~cjlin/libsvmtools/datasets/multiclass.html#usps} \footnote{All the results of this paper can be reproduced using our Julia codes available in \url{https://github.com/KhalilElkhalil/Large-Dimensional-Discriminant-Analysis-Classifiers-with-Random-Matrix-Theory}}.	
%	For that, we use the Julia package in \url{https://github.com/johnmyleswhite/MNIST.jl}  which permits to access the data set consisting of $784 \times 60,000$ matrix of images ($p = 784$ features) in the training set and $784 \times 10,000$ matrix of images in the testing set.
First, we examine the RMS performance of the different error estimators on the data for different values of the training size and for different class labels. The RMS is determined by averaging the error over a number of training sets randomly selected from the total training dataset. As shown in Figure \ref{fig:USPS_RMS}, the proposed G-estimator gives a good RMS performance especially for R-QDA where it can actually outperform state-of-the art estimators such as cross validation and Bootstrap. Moreover, it is clear that the plugin estimator has a higher RMS performance for most of the considered scenarios.

\begin{figure}
	\centering
	\includegraphics[scale=.45]{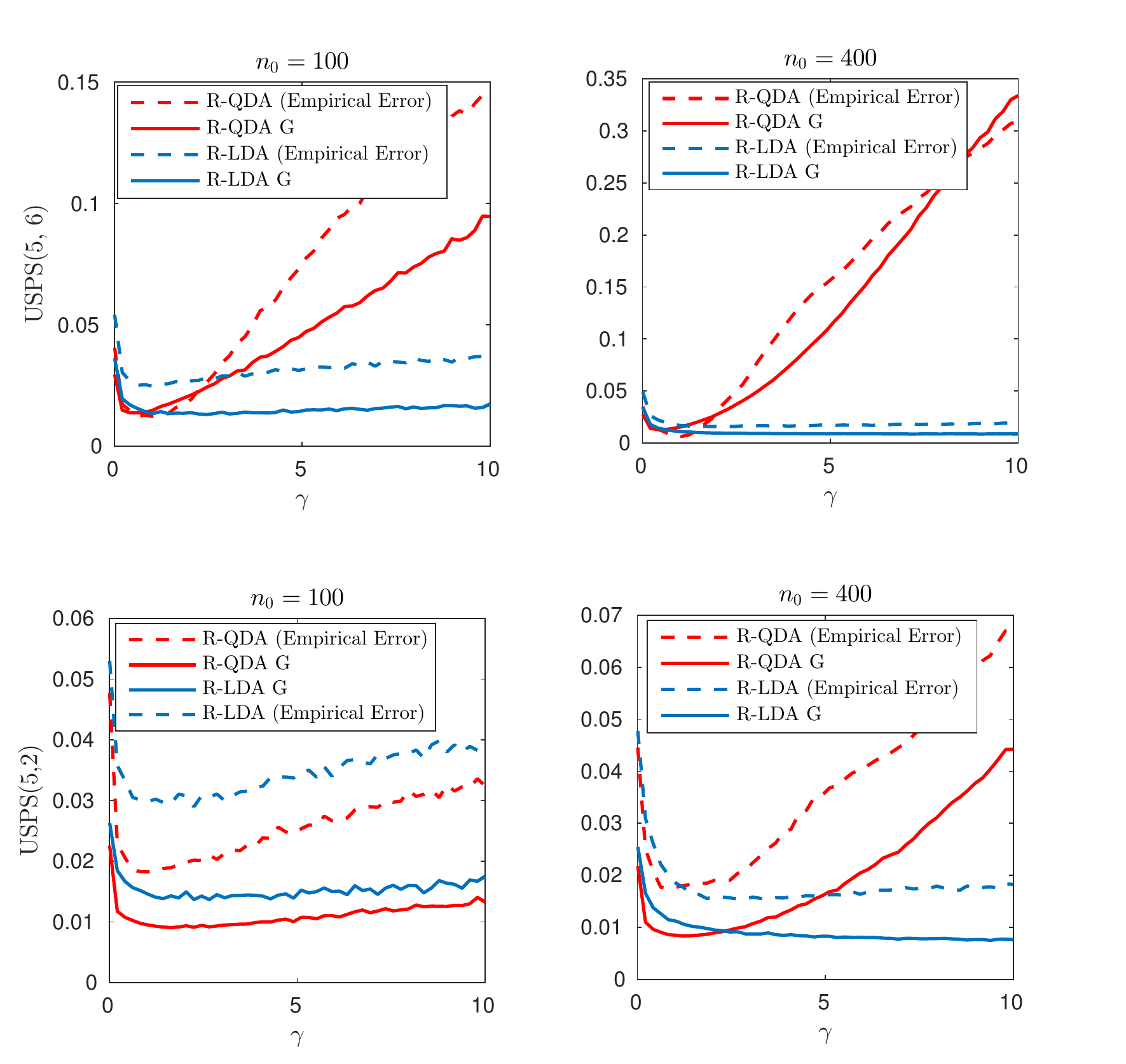}
	\caption{Average misclassification rate versus the regularization parameter $\gamma$ of the USPS dataset for different instances of digits and assuming equal training size ($n_0=n_1$). The solid red line refers to the performance of the proposed G-estimator whereas the dotted black line refers to the empirical error computed using the testing data.}
	\label{fig:MNIST_vs_gamma}
\end{figure}	
Now, we turn our attention to finding the \emph{optimal} $\gamma$ that results in the minimum testing error. 
Since the construction of the G-estimator is heavily based on the Gaussian assumption of the data, picking the regularizer that minimizes the estimated error using the G-estimator will not necessarily minimize the error computed on the testing data for USPS. One straightforward approach is  to compute the testing error for all possible value of $\gamma$ in the range $\left(0, \infty\right)$, then pick the regularizer resulting in the minimum error\footnote{Usually, we perform cross validation or Bootstrap to have an estimate of the error from the training set, but since we have enough testing data we rely on the testing error for the USPS dataset.}. Obviously, this approach is far from being practical and is simply unfeasible. Motivated by this issue, we propose a two-stage optimization explained as follows.
\subsection{Two-stage optimization}
Although real data are far from being Gaussian, the proposed G-estimator can be used to have a glimpse on the optimal regularizer. More specifically, we can use the G-estimator to determine the interval  in which  the optimal regularizer is likely to belong, then we perform cross validation (or testing if we have enough testing data) for multiple values of $\gamma$ inside that interval and finally pick the value that results on the minimum cross-validation error (or testing error). As seen in Figure \ref{fig:MNIST_vs_gamma}, both the R-LDA and the R-QDA G-estimators are able to mimic the real behavior of the testing error when $\gamma$ varies for both situations when $n_0 < p$ and $n_0 > p$. Similarly to synthetic data, Figure \ref{fig:MNIST_vs_gamma} also shows how R-QDA is vulnerable to the choice of $\gamma$ which justifies the need to find a \emph{good} regularization parameter $\gamma$. In Table \ref{table:optimal_gamma}, we provide numerical values for the output of the two-step optimization using a confidence interval $\left(\left(\widehat{\gamma}_G - \frac{2}{\sqrt{p}}\right)^+, \widehat{\gamma}_G + \frac{2}{\sqrt{p}}\right)$ \footnote{$x^{+} = \max(x, 0)$, for $x \in \mathbb{R}$.} with a uniform grid of 50 points where $\widehat{\gamma}_G$ is a minimizer of the G-estimator built based on the Gaussian assumption.

\begin{table}[]
	\centering
	\caption{Estimates of the optimal regularizer using the two-step optimization method with their corresponding testing error.}
	\label{table:optimal_gamma}
	\begin{tabular}{l|l|l||l|l|}
		\cline{2-5}
		&$\widehat{\gamma}^{R-LDA}$ &$\epsilon_{testing}^{R-LDA}$  &$\widehat{\gamma}^{R-QDA}$& $\epsilon_{testing}^{R-QDA}$ \\ \hline
		\multicolumn{1}{|l|}{USPS(5, 2), $n_0=100$} & 3.54 & 0.0307 & 0.896 & 0.0111  \\ \hline
		\multicolumn{1}{|l|}{USPS(5, 2), $n_0=400$} & 20.98 &  0.0251 & 1.049 & 0.0223 \\ \hline
		\multicolumn{1}{|l|}{USPS(5, 6), $n_0=100$} & 4.753  & 0.0272 & 0.567 & 0.0272 \\ \hline
		\multicolumn{1}{|l|}{USPS(5, 6), $n_0=400$} & 12.1572  & 0.01818  & 0.562  & 0.009 \\ \hline
	\end{tabular}
\end{table}
%\section{Limitations}
\section{Concluding remarks}
\label{section:conclusion}
In this work, we carried out a performance analysis of the asymptotic misclassification rate for R-LDA and R-QDA based classifiers in the regime where the dimension of the training data and their number grow large with the same pace. By leveraging results from random matrix theory, we identify the growth rate regimes in which R-LDA and R-QDA result in non trivial mis-classification rates. These latter are characterized in the asymptotic regime by closed-form expressions reflecting the
impact of the means and covariances of each class on the classification performance.  Several insights are drawn from our results, which can guide the practitioners to choose the best classifier according to the setting into consideration. Particularly, we highlight that R-LDA achieves perfect classification rates when the difference in the mean vectors is higher than $O(p^{\alpha})$ for $\alpha >0$. The R-QDA, on the other hand, results in  perfect classification when the
number of significant eigenvalues of the difference between both covariance matrices scales larger than $O(\sqrt{p})$ or the difference in means is higher than $O(\sqrt{p})$. Such findings reveal a fundamental difference in the way the information about the classes means and covariances are leveraged by both methods. Unlike the R-LDA which tends to leverage only the information about the means, the R-QDA exploits both discriminative statistics, but requires a higher order in the mean difference so that it is reflected in
the classification performances.            
%	This work permits to derive closed form expressions of the asymptotic misclassification rate for both classifiers R-LDA and R-QDA in the regime where the data dimension and the training size tend to infinite at the same pace. The obtained results allow to draw new insights on the performance of the underlying classifiers. In particular, they permit to identify scenarios in which the use of one classifier would ultimately result in a better performance than the other. Moreover, we have shown that under some assumptions, counter-intuitive phenomenon are observed such as the asymptotic independence of the R-QDA performance on the class' separation in terms of the statistical means. Such an observation would draw insights on the limitations of such classifies under the taken assumptions. 		
%	The obtained results can be further useful since they allow to shed light on the way we should optimize the regularization parameter.  
\par 
This work can be extended to study more sophisticated classifiers such as kernel Fisher discriminant analysis \cite{kernel_fisher} or kernel discriminant analysis \cite{kernel_da}  where data is mapped into a feature space through a non-linear kernel prior to the application of the classifier. The extension is though not trivial since the non-linear mapping of data makes it difficult to examine the misclassification probability in closed-form. 
% Another interesting future research direction, mostly driven by the established sensitivity of R-QDA to estimation noises, is to investigate a solution for this problem by adapting its discriminant score.  

%    Another interesting future research direction, that comes as a natural extension of the present work, is to investigate the performance of regularized discriminant analysis based classifiers proposed in  \citep{friedman_JASA1989} which, in using two regularization parameters,  defines a larger family of classifiers that  encompasses the R-LDA and R-QDA as special cases.   
%One interesting future research direction is to consider the general framework of regularized discriminant analysis where we can consider  the general class of discriminant analysis classifiers in   \citep{friedman_JASA1989} that encompasses as a special case the R-LDA and R-QDA.  %more general involved  estimates of the covariance matrix for each class following the work of \citep{friedman_JASA1989}. 
\bibliographystyle{IEEEtran}
\bibliography{References}
\begin{appendices}
	\section{Proof of Theorem \ref{DE_LDA}} 
	\label{appendix:theorem1}
	
	\subsection{Notations}
	Through this appendix, the following notations are used. For $i\in\left\{0,1\right\}$, we let ${\bf X}_i\in\mathbb{R}^{p\times n_i}$ the matrix of $n_i$ observations associated with class $i$. Thus, there exists ${\bf Y}_i= \mathbf{\Sigma}_i^{1/2} \mathbf{Z}_i$  such that ${\bf X}_i=\boldsymbol{\Sigma}_i^{\frac{1}{2}}{{\bf Z}_i}+ \boldsymbol{\mu}_i{\bf 1}_{n_i}^{T}$ where ${\bf 1}_{n_i}\in\mathbb{R}^{n_i}$  is the vector of all ones, 
	and  ${\bf Z}_i=\left[\mathbf{z}_{i,1},\cdots,\mathbf{z}_{i,n_i}\right] \in \mathbb{R}^{p \times n_i}$ where $\mathbf{z}_{i,j}$ are independent random vectors with standard multivariate Gaussian distribution. We define the following resolvent matrix:
	\begin{equation}
	\label{resolvent_matrix}
	\mathbf{Q}\left(z\right) = \left( \frac{\mathbf{Y}_0\mathbf{Y}_0^T}{p} + \frac{\mathbf{Y}_1\mathbf{Y}_1^T}{p} - z \mathbf{I}_p \right)^{-1}
	\end{equation}
	the behavior of which has  been extensively studied in \cite[Proposition 5]{gram_mixture}. Particularly, it was shown that under assumptions \ref{As:1L}, \ref{As:2L} and \ref{As:4L}, $\mathbf{Q}\left(z\right)$ is equivalent to a deterministic matrix $\bar{\mathbf{Q}}\left(z\right)$ or $\mathbf{Q}\left(z\right)\leftrightarrow \bar{\mathbf{Q}}\left(z\right)$ in the sense that 
	
	\begin{equation}
	\label{Q_bar}
	\begin{split}
	& \frac{1}{p} \tr \mathbf{M}\left(\mathbf{Q}\left(z\right)-\bar{\mathbf{Q}}\left(z\right)\right) \probto 0. \\
	&\mathbf{u}^T \left(\mathbf{Q}\left(z\right)-\bar{\mathbf{Q}}\left(z\right)\right) \mathbf{v} \probto 0,
	\end{split}
	\end{equation}
	for all deterministic matrices $\mathbf{M}$ of bounded spectral norms and all deterministic vectors $\mathbf{u}$ and $\mathbf{v}$ of bounded euclidean norms.
	Moreover, it has been shown in \cite[Proposition 6]{gram_mixture} that
	\begin{align}
	\label{Q_tilde}
	\mathbf{Q}\left(z\right)\mathbf{\Sigma}_i\mathbf{Q}\left(z\right) \leftrightarrow \widetilde{\mathbf{Q}}_i\left(z\right), \; \textnormal{for} \: i \in \{0,1\}.
	\end{align}
	where $\widetilde{\bf Q}_i$ is given in \eqref{eq:Qtildei}.
	Based on the results of (\ref{Q_bar}) and (\ref{Q_tilde}), we successively prove (\ref{G_i}) and (\ref{D_i}).
	\subsection{Proof of (\ref{G_i})}
	With the aforementioned notations at hand, it is easy to show that  $\widehat{\mathbf{\Sigma}}_i$ can be expressed as
	%For $i\in \{0,1\}$, let $\mathbf{Y}_i = \mathbf{\Sigma}_i^{1/2}\mathbf{Z}_i$ where $\mathbf{Z}_i \in \mathbb{R}^{p\times n_i}$ contains i.i.d zero mean unit variance Gaussian entries and let $\mathbf{u}_i = \frac{\mathbf{1}_{n_i\times 1}}{n_i}$. Then by simple manipulations, we can show that $\widehat{\mathbf{\Sigma}}_i$ writes as
	\begin{align*}
	\widehat{\mathbf{\Sigma}}_i = \frac{1}{n_i-1}\left( \mathbf{Y}_i\mathbf{Y}_i^T - \mathbf{Y}_i \frac{\mathbf{1}_{n_i}\mathbf{1}_{n_i}^T}{n_i} \mathbf{Y}_i^T\right).
	\end{align*}
	Let $ \frac{\mathbf{1}_i\mathbf{1}_i^T}{n_i} = \mathbf{O}_i \mathbf{E}_i\mathbf{O}_i^T$, be the eigenvalue decomposition of  $\frac{\mathbf{1}_i\mathbf{1}_i^T}{n_i}$ where $\mathbf{E}_i =\diag\left(\left[1, \mathbf{0}_{(n_i-1)\times 1} \right]\right)$ and $\mathbf{O}_i$ is a $n_i\times n_i$ orthogonal matrix with first column $\frac{1}{\sqrt{n_i}}{\bf 1}_{n_i}$. Let $\tilde{\bf Y}_i={\bf Y}_i{\bf O}_i$. Hence 
	\begin{equation}
	\label{Sigma_simplified}
	\begin{split}
	\widehat{\mathbf{\Sigma}}_i & = \frac{1}{n_i-1} \mathbf{Y}_i\mathbf{O}_i\mathbf{O}_i^T\mathbf{Y}_i^T - \frac{1}{n_i-1} {\mathbf{Y}_i\mathbf{O}_i}  \mathbf{E}_i\mathbf{O}_i^T \mathbf{Y}_i^T \\
	& = \frac{1}{n_i-1} \widetilde{\mathbf{Y}}_i\widetilde{\mathbf{Y}}_i^T - \frac{1}{n_i-1} \widetilde{\mathbf{y}}_{i,1}\widetilde{\mathbf{y}}_{i,1}^T,
	\end{split}
	\end{equation}
	where $\widetilde{\mathbf{y}}_{i,1}$ being the first column of $\widetilde{\mathbf{Y}}_i$. Since the Gaussian distribution is invariant to multiplication by a unitary matrix, $\tilde{\bf Y}_i$ has the same distribution as ${\bf Y}_i$.  %By the invariance of the Gaussian distribution by multiplication by a unitary matrix, $\widetilde{\mathbf{Y}}_i$ has the same distribution as $\mathbf{Y}_i$.
	and as such matrix ${\bf H}$ can be expressed as:
	\begin{equation}
	\label{H_new}
	\begin{split}
	\mathbf{H} & = \Biggl[\mathbf{I}_p + \frac{\gamma}{n-2}\widetilde{\mathbf{Y}}_0\widetilde{\mathbf{Y}}_0^T + \frac{\gamma}{n-2}\widetilde{\mathbf{Y}}_1\widetilde{\mathbf{Y}}_1^T   - \frac{\gamma}{n-2}\widetilde{\mathbf{y}}_{0,1}\widetilde{\mathbf{y}}_{0,1}^T  \\ & - \frac{\gamma}{n-2}\widetilde{\mathbf{y}}_{1,1}\widetilde{\mathbf{y}}_{1,1}^T
	\Biggr]^{-1}.
	\end{split}
	\end{equation}
	Then, the following relation holds for $z=-\frac{n}{p\gamma}$%-\frac{n-2}{p\gamma} = -\frac{n}{p\gamma}+O(p^{-1}) $
	\begin{align*}
	\mathbf{H} = -z \mathbf{Q}\left(z\right) +O_{\|.\|}(p^{-1}),
	\end{align*}
	where $O_{\|.\|}(p^{-1})$ refers to a matrix whose spectral norm is $O(p^{-1})$. 
	We are now ready to handle the term $G\left(\bm{\mu}_i,\hat{\bm{\mu}}_0,\hat{\bm{\mu}}_1,\mathbf{H}\right)$. To start, we first express it as
	\begin{align*}
 	G\left(\bm{\mu}_i,\hat{\bm{\mu}}_0,\hat{\bm{\mu}}_1,\mathbf{H}\right)  &=\left(\frac{(-1)^i}{2}\boldsymbol{\mu}^{T}-\frac{1}{2n_0}{\bf 1}^{T}{\bf Y}_0-\frac{1}{2n_1}{\bf 1}_{n_1}^{T}{\bf Y}_1\right) \\ & \times {\bf H}(\frac{1}{n_0}{\bf Y}_0 {\bf 1}_{n_0}-\frac{1}{n_1}{\bf Y}_1{\bf 1}_{n_1}+\boldsymbol{\mu}).
	\end{align*}
	and thus can be expanded as
	%    It now remains to expand $G\left(\bm{\mu}_i,\bar{\mathbf{x}}_0,\bar{\mathbf{x}}_1,\mathbf{H}\right)$ as follows
	\begin{align*}
	&G\left(\bm{\mu}_i,\hat{\bm{\mu}}_0,\hat{\bm{\mu}}_1,\mathbf{H}\right)  = \frac{\left(-1\right)^i}{2}\bm{\mu}^T \mathbf{H}\bm{\mu} + \frac{\left(-1\right)^i}{2n_0} \bm{\mu}^T \mathbf{H}\mathbf{Y}_0^T \mathbf{1}_{n_0} \\ &  + \frac{\left(-1\right)^{i+1}}{2n_1} \bm{\mu}^T \mathbf{H}\mathbf{Y}_1^T\mathbf{1}_{n_1} - \frac{1}{2n_0}\bm{\mu}^T \mathbf{H}\mathbf{Y}_0\mathbf{1}_{n_0} \\
	& -\frac{1}{2n_0^2}\mathbf{1}_{n_0}^T \mathbf{Y}_0^T \mathbf{HY}_0\mathbf{1}_{n_0} + \frac{1}{2n_0n_1}\mathbf{1}_{n_0}^T\mathbf{Y}_0^T\mathbf{HY}_1\mathbf{1}_{n_1} \\ & - \frac{1}{2n_1}\bm{\mu}^T \mathbf{H}\mathbf{Y}_1\mathbf{1}_{n_1}-\frac{1}{2n_0n_1}\mathbf{1}_{n_0}^T\mathbf{Y}_0^T \mathbf{HY}_1\mathbf{1}_{n_1}\\
	& + \frac{1}{2n_1^2}\mathbf{1}_{n_1}^T \mathbf{Y}_1^T \mathbf{HY}_1\mathbf{1}_{n_1}.
	\end{align*}
	It follows from \eqref{H_new}  that $\tilde{\bf y}_{0,1}=\frac{1}{\sqrt{n}_0}{\bf Y}_0{\bf 1}$ and $\tilde{\bf y}_{1,1}=\frac{1}{\sqrt{n}_1}{\bf Y}_1{\bf 1}$ are independent of ${\bf H}$. The following convergence holds thus true
	%	Exploiting the independence of $\widetilde{\mathbf{y}}_{i}$ and $\mathbf{H}$ from \eqref{H_new} , it is easy to show the following
	\begin{align*}
	&\frac{1}{n_0}\bm{\mu}^T \mathbf{H}\mathbf{Y}_0\mathbf{1}_{n_0} \asto 0. \\
	& \frac{1}{n_1}\bm{\mu}^T \mathbf{H}\mathbf{Y}_1\mathbf{1}_{n_1} \asto 0. \\
	& \frac{1}{n_0n_1}\mathbf{1}_{n_0}^T\mathbf{Y}_0^T\mathbf{HY}_1\mathbf{1}_{n_1} \asto 0.
	\end{align*}
	On the other hand, we have
	\begin{align}
	\bm{\mu}^T \mathbf{H}\bm{\mu}= -z \bm{\mu}^T \mathbf{Q}\left(z\right)\bm{\mu}+o(1).
	\end{align}
	and thus from (\ref{Q_bar})
	\begin{align}
	\bm{\mu}^T \mathbf{H}\bm{\mu} + z \bm{\mu}^T \bar{\mathbf{Q}}\left(z\right)\bm{\mu} \probto 0.
	\end{align}
	Moreover,
	\begin{align*}
	\frac{1}{n_i^2} \mathbf{1}_{n_i}^T\mathbf{Y}_i^T \mathbf{H}\mathbf{Y}_i\mathbf{1}_{n_i}  %\tr \mathbf{Y}_i\mathbf{u}_i 	\mathbf{u}_i^T\mathbf{Y}_i^T \mathbf{H} \\
	%	& = \frac{1}{n_i} \tr \widetilde{\mathbf{y}}_{i}\widetilde{\mathbf{y}}_{i}^T \mathbf{H} \\
	& =  \frac{1}{n_i} \widetilde{\mathbf{y}}_{i,1}^T \mathbf{H} \widetilde{\mathbf{y}}_{i,1}
	\end{align*}
	Again, from the independence of $\widetilde{\mathbf{y}}_{i}$ and $\mathbf{H}$ and the application of the trace Lemma \cite[Theorem 3.7]{couillet} it follows that: %and by simple application of the trace lemma \citep[Theorem 3.7]{couillet}
	\begin{align*}
	\frac{1}{n_i^2} \mathbf{1}_{n_i}\mathbf{Y}_i^T \mathbf{H}\mathbf{Y}_i\mathbf{1}_{n_i} - \frac{1}{n_i}\tr \mathbf{\Sigma}_i \mathbf{H} \asto 0.
	\end{align*}
	%Finally, by \eqref{Q_bar}
	which gives using \eqref{Q_bar},
	% where, it is easy to show that $\mathbf{z}_i=\sqrt{n_i} \mathbf{Z}_i\mathbf{u}_i \sim \mathcal{N}\left(\mathbf{0}_{p\times 1},\mathbf{I}_p\right)$. By virtue of the trace lemma \citep[Theorem 3.7.]{couillet}, we have
	% \begin{align*}
	% \frac{1}{n_i} \mathbf{z}_i^T \mathbf{\Sigma}_i^{1/2} \mathbf{H} \mathbf{\Sigma}_i^{1/2} \mathbf{z}_i - \frac{1}{n_i}\tr \mathbf{\Sigma}_i\mathbf{H} \xrightarrow{a.s.}0.
	% \end{align*}
	% Again using (\ref{Q_bar}), we establish
	\begin{align*}
	\frac{1}{n_i^2}  \mathbf{1}_{n_i}\mathbf{Y}_i^T \mathbf{H}\mathbf{Y}_i\mathbf{1}_{n_i} + \frac{z}{n_i}\tr \mathbf{\Sigma}_i\bar{\mathbf{Q}}\left(z\right) \probto 0.
	\end{align*}
	This completes the proof of (\ref{G_i}).
	\subsection{Proof of (\ref{D_i})}
	Using the notations employed in the proof of \eqref{G_i}, $D(\overline{\bf x}_0,\overline{\bf x}_1,{\bf H},\boldsymbol{\Sigma}_i)$ can be expressed as:
	\begin{equation}
	\begin{split}
	D(\hat{\bm{\mu}}_0,\hat{\bm{\mu}}_1,{\bf H},\boldsymbol{\Sigma}_i)& =\left(\boldsymbol{\mu}^{T}+{\bf 1}_{n_0}^{T}\frac{{\bf Y}_0}{n_0}-{\bf 1}_{n_1}^{T}\frac{{\bf Y}_1}{n_1}\right){\bf H}\boldsymbol{\Sigma}_i{\bf H}  \\ 
	& \times \left(\boldsymbol{\mu}+{{\bf Y}_0}\frac{{\bf 1}_{n_0}}{n_0}-{{\bf Y}_1}\frac{{\bf 1}_{n_1}}{n_1}\right).
		\end{split}
	\label{eq:D}
	\end{equation}
	As in the proof \eqref{G_i}, from the independence of $\frac{1}{n_1}{\bf Y}_1{\bf 1}$ and $\frac{1}{n_0}{\bf Y}_0{\bf 1}$
	of ${\bf H}$, it is easy to see that the cross-products in \eqref{eq:D} will converge to zero almost surely. We thus have:  		%Using similar arguments used in the proof of (\ref{G_i}), we can easily show that
	\begin{align*}
	D\left(\hat{\bm{\mu}}_0,\hat{\bm{\mu}}_1,\mathbf{H},\mathbf{\Sigma}_i\right) & =\bm{\mu}^T \mathbf{H\Sigma}_i\mathbf{H}\bm{\mu} + \frac{1}{n_0^2}\mathbf{1}_{n_0}^T \mathbf{Y}_0^T\mathbf{H\Sigma}_i\mathbf{H}\mathbf{Y}_0\mathbf{1}_{n_0} \\ & +\frac{1}{n_1^2} \mathbf{1}_{n_1}^T \mathbf{Y}_1^T\mathbf{H\Sigma}_i\mathbf{HY}_1\mathbf{1}_{n_1}.
	\end{align*}
	Finally, we use \eqref{Q_tilde} to obtain%	Using (\ref{Q_tilde}), we have
	\begin{align*}
	&\bm{\mu}^T \mathbf{H\Sigma}_i\mathbf{H}\bm{\mu} - z^2 \bm{\mu}^T \widetilde{\mathbf{Q}}_i\left(z\right) \bm{\mu} \probto 0 \\
	&\frac{1}{n_j^2}\mathbf{1}_{n_j}^T\mathbf{Y}_j^T\mathbf{H\Sigma}_i\mathbf{H}\mathbf{Y}_j\mathbf{1}_{n_j}-\frac{z^2}{n_j}\tr \mathbf{\Sigma}_j\widetilde{\mathbf{Q}}_i\left(z\right) \probto 0, \: j=1-i.
	\end{align*}
	which completes the proof of (\ref{D_i}).

	\section{Proof of Proposition \ref{CLT_quadratic}}
	\label{appendix:clt}
	To reduce the amount of notations, we drop the class subscript $i$. In all the proof $\mathbf{B}$ plays the role of $\mathbf{B}_i$ and $\mathbf{y}$ plays the role of $\mathbf{y}_i$, for $i \in \{0,1\}$.
	To begin with, let $\mathbf{B} = \mathbf{U}_b \mathbb{B} \mathbf{U}_b^{T}$ be the eigenvalue decomposition of ${\bf B}$, so that ${\bm z}^{T}\mathbf{\bf B}{\bm z}+2{\bm z}^{T}{\bf r}$ has the same distribution as: %where $\mathbb{B} = \textbf{diag}\left(\alpha_j\right)$ and $\mathbf{U}_b$ is a unitary matrix. Then based on the fact that Gaussian vectors are invariant in distribution by multiplication with unitary matrices,  $g\left(\bm{\omega}\right)= \bm{\omega}^T \mathbf{B} \bm{\omega} + 2 \bm{\omega}^T \mathbf{y}$ can be expressed as
	\begin{align*}
	g\left({\bf z}\right) \triangleq& %\stackrel{d}{=}   \bm{\omega}^T \mathbb{B} \bm{\omega} + 2 \bm{\omega}^T \mathbf{\tilde{y}}  =
	\sum_{j=1}^{p} {\left(\alpha_j z_j^2 + 2z_j \tilde{r}_j\right)},
	\end{align*}
	where $\mathbf{\tilde{y}} = \mathbf{U}_b \mathbf{r}$, $\alpha_i$ diagonal elements of $\mathbb{B}$ and $z_j$ and $\tilde{r}_j$ are respectively the $j$th entries of $\bm{\omega}$ and $\mathbf{\tilde{y}}$. Let $\mathbf{\Psi} = \left[\mathbf{X}_0,\mathbf{X}_1\right]$ be the observations associated with class $0$ and $1$. Then, conditioning on $\mathbf{\Psi}$, $g\left(\bm{ z}\right)$ is the sum of independent but not identically distributed r.v's. $q_j=\alpha_j z_j^2
	+ 2z_j \tilde{r}_j$. To prove the CLT, we
	resort to the Lyapunov CLT Theorem, \cite[Theorem 27.3]{Billingsley}.
	We first calculate the mean and the variance of $r_j$ conditioned on $\mathbf{\Psi}$
	\begin{align*}
	\E\left[q_j | \mathbf{\Psi}\right] & = \alpha_j \\
	\text{var}\left[q_j | \mathbf{\Psi}\right] & = \sigma_j^2= 2 \alpha_j^2+4\tilde{r}_j^2.
	\end{align*}
	Define the total variance $s_p^2$ as
	\begin{equation}
	s_p^2 = \sum_{j=1}^{p} \sigma_j^2 = 2 \tr \mathbf{B}^2 + 4 \mathbf{\tilde{r}}^T\mathbf{\tilde{r}}.
	\end{equation}
	To prove the CLT, it suffices to check the Lyapunov's condition. Under the setting of Proposition \ref{CLT_quadratic},
	% \begin{align}
	%	\lim_{p \to \infty} \frac{1}{s_p^4} \sum_{j=1}^{p}\E \left[\left |r_j-\alpha_j\right|^4\right] = 0.
	%	\end{align}
	%	In fact,
	\begin{align*}
	& \lim_{p \to \infty} \frac{1}{s_p^4} \sum_{j=1}^{p}\E \left[\left |q_j-\alpha_j\right|^4| |\boldsymbol{\Psi}\right] \\ & = \lim_{p \to \infty} \frac{\sum_{j=1}^{p} 60 \alpha_j^4+240 \alpha_j^2 \tilde{r}_j^2+48 \tilde{r}_j^4}{\left(2 \tr \mathbf{B}^2 + 4 \mathbf{\tilde{r}}^T\mathbf{\tilde{r}}\right)^2} \\                                                                                                  & \leq   \lim_{p
		\to \infty} \frac{60/p^2\tr {\bf B}^2+240/p^2 \tr {\bf B}^2\|\tilde{\bf r}\|_2^2+48/p^2\|\tilde{\bf r}\|_2^4}{\left(2/p\tr {\bf B}^2+4/p\|\tilde{\bf r}\|_2^2\right)^2}.
	\end{align*}	
	%where $\overline{B}$ plays the same role of $\overline{B}_i$ upon dropping the class index $i$.					
	
	%			since $\alpha_j$ and $\tilde{y}_j$ both are of $\bigo\left(1\right)$. As a matter of fact, conditioning on the event $E$, we get
	%			\begin{align*}
	%			\frac{1}{s_p}\sum_{j=1}^p \left(r_j-\alpha_j\right) \overset{d}{\rightarrow} \mathcal{N}\left(0,1\right).
	%			\end{align*}
	%			This translates to
	%			\begin{align*}
	%			\E \left[\exp\left(it \frac{g\left(\bm{v}\right)-\tr \mathbf{B}}{s_p}\right) \lvert E\right] \rightarrow \exp\left(-\frac{t^2}{2}\right).
	%			\end{align*}
	%			Finally, by virtue of Theorem \ref{DE_th_QDA}, $\mathbb{P}\left(\bar{E}\right) \to 0$ and $\left | \E \left[\exp\left(it \frac{g\left(\bm{v}\right)-\tr \mathbf{B}}{s_p}\right) \lvert \bar{E}\right] \right | \leq 1$. Thus,
	%			\begin{align*}
	%			\E \left[\exp\left(it \frac{g\left(\bm{v}\right)-\tr \mathbf{B}}{s_p}\right)\right] & = \E \left[\exp\left(it \frac{g\left(\bm{v}\right)-\tr \mathbf{B}}{s_p}\right) \lvert E\right] \mathbb{P}\left(E\right) + \E \left[\exp\left(it \frac{g\left(\bm{v}\right)-\tr \mathbf{B}}{s_p}\right) \lvert \bar{E}\right] \mathbb{P}\left(\bar{E}\right) \\
	%			& \to \exp\left(-\frac{t^2}{2}\right).
	%			\end{align*}
	%			Applying the Levy's continuity theorem,
	%			\begin{align*}
	%			\frac{g\left(\bm{v}\right)-\tr \mathbf{B}}{s_p} \overset{d}{\rightarrow} \mathcal{N}\left(0,1\right).
	%		\end{align*}
	%		which gives the claim of Proposition \ref{CLT_quadratic}.
	\section{Proof of Theorem \ref{DE_QDA}}
	\label{appendix:DE_QDA}

	The proof consists in showing the following convergences
	\begin{align}
	&	\label{xi}
	\frac{1}{\sqrt{p}}\xi_i - \overline{\xi}_i  \probto 0.
	\\
	&\label{Bi}
	\frac{1}{\sqrt{p}}\tr \mathbf{B}_i - \overline{b_i} \asto 0.
	\\
	\label{B2}
	&\frac{1}{p}\tr \mathbf{B}_i^2 - 	\overline{B_i} \asto 0.
	\\
	\label{yi}
	&\frac{1}{p} \mathbf{r}_i^T\mathbf{r}_i \asto 0.
	\end{align}
	and establishing that the condition in \ref{CLT_quadratic} holds with probability $1$. 
	We will prove sequentially equations (\ref{xi})-(\ref{yi}).
	\subsection{Proof of (\ref{xi})}
	Using the simplified expression of $\widehat{\mathbf{\Sigma}}_i$ in \eqref{Sigma_simplified}, we can write
	\begin{align*}
	\mathbf{H}_i = \left(\mathbf{I}_p + \frac{\gamma}{n_i-1} \mathbf{Y}_i\mathbf{Y}_i^T -\frac{\gamma}{n_i-1} \mathbf{y}_i\mathbf{y}_i^T\right)^{-1}.
	\end{align*}
	Recall that $\frac{1}{\sqrt{p}}\xi_i$ writes as
	\begin{align*}
	\frac{1}{\sqrt{p}}\xi_i & = -\frac{1}{\sqrt{p}}\log\frac{|{\bf H}_0|}{|{\bf H}_1|} +\frac{1}{\sqrt{p}}(\boldsymbol{\mu}_i-\hat{\bm{\mu}}_0)^{T}{\bf H}_0(\boldsymbol{\mu}_i-\hat{\bm{\mu}}_0) \\ & -\frac{1}{\sqrt{p}}(\boldsymbol{\mu}_i-\hat{\bm{\mu}}_1)^{T}{\bf H}_1(\boldsymbol{\mu}_i-\hat{\bm{\mu}}_1) +\frac{2}{\sqrt{p}}\log\frac{\pi_1}{\pi_0}.
	\end{align*}
	Under Assumption \ref{As:3},  Matrix $\mathbf{H}_i$ follows the model in \cite{walid_new}. According to  \cite[Theorem 1]{walid_new}, % the results of which are then applicable.
	
	%We exploit the results of \citep[Theorem 1]{walid_new} with assumption \ref{As:3} to write
	
	\begin{align}
	\frac{1}{p} \log |\mathbf{H}_i|- \frac{1}{p}\left(\log |\mathbf{T}_i|-n_i \log \left(1+\gamma \delta_i\right)+   \gamma \frac{n_i \delta_i}{1+\gamma \delta_i}\right) \asto 0.
	\end{align}
	The convergence holds with  rate $O(p^{-1})$ hence,
	$$
	\frac{1}{\sqrt{p}} \log |\mathbf{H}_i|- \frac{1}{\sqrt{p}}\left(\log |\mathbf{T}_i|-n_i \log \left(1+\gamma \delta_i\right)+   \gamma \frac{n_i \delta_i}{1+\gamma \delta_i}\right) \probto 0.
	$$
	%		Moreover, since $\|\mathbf{H}_i\| = \bigo\left(1\right)$ and based on assumptions \ref{As:2} and \ref{As:3}, it is easy to show using \citep{walid_new} that
	and
	\begin{align*}
	& \frac{1}{\sqrt{p}} \left( \left(\bm{\mu}_i-\hat{\bm{\mu}}_0\right)^T \mathbf{H}_0 \left(\bm{\mu}_i-\hat{\bm{\mu}}_0\right)
	- \left(\bm{\mu}_i-\hat{\bm{\mu}}_1\right)^T \mathbf{H}_1 \left(\bm{\mu}_i-\hat{\bm{\mu}}_1\right)  \right) \\ & 
	-\frac{\left(-1\right)^{i+1}}{\sqrt{p}}\bm{\mu}^T \mathbf{T}_{1-i}\bm{\mu} \probto 0.
	\end{align*}

	\subsection{Proof of (\ref{Bi})}
%	We start by writing $\frac{1}{\sqrt{p}} \tr \mathbf{B}_i$ as
%	\begin{align*}
%	\frac{1}{\sqrt{p}} \tr \mathbf{B}_i &= \frac{1}{\sqrt{p}} \tr \mathbf{\Sigma}_i^{1/2} \left(\mathbf{H}_1-\mathbf{H}_0\right) \mathbf{\Sigma}_i^{1/2}  = \frac{1}{\sqrt{p}} \tr \mathbf{\Sigma}_i \left(\mathbf{H}_1-\mathbf{H}_0\right) \\
%	& = \frac{1}{\sqrt{p}} \tr \mathbf{\Sigma}_i\mathbf{H}_1 - \frac{1}{\sqrt{p}} \tr \mathbf{\Sigma}_i\mathbf{H}_0
%	\end{align*}
	From \cite{walid_new}, we know that
	\begin{equation}
	\frac{1}{p} \tr \mathbf{\Sigma}_i\mathbf{H}_i - \frac{1}{p} \tr \mathbf{\Sigma}_i\mathbf{T}_i \asto 0			\end{equation}
	where the above convergence holds with rate $O(p^{-1})$. 
	
	Thus,
	\begin{align*}
	\frac{1}{\sqrt{p}} \tr \mathbf{B}_i - \frac{1}{\sqrt{p}} \tr \mathbf{\Sigma}_i \left(\mathbf{T}_1-\mathbf{T}_0\right) \probto 0.
	\end{align*}
%	which yields the convergence in probability of $\frac{1}{\sqrt{p}}\tr \mathbf{B}_i$ to $\overline{b}_i$. 
	%	so, we have convergence in probability of $\frac{1}{\sqrt{p}} \tr \mathbf{B}_i $ to $\overline{b}_i$. 	
	
	\subsection{Proof of (\ref{B2})}
	To prove (\ref{B2}), we need the following lemma, the proof of which is omitted since it follows from the techniques established in \cite{walid_new}
	
	\begin{lemma}
		\label{lemma_trace_hachem}
		Let ${\bf A}$ be a matrix with uniformly bounded spectral norm. Then, for $i \in \{0,1\}$
		%For all $\mathbf{A}$ and $\mathbf{B}$ of finite spectral norms, we have 
		the following convergence 	holds true
		\begin{equation}
		\begin{split}
		& \frac{1}{p}\tr \mathbf{A}\mathbf{H}_i\mathbf{A}\mathbf{H}_i - \left[\frac{1}{p}\tr \mathbf{T}_i^2\mathbf{A}^2 + \frac{\gamma^2 \tilde{\phi}_i}{1-\gamma^2\phi_i\tilde{\phi}_i}\left( \frac{1}{n_i}\tr \mathbf{A}\mathbf{\Sigma}_i \mathbf{T}_i^2\right)^2 \right]
		\\ & \asto 0.
		\end{split}
		\end{equation}
	\end{lemma}
	%	\begin{proof}
	%		The proof is omitted because it is based on basic results from \citep{walid_new}.
	%	\end{proof}
	%	In a second step, we write
	With the above Lemma at hand, we are now ready to handle $\frac{1}{p}\tr \mathbf{B}_i^2$.
	\begin{align*}
	\frac{1}{p}\tr \mathbf{B}_i^2 & = \frac{1}{p}\tr \mathbf{\Sigma}_i \left(\mathbf{H}_1-\mathbf{H}_0\right)\mathbf{\Sigma}_i \left(\mathbf{H}_1-\mathbf{H}_0\right) \\
	& = \frac{1}{p}\tr \mathbf{\Sigma}_i\mathbf{H}_1\mathbf{\Sigma}_i\mathbf{H}_1 + \frac{1}{p}\tr \mathbf{\Sigma}_i\mathbf{H}_0\mathbf{\Sigma}_i\mathbf{H}_0 -\frac{2}{p}\tr \mathbf{\Sigma}_i\mathbf{H}_0\mathbf{\Sigma}_i\mathbf{H}_1.
	%	& = \frac{1}{p}\tr\mathbf{\Sigma}_1\mathbf{H}_1\mathbf{\Sigma}_1\mathbf{H}_1 + \frac{1}{p}\tr \mathbf{\Sigma}_0\mathbf{H}_0\mathbf{\Sigma}_0\mathbf{H}_0 -\frac{2}{p}\tr \mathbf{\Sigma}_i\mathbf{H}_0\mathbf{\Sigma}_i\mathbf{H}_1 \\ &+ \frac{2}{p}\tr\left(\mathbf{\Sigma}_i-\mathbf{\Sigma}_1\right)\mathbf{H}_1\mathbf{\Sigma}_1\mathbf{H}_1 + \frac{1}{p}\tr\left(\mathbf{\Sigma}_i-\mathbf{\Sigma}_1\right)\mathbf{H}_1\left(\mathbf{\Sigma}_i-\mathbf{\Sigma}_1\right)\mathbf{H}_1 \\
	%	&+ \frac{2}{p}\tr\left(\mathbf{\Sigma}_i-\mathbf{\Sigma}_0\right)\mathbf{H}_0\mathbf{\Sigma}_0\mathbf{H}_0 + \frac{1}{p}\tr\left(\mathbf{\Sigma}_i-\mathbf{\Sigma}_0\right)\mathbf{H}_0\left(\mathbf{\Sigma}_i-\mathbf{\Sigma}_0\right)\mathbf{H}_0
	\end{align*}
	By product of Lemma \ref{lemma_trace_hachem}, we can easily get
	\begin{align}
	& \frac{1}{p}\tr \mathbf{\Sigma}_i\mathbf{H}_i\mathbf{\Sigma}_i\mathbf{H}_i - \frac{c\phi_i}{1-\gamma^2 \phi_i\tilde{\phi}_i} \asto 0  	\label{hachem_results_Hi2}
	\end{align}
	\begin{equation}
	\begin{split}
	&  \frac{1}{p}\tr \mathbf{\Sigma}_i\mathbf{H}_{1-i}\mathbf{\Sigma}_i\mathbf{H}_{1-i}  \\ & - \left[\frac{1}{p} \tr \mathbf{\Sigma}_i^2\mathbf{T}_{1-i}^2 + \frac{\gamma^2 c \tilde{\phi}_{1-i}\left(\frac{1}{n_i}\tr \mathbf{\Sigma}_{i}\mathbf{\Sigma}_{1-i}\mathbf{T}_{1-i}^2\right)^2}{1-\gamma^2 \phi_{1-i}\tilde{\phi}_{1-i}}\right] \\ &  \asto 0.
	\end{split}
	\end{equation}
	Finally, it is straightforward to obtain
	\begin{align}
	\label{hachem_results_Hi1}
	&	\frac{1}{p}\tr \mathbf{\Sigma}_i\mathbf{H}_1\mathbf{\Sigma}_i\mathbf{H}_0 - \frac{1}{p}\tr \mathbf{\Sigma}_i\mathbf{T}_1\mathbf{\Sigma}_i\mathbf{T}_0 \asto 0.
	\end{align}
	This completes the proof of (\ref{B2}).
	\subsection{Proof of (\ref{yi})}
	% \textcolor{red}{Abla: I think not correct. what is $j$ and $i$. Not clear, Professor Tareq pointed out that this is not correct, and i think he is right. }	By simple manipulations, we can show that
	Let $i\in\left\{0,1\right\}$. Then, one can see that:
	\begin{align}
	\frac{1}{p}\mathbf{r}_i^T\mathbf{r}_i - \frac{1}{p} \bm{\mu}^T \mathbf{H}_{1-i}\mathbf{\Sigma}_i\mathbf{H}_{1-i}\bm{\mu} \probto 0,
	\end{align}
	where by Assumptions \ref{As:2} and \ref{As:3}
	\begin{align*}
	&	\frac{1}{p} \bm{\mu}^T \mathbf{H}_{1-i}\mathbf{\Sigma}_i\mathbf{H}_{1-i}\bm{\mu} = O\left(\frac{1}{\sqrt{p}}\right).
	%& = \frac{1}{p} \tr \bm{\mu}\bm{\mu}^T \mathbf{H}_j\mathbf{\Sigma}_i\mathbf{H}_j \\
	%	& =\frac{1}{\sqrt{p}} \tr \frac{\bm{\mu}\bm{\mu}^T}{\sqrt{p}}\mathbf{H}_j\mathbf{\Sigma}_j\mathbf{H}_j + \frac{1}{\sqrt{p}} \tr \frac{\bm{\mu}\bm{\mu}^T}{\sqrt{p}}\mathbf{H}_j\left(\mathbf{\Sigma}_i-\mathbf{\Sigma}_j\right)\mathbf{H}_j
	\end{align*}
	Finally, by applying the continous mapping theorem \cite{serfling}, we complete the proof of Theorem \ref{DE_QDA}.
	
	Now, to conclude we need to check that the condition in Proposition \ref{CLT_quadratic} holds with probability 1. This can be easily seen by replacing in \eqref{cond}, $\frac{1}{p}\tr {\bf B}^2$ by its deterministic equivalent and noting that it has order $O(1)$. 
	
	\section{Proof of Proposition \ref{prop:order}}
	\label{app:order}
	%  \textcolor{red}{to be moved}
	%	\begin{align*}
	%		\frac{1}{\sqrt{p}} \left(\log |\mathbf{H}_1| - \log |\mathbf{H}_0| \right) = 	\frac{1}{\sqrt{p}} \left(\log |\mathbf{H}_1| - p\overline{L}_1 \right)
	%	\end{align*}
	%	Unlike the convergence in \citep{walid_new}, the convergence here is in the order of $\bigo\left(\frac{1}{\sqrt{p}}\right)$. Thus, to complete the proof of (\ref{xi}), we need to show the following
	To prove Proposition \ref{prop:order}, it suffices to show
	\begin{align}\label{req1}
	&\frac{1}{\sqrt{p}}\left(\frac{n_0 \delta_0}{1+\gamma \delta_0}-\frac{n_1 \delta_1}{1+\gamma \delta_1}\right) = O\left(1\right) \\
	\label{req2}
	&\frac{1}{\sqrt{p}} \left(n_1\log \left(1+\gamma \delta_1\right)-n_0\log \left(1+\gamma \delta_0\right)\right) = O\left(1\right) \\
	\label{req3}
	&\frac{1}{\sqrt{p}} \left(\log |\mathbf{T}_0| - \log |\mathbf{T}_1|\right) = O\left(1\right)
	\end{align}
Relying on Assumption \ref{As:1}, we can assume $n_0=n_1=\frac{n}{2}$.  To begin, we first note that
	\begin{align*}
	\frac{ \delta_0}{1+\gamma \delta_0}-\frac{ \delta_1}{1+\gamma \delta_1} = \frac{1}{\left(1+\gamma \delta_0\right)\left(1+\gamma \delta_1\right)} \left(\delta_0-\delta_1\right).
	\end{align*}
	On the other hand,
	\begin{align*}
	\delta_0-\delta_1 & = \frac{2}{n} \tr \mathbf{\Sigma}_0\mathbf{T}_0 - \frac{2}{n} \tr \mathbf{\Sigma}_1\mathbf{T}_1 	\\
	& = \frac{2}{n} \tr \mathbf{\Sigma}_0 \left(\mathbf{T}_0-\mathbf{T}_1\right) + \frac{2}{n} \tr \left(\mathbf{\Sigma}_0-\mathbf{\Sigma}_1\right)\mathbf{T}_1.
	\end{align*}
	Recall that for invertible square matrices $\mathbf{A}$ and $\mathbf{B}$, we have the \emph{resolvent identity} given by
	\begin{align}
	\label{resolvent_id}
	\mathbf{A}^{-1}-\mathbf{B}^{-1} = \mathbf{A}^{-1} \left(\mathbf{B}-\mathbf{A}\right)\mathbf{B}^{-1}.
	\end{align}
	Thus,
	\begin{align*}
	& \delta_0-\delta_1 \\  &  = \frac{2}{n} \tr \mathbf{\Sigma}_0 \mathbf{T}_0 \left(\mathbf{T}_1^{-1}-\mathbf{T}_0^{-1}\right)\mathbf{T}_1+ \frac{2}{n} \tr \left(\mathbf{\Sigma}_0-\mathbf{\Sigma}_1\right)\mathbf{T}_1 \\
	& = \frac{2}{n} \tr \mathbf{\Sigma}_0 \mathbf{T}_0 \left(\frac{\gamma}{1+\gamma \delta_1}\mathbf{\Sigma}_1 -\frac{\gamma}{1+\gamma \delta_0}\mathbf{\Sigma}_0 \right) \mathbf{T}_1 + \frac{2}{n} \tr \left(\mathbf{\Sigma}_0-\mathbf{\Sigma}_1\right)\mathbf{T}_1 \\
	& = \frac{2\gamma}{n} \tr \mathbf{\Sigma}_0 \mathbf{T}_0 \left(\frac{\mathbf{\Sigma}_1}{1+\gamma \delta_1}-\frac{\mathbf{\Sigma}_1}{1+\gamma \delta_0}+\frac{\mathbf{\Sigma}_1}{1+\gamma \delta_0}-\frac{\mathbf{\Sigma}_0}{1+\gamma \delta_0}\right) \mathbf{T}_1   \\ & + \frac{2}{n} \tr \left(\mathbf{\Sigma}_0-\mathbf{\Sigma}_1\right)\mathbf{T}_1 \\
	& = \frac{\gamma^2 \left(\delta_0-\delta_1\right)}{\left(1+\gamma \delta_0\right)\left(1+\gamma \delta_1\right)} \frac{2}{n} \tr \mathbf{\Sigma}_0 \mathbf{T}_0 \mathbf{\Sigma}_1 \mathbf{T}_1  \\ & - \frac{\gamma}{1+\gamma \delta_0} \frac{2}{n} \tr \mathbf{\Sigma}_0 \mathbf{T}_0 \left(\mathbf{\Sigma}_0-\mathbf{\Sigma}_1\right)\mathbf{T}_1
	 + \frac{2}{n} \tr \left(\mathbf{\Sigma}_0-\mathbf{\Sigma}_1\right)\mathbf{T}_1.
	\end{align*}
	Therefore,
	\begin{align*}
	& \left(\delta_0-\delta_1\right) \left[1-\frac{2\gamma^2 }{\left(1+\gamma \delta_0\right)\left(1+\gamma \delta_1\right)} \frac{1}{n} \tr \mathbf{\Sigma}_0 \mathbf{T}_0 \mathbf{\Sigma}_1 \mathbf{T}_1\right] \\ & = - \frac{\gamma}{1+\gamma \delta_0} \frac{2}{n} \tr \mathbf{\Sigma}_0 \mathbf{T}_0 \left(\mathbf{\Sigma}_0-\mathbf{\Sigma}_1\right)\mathbf{T}_1
	+ \frac{2}{n} \tr \left(\mathbf{\Sigma}_0-\mathbf{\Sigma}_1\right)\mathbf{T}_1.
	\end{align*}
	or equivalently
	\begin{align*}
	\delta_0-\delta_1   &= \left[1-\frac{2\gamma^2 }{\left(1+\gamma \delta_0\right)\left(1+\gamma \delta_1\right)} \frac{1}{n} \tr \mathbf{\Sigma}_0 \mathbf{T}_0 \mathbf{\Sigma}_1 \mathbf{T}_1\right] ^{-1} \\&\times
	\left[- \frac{2\gamma}{1+\gamma \delta_0} \frac{1}{n} \tr \mathbf{\Sigma}_0 \mathbf{T}_0 \left(\mathbf{\Sigma}_0-\mathbf{\Sigma}_1\right)\mathbf{T}_1
	+ \frac{2}{n} \tr \left(\mathbf{\Sigma}_0-\mathbf{\Sigma}_1\right)\mathbf{T}_1\right].
	\end{align*}
	All in all, we have
	\begin{equation}
	\label{ineq}
	\begin{split}
	& \frac{ \delta_0}{1+\gamma \delta_0}-\frac{ \delta_1}{1+\gamma \delta_1} \\ & =  \frac{1}{\left(1+\gamma \delta_0\right)\left(1+\gamma \delta_1\right)} \left[1-\frac{2\gamma^2 }{\left(1+\gamma \delta_0\right)\left(1+\gamma \delta_1\right)} \frac{1}{n} \tr \mathbf{\Sigma}_0 \mathbf{T}_0 \mathbf{\Sigma}_1 \mathbf{T}_1\right] ^{-1}\\&\times
	\left[- \frac{\gamma}{1+\gamma \delta_0} \frac{2}{n} \tr \mathbf{\Sigma}_0 \mathbf{T}_0 \left(\mathbf{\Sigma}_0-\mathbf{\Sigma}_1\right)\mathbf{T}_1
	+ \frac{1}{n} \tr \left(\mathbf{\Sigma}_0-\mathbf{\Sigma}_1\right)\mathbf{T}_1\right]
	\end{split}
	\end{equation}
	To guarantee that the left hand side of \eqref{ineq} does not blow up, we shall prove that
	\begin{align}
	\label{liminf} 
	\lim \inf_{p} \left( 1-\frac{2\gamma^2 }{\left(1+\gamma \delta_0\right)\left(1+\gamma \delta_1\right)} \frac{1}{n} \tr \mathbf{\Sigma}_0 \mathbf{T}_0 \mathbf{\Sigma}_1 \mathbf{T}_1 \right) > 0.
	\end{align}
	or equivalently
	\begin{align}
	\label{limsup}
	\lim \sup_{p} \frac{2\gamma^2 }{\left(1+\gamma \delta_0\right)\left(1+\gamma \delta_1\right)} \frac{1}{n} \tr \mathbf{\Sigma}_0 \mathbf{T}_0 \mathbf{\Sigma}_1 \mathbf{T}_1 < 1.
	\end{align}
	%	This is important to guarantee that the quantity $\left(1-\frac{\gamma^2 }{\left(1+\gamma \delta_0\right)\left(1+\gamma \delta_1\right)} \frac{1}{n} \tr \mathbf{\Sigma}_0 \mathbf{T}_0 \mathbf{\Sigma}_1 \mathbf{T}_1\right) ^{-1}$ does not blow up.
	For that, recall that for a symmetric matrix $\mathbf{A}$ and non-negative definite matrix $\mathbf{B}$ \cite{matrix_ineq}, we have
	\begin{align}
	\tr \mathbf{AB} \leq \|\mathbf{A}\|\tr \mathbf{B}.
	\end{align}
	Thus,
	\begin{align*}
	& \frac{\gamma^2 }{\left(1+\gamma \delta_0\right)\left(1+\gamma \delta_1\right)} \frac{1}{n} \tr \mathbf{\Sigma}_0 \mathbf{T}_0 \mathbf{\Sigma}_1 \mathbf{T}_1 \\  & \leq
	\frac{\gamma^2 \|\mathbf{\Sigma}_0\mathbf{T}_0\| }{\left(1+\gamma \delta_0\right)\left(1+\gamma \delta_1\right)} \underbrace{\frac{1}{n} \tr \mathbf{\Sigma}_1 \mathbf{T}_1}_{\delta_1/2}  = \underbrace{\frac{\gamma \delta_1/2}{1+\gamma \delta_1}}_{<1/2} \frac{\gamma  \|\mathbf{\Sigma}_0\mathbf{T}_0\|}{1+\gamma \delta_0} \\ &
	< \frac{\gamma/2  \|\mathbf{\Sigma}_0\mathbf{T}_0\|}{1+\gamma \delta_0}.
	\end{align*}
	Since $\mathbf{\Sigma}_0$ and $\mathbf{T}_0$ share the same eigenvectors, there exists a $\lambda$ an eigenvalue of $\mathbf{\Sigma}_0$ such that
	\begin{align*}
	\|\mathbf{\Sigma}_0\mathbf{T}_0\| = \frac{\lambda}{1+\frac{\gamma \lambda}{1+\gamma \delta_0}}.
	\end{align*}
	Thus,
	\begin{align*}
		& \frac{2\gamma^2 }{\left(1+\gamma \delta_0\right)\left(1+\gamma \delta_1\right)} \frac{1}{n} \tr \mathbf{\Sigma}_0 \mathbf{T}_0 \mathbf{\Sigma}_1 \mathbf{T}_1
	 < \frac{\gamma  \|\mathbf{\Sigma}_0\mathbf{T}_0\|}{1+\gamma \delta_0}  \\ & = \frac{\frac{\gamma \lambda}{1+\frac{\gamma \lambda}{1+\gamma \delta_0}}}{1+\gamma \delta_0}	 = \frac{\frac{\gamma \lambda}{1+\gamma \delta_0}}{1+\frac{\gamma \lambda}{1+\gamma \delta_0}} < 1.
	\end{align*}
	Thus, (\ref{limsup}) holds. Using Assumptions \ref{As:1} and \ref{As:3} and by (\ref{limsup})
	\begin{align*}
	&\frac{1}{\left(1+\gamma \delta_0\right)\left(1+\gamma \delta_1\right)} \\ & \times  \left(1-\frac{2\gamma^2 }{\left(1+\gamma \delta_0\right)\left(1+\gamma \delta_1\right)} \frac{1}{n} \tr \mathbf{\Sigma}_0 \mathbf{T}_0 \mathbf{\Sigma}_1 \mathbf{T}_1\right) ^{-1}   = O\left(1\right),
	\end{align*}
	which implies using Assumption \ref{As:4} that
	\begin{align}
	\label{delta_gap}
	\delta_0-\delta_1 = O\left(\frac{1}{\sqrt{p}}\right).
	\end{align}
	This gives the claim of (\ref{req1}).
	For (\ref{req2}), and using the inequality $\log \left(x\right) \leq x-1$, for $x>0$ we can show that
	\begin{align*}
	& \begin{vmatrix}
	\log \left(1+\gamma \delta_1\right)-\log \left(1+\gamma \delta_0\right)
	\end{vmatrix}  = \begin{vmatrix}\log \frac{1+\gamma \delta_1}{1+\gamma \delta_0} \end{vmatrix}  \\ &
	 =  \begin{vmatrix}\log \left( 1+\gamma\frac{ \delta_1 - \delta_0}{1+\gamma \delta_0} \right )\end{vmatrix} 
	\leq \frac{\gamma \begin{vmatrix}
		\delta_0 - \delta_1
		\end{vmatrix}}{1+ \gamma \min \left(\delta_0,\delta_1\right)}.
	\end{align*}
	Following the result of (\ref{req1}), (\ref{req2}) also holds. \\
	As for \eqref{req3}, it suffices to notice that:
	\begin{align*}
	& \frac{1}{\sqrt{p}}\log\det {\bf T}_0{\bf T}_1^{-1} \\ &=\frac{1}{\sqrt{p}}\log\det \left({\bf I}_p+{\bf T}_0^{\frac{1}{2}}\left(\gamma\tilde{\delta_1}\boldsymbol{\Sigma}_1-\gamma\tilde{\delta_0}\boldsymbol{\Sigma}_0\right){\bf T}_0^{\frac{1}{2}}\right)\\
	&=\frac{1}{\sqrt{p}}\log\det \left[{\bf I}_p+\gamma \tilde{\delta}_1{\bf T}_0(\boldsymbol{\Sigma}_1-\boldsymbol{\Sigma}_0)+ \frac{\tilde{\delta}_1-\tilde{\delta}_0}{\tilde{\delta}_0}\left(\mathbf{I}_p - \mathbf{T}_0\right)\right]\\
	%	\textcolor{red}{to be finished}																	%	   &\leq \frac{1}{\sqrt{p}}\sum_{i=1}^p \log\left(1+\frac{\gamma}{1+\gamma\delta_1}\left[{\bf T}_0(\boldsymbol{\Sigma}_1-\boldsymbol{\Sigma}_0)^2{\bf T}_0\right]_{ii}+\frac{\gamma^2 \left[{\bf T}_0\boldsymbol{\Sigma}\right]_{ii}(\delta_0-\delta_1)^2}{(1+\gamma\delta_1)(1+\gamma\delta_0)}\right)
	\end{align*}
	Define $\mathbf{\Phi}$ the matrix that has the same eigenvectors as $\mathbf{\Sigma}_1-\mathbf{\Sigma}_0$ with eigenvalues $\phi_i = |\lambda_i\left(\mathbf{\Sigma}_1-\mathbf{\Sigma}_0\right)|$, $i = 1, \cdots, p$. Then, since ${\bf T}_0\preceq {\bf I}_p$, we have the following  
	\begin{align*}
	& \left|	\frac{1}{\sqrt{p}}\log\det {\bf T}_0{\bf T}_1^{-1} \right| \\ &\leq\frac{1}{\sqrt{p}} \log \det \left[ \mathbf{I}_p + \gamma \tilde{\delta}_1 {\bf T}_0^{\frac{1}{2}}\mathbf{\Phi}{\bf T}_0^{\frac{1}{2}} + \frac{\left|\tilde{\delta}_1-\tilde{\delta}_0\right|}{\tilde{\delta}_0} \mathbf{I}_p\right]. 
	\end{align*}
	Given that $0 \preceq \mathbf{\Phi}$, we have 
	\begin{align*}
	& \frac{1}{\sqrt{p}} \log \det \left[ \mathbf{I}_p + \gamma \tilde{\delta}_1 {\bf T}_0^{\frac{1}{2}}\mathbf{\Phi}{\bf T}_0^{\frac{1}{2}} + \frac{\left|\tilde{\delta}_1-\tilde{\delta}_0\right|}{\tilde{\delta}_0} \mathbf{I}_p\right] \\ &  \leq \frac{1}{\sqrt{p}} \tr \left[\gamma \tilde{\delta}_1 \mathbf{\Phi} {\bf T}_0+ \frac{\left|\tilde{\delta}_1-\tilde{\delta}_0\right|}{\tilde{\delta}_0} \mathbf{I}_p\right]
	\end{align*}
	By Assumption \ref{As:4}, $\frac{1}{\sqrt{p}} \tr \left[\gamma \tilde{\delta}_1 \mathbf{\Phi}{\bf T}_0 + \frac{\left|\tilde{\delta}_1-\tilde{\delta}_0\right|}{\tilde{\delta}_0} \mathbf{I}_p\right] = O\left(1\right)$. This completes the proof.

	\section{Proof of Theorem \ref{LDA_estimator}}
	\label{appendix:GE_LDA}
	We start the proof by showing the following
	\begin{align*}
	G\left(\hat{\bm{\mu}}_i,\hat{\bm{\mu}}_0,\hat{{\bm{\mu}}}_1,\mathbf{H}\right) + \left(-1\right)^{i+1} \widehat{\theta}_i - G\left(\bm{\mu}_i,{\hat{\bm{\mu}}}_0,\hat{\bm{\mu}}_1,\mathbf{H}\right) \asto 0.
	\end{align*}
	To this end, note that
	%In fact, we have
	\begin{align*}
	G\left(\hat{\bm{\mu}}_i,\hat{\bm{\mu}}_0,\hat{\bm{\mu}}_1,\mathbf{H}\right) = G\left(\bm{\mu}_i,\hat{\bm{\mu}}_0,\hat{\bm{\mu}}_1,\mathbf{H}\right) + \frac{\mathbf{1}_{n_i}^T}{n_i} \mathbf{Y}_i \mathbf{H}\left(\hat{\bm{\mu}}_0-\hat{\bm{\mu}}_1\right).
	\end{align*}
	Using the same arguments used to prove \eqref{G_i}, we can easily show that
	\begin{align*}
	G\left(\hat{\bm{\mu}}_i,\hat{\bm{\mu}}_0,\hat{\bm{\mu}}_1,\mathbf{H}\right) = G\left(\bm{\mu}_i,\hat{\bm{\mu}}_0,\hat{\bm{\mu}}_1,\mathbf{H}\right) +  \frac{\left(-1\right)^i}{n_i} \tr \mathbf{\Sigma}_i\mathbf{H}+o(1)
	\end{align*}
	It remains now to show that
	\begin{align}
	\label{theta_hat}
	\widehat{\theta}_i - \frac{1}{n_i} \tr \mathbf{\Sigma}_i\mathbf{H} \asto 0.
	\end{align}
	To this end, we examine the convergence of the quantity $\frac{1}{n_i} \tr \widehat{\mathbf{\Sigma}}_i\mathbf{H}$. Recall from \eqref{Sigma_simplified} that
	\begin{align*}
	\widehat{\mathbf{\Sigma}}_i & = \frac{1}{n_i-1} \widetilde{\mathbf{Y}}_i\widetilde{\mathbf{Y}}_i^T - \frac{1}{n_i-1} \widetilde{\mathbf{y}}_{i,1}\widetilde{\mathbf{y}}_{i,1}^T \\
	& = \frac{1}{n_i-1} \sum_{j=1}^{n_i} \widetilde{\mathbf{y}}_{i,j}\widetilde{\mathbf{y}}_{i,j}^T - \frac{1}{n_i-1} \widetilde{\mathbf{y}}_{i,1}\widetilde{\mathbf{y}}_{i,1}^T
	\end{align*}
	%	By simple application of the Sherman-Morrison inversion Lemma \citep{numerical_recipes}, we have
	%		\begin{align*}
	%		\mathbf{H} = \mathbf{H}_j - \frac{\frac{\gamma}{n-2} \mathbf{H}_j \widetilde{\mathbf{y}}_{i,j}\widetilde{\mathbf{y}}_{i,j}^T \mathbf{H}_j}{1+\frac{\gamma}{n-2}\widetilde{\mathbf{y}}_{i,j}^T \mathbf{H}_j \widetilde{\mathbf{y}}_{i,j}}, \:\: \forall j=1,\cdots,n_i.
	%		\end{align*}
	Let ${\bf H}_{[j]}=\left(\gamma\widehat{\boldsymbol{\Sigma}}-\frac{\gamma}{n-2}\widetilde{\bf y}_{i,j}\widetilde{\bf y}_{i,j}^{T}+{\bf I}_p\right)^{-1}$
	Thus,
	\begin{align*}
	\frac{1}{n_i} \tr \widehat{\mathbf{\Sigma}}_i\mathbf{H} = \frac{1}{n_i}\sum_{j=2}^{n_i}\frac{\frac{1}{n_i-1}\widetilde{\bf y}_{i,j}^{T}{\bf H}_{[j]}\widetilde{\bf y}_{i,j}}{1+\frac{\gamma}{n-2}\widetilde{\bf y}_{i,j}^{T}{\bf H}_{[j]}\widetilde{\bf y}_{i,j}}%\frac{1}{n_i} \sum_{j=2}^{n_i} \frac{\widetilde{\mathbf{y}}_{i,j}^T \mathbf{H}_j \widetilde{\mathbf{y}}_{i,j}}{n_i-1} -  \frac{\frac{\gamma}{(n-2)(n_i-1)} \widetilde{\mathbf{y}}_{i,j}^T\mathbf{H}_j \widetilde{\mathbf{y}}_{i,j}\widetilde{\mathbf{y}}_{i,j}^T \mathbf{H}_j \widetilde{\mathbf{y}}_{i,j}}{1+\frac{\gamma}{n-2}\widetilde{\mathbf{y}}_{i,j}^T \mathbf{H}_j \widetilde{\mathbf{y}}_{i,j}}
	\end{align*}
	Thanks to the independence between $\mathbf{H}_j$ and $\widetilde{\mathbf{y}}_{i,j}$ and by simple application of the trace Lemma \cite{couillet}, we have
	\begin{align*}
	\frac{1}{n_i} \tr \widehat{\mathbf{\Sigma}}_i\mathbf{H} - \frac{1}{n_i} \tr \mathbf{\Sigma}_i\mathbf{H} \frac{1}{1+\frac{\gamma}{n-2}\tr \boldsymbol{\Sigma}_i{\bf H}}\to_{a.s.} 0.%{1+}\left(1-\frac{\frac{\gamma}{n-2} \tr \mathbf{\Sigma}_i\mathbf{H}}{1+\frac{\gamma}{n-2} \tr \mathbf{\Sigma}_i\mathbf{H}}\right) \xrightarrow{a.s.}0.
	\end{align*}
	By simple manipulations, we have the convergence in \eqref{theta_hat}. \par
	Now, using the same tricks consisting in using the inversion Lemma along with the trace Lemma, we obtain%on the use of the inverse Lemma and the trace Lemma, we can get
	\begin{align*}
	\widehat{\psi}_i^2	D\left({\hat{\bm{\mu}}}_0,\hat{\bm{\mu}}_1,\mathbf{H},\widehat{\mathbf{\Sigma}}_i\right) - D\left({\hat{\bm{\mu}}}_0,\hat{\bm{\mu}}_1,\mathbf{H},\mathbf{\Sigma}_i\right) \asto 0.
	\end{align*}
	\section{Proof of Theorem \ref{QDA_estimator}}
	\label{appendix:QDA_estimator}
	The proof consists in proving the following convergences
	\begin{align}
	\label{xi_estim}
	\widehat{\xi}_i - \frac{1}{\sqrt{p}}\xi_i \probto 0.
	\end{align}
	\begin{align}
	\label{bi_estim}
	\widehat{b}_i - \frac{1}{\sqrt{p}} \tr \mathbf{B}_i \probto 0.
	\end{align}
	\begin{align}
	\label{Bi_estim}
	\widehat{B}_i - \frac{1}{p} \tr \mathbf{B}_i^2 \probto 0.
	\end{align}
	The proof of \eqref{xi_estim} is straightforward and relies on the following facts
	\begin{align*}
	\frac{1}{\sqrt{p}} \left(\bm{\mu}_i -\hat{\bm{\mu}}_i\right)^T \mathbf{H}_i  \left(\bm{\mu}_i -\hat{\bm{\mu}}_i\right) \probto 0.
	\end{align*}
	\begin{align*}
	& \frac{1}{\sqrt{p}}	\left(\bar{\hat{\bm{\mu}}}_0-\hat{\bm{\mu}}_1\right)^T \mathbf{H}_{1-i} \left(\hat{\bm{\mu}}_0-\hat{\bm{\mu}}_1\right) \\ &  -  	\frac{1}{\sqrt{p}}	\left(\bm{\mu}_i-\hat{\bm{\mu}}_{1-i}\right)^T \mathbf{H}_{1-i} 	\left(\bm{\mu}_i-\hat{\bm{\mu}}_{1-i}\right) \probto 0.
	\end{align*}
	The proof of \eqref{bi_estim} relies on the fact that $\widehat{\delta}_i$ is a consistent estimator of $\delta_i$ as shown in \cite{zollanvari}, the variance of which can be shown to be of order $O(p^{-2})$.  Thus,
	\begin{align*}
	\frac{1}{\sqrt{p}} \tr \mathbf{\Sigma}_i\mathbf{T}_i -  \frac{n_i}{\sqrt{p}} \widehat{\delta}_i \probto 0.
	\end{align*}
	Also, we have, for $i\in\left\{0,1\right\}$,
	\begin{align*}
	\frac{1}{\sqrt{p}} \tr \widehat{\mathbf{\Sigma}}_i\mathbf{H}_{1-i} - \frac{1}{\sqrt{p}} \tr \mathbf{\Sigma}_i\mathbf{H}_{1-i} \probto 0.
	\end{align*}
	which gives the convergence in \eqref{bi_estim}. \\
	\subsection{Proof of \eqref{Bi_estim}}
	The proof of \eqref{Bi_estim} is a bit more involved than those of \eqref{xi_estim} and \eqref{bi_estim} as we will show in the following. The proof mainly relies on the application of the inversion lemma followed by the trace lemma \cite{couillet}. Recall that 
	\begin{align*}
	\frac{1}{p} \tr \mathbf{B}_i^2  &=\frac{1}{p}\tr \boldsymbol{\Sigma}_i {\bf H}_1\boldsymbol{\Sigma}_i{\bf H}_1-\frac{2}{p}\tr \boldsymbol{\Sigma}_i{\bf H}_0\boldsymbol{\Sigma}_i{\bf H}_1 \\ &  +\frac{1}{p}\tr \boldsymbol{\Sigma}_i{\bf H}_0\boldsymbol{\Sigma}_i{\bf H}_0.
	\end{align*}
	Without loss of generality, we can assume $i=1$, the other case follows naturally. We start by handling the term $\frac{1}{p}\tr \boldsymbol{\Sigma}_1 {\bf H}_1\boldsymbol{\Sigma}_1{\bf H}_1$. The common method here is to replace $\mathbf{\Sigma}_1$ by its sample estimate $\widehat{\mathbf{\Sigma}}_1$, then compute the limit of the obtained expression and perform the necessary corrections to obtain the estimate of interest. In fact, we have 
	\begin{align*}
	&\frac{1}{p}\tr \widehat{\boldsymbol{\Sigma}}_1{\bf H}_1\widehat{\boldsymbol{\Sigma}}_1{\bf H}_1=\frac{1}{p}\sum_{j=1}^{n_1-1}\sum_{k=1}^{n_1-1}\frac{1}{p}\tr \frac{\widetilde{\bf y}_{1,j}\widetilde{\bf y}_{1,j}^{T}}{n_1-1}{\bf H}_1\frac{\widetilde{\bf y}_{1,k}\widetilde{\bf y}_{1,k}^{T}}{n_1-1}{\bf H}_1\\
	&=\frac{1}{p}\sum_{j=1}^{n_1-1}\sum_{k\neq j}\tr \frac{\widetilde{\bf y}_{1,j}\widetilde{\bf y}_{1,j}^{T}}{n_1-1}{\bf H}_1\frac{\widetilde{\bf y}_{1,k}\widetilde{\bf y}_{1,k}^{T}}{n_1-1}{\bf H}_1 \\ & +\frac{1}{p}\sum_{j=1}^{n_1-1}\tr\frac{\widetilde{\bf y}_{1,j}\widetilde{\bf y}_{1,j}^{T}}{n_1-1}{\bf H}_1\frac{\widetilde{\bf y}_{1,j}\widetilde{\bf y}_{1,j}^{T}}{n_1-1}{\bf H}_1
	\end{align*}
	Using the inversion lemma, we handle the first term in the previous equation as follows
	\begin{align*}
	&\frac{1}{p}\sum_{j=2}^{n_1-1}\sum_{k\neq j}\tr \frac{\widetilde{\bf y}_{1,j}\widetilde{\bf y}_{1,j}^{T}}{n_1-1}{\bf H}_1\frac{\widetilde{\bf y}_{1,k}\widetilde{\bf y}_{1,k}^{T}}{n_1-1}{\bf H}_1 \\ & =\frac{1}{p}\sum_{j=2}^{n_1-1}\sum_{k\neq j}\frac{\left(\frac{1}{n_i-1}\widetilde{\bf y}_{1,j}^{T}{\bf H}_{1,j,k}\widetilde{\bf y}_{1,k}\right)^2}{(1+\frac{\gamma}{n_1-1}\widetilde{\bf y}_{1,j}^{T}{\bf H}_{1,j}\widetilde{\bf y}_{1,j})^2(1+\frac{\gamma}{n_i}\widetilde{\bf y}_{1,k}{\bf H}_{1,j,k}\widetilde{\bf y}_{1,k})^2}
	\end{align*}
	where $${\bf H}_{1,j}=\left({\bf I}_p+\gamma\widehat{\boldsymbol{\Sigma}}_i-\frac{\gamma}{n_1-1}\widetilde{\bf y}_{1,j}\widetilde{\bf y}_{1,j}^{T}\right)^{-1}$$
	and $${\bf H}_{1,j,k}=\left({\bf I}_p+\gamma\widehat{\boldsymbol{\Sigma}}_i-\frac{\gamma}{n_1-1}\widetilde{\bf y}_{1,j}\widetilde{\bf y}_{1,j}^{T}-\frac{\gamma}{n_1-1}\widetilde{\bf y}_{1,k}\widetilde{\bf y}_{1,k}^{T}\right)^{-1}.$$
	We now refer to the use of the trace lemma to 
	replace the denominator by its deterministic equivalents, thus we get
	\begin{align*}
	&\frac{1}{p}\sum_{j=2}^{n_1-1}\sum_{k\neq j}\tr \frac{\widetilde{\bf y}_{1,j}\widetilde{\bf y}_{1,j}^{T}}{n_1-1}{\bf H}_1\frac{\widetilde{\bf y}_{1,k}\widetilde{\bf y}_{1,k}^{T}}{n_1-1}{\bf H}_1 \\ & =\frac{n_i}{p}\frac{\left(\frac{1}{n_1-1}\tr \boldsymbol{\Sigma}{\bf H}_1\boldsymbol{\Sigma}{\bf H}_1\right)}{(1+\frac{\gamma}{n_1-1}\tr \boldsymbol{\Sigma}_1{\bf H}_1)^4} + o(1).
	\end{align*}
	Using similar steps, the second term can be approximated as follows
	$$ \frac{1}{p}\sum_{j=1}^{n_1-1}\tr\frac{\widetilde{\bf y}_{1,j}\widetilde{\bf y}_{1,j}^{T}}{n_1-1}{\bf H}_1\frac{\widetilde{\bf y}_{1,j}\widetilde{\bf y}_{1,j}^{T}}{n_1-1}{\bf H}_1 = 
	\frac{n_i}{p}\frac{\left(\frac{1}{n_i-1}\tr \boldsymbol{\Sigma}_1{\bf H}_1\right)^2}{(1+\frac{\gamma}{n_1-1}\tr \boldsymbol{\Sigma}_1{\bf H}_1)^2} + o(1).
	$$
	We thus obtain
	\begin{align*}
	\frac{1}{p}\tr \widehat{\boldsymbol{\Sigma}}_1{\bf H}_1\widehat{\boldsymbol{\Sigma}}_1{\bf H}_1& =\frac{n_i}{p}\frac{\left(\frac{1}{n_i-1}\tr \boldsymbol{\Sigma}_1{\bf H}_1\boldsymbol{\Sigma}_1{\bf H}_1\right)}{(1+\frac{\gamma}{n_i-1}\tr \boldsymbol{\Sigma}_1{\bf H}_1)^4}
	 \\ & +\frac{n_i}{p}\frac{\left(\frac{1}{n_i-1}\tr \boldsymbol{\Sigma}_1{\bf H}_1\right)^2}{(1+\frac{\gamma}{n_1-1}\tr \boldsymbol{\Sigma}_1{\bf H}_1)^2}+o(1).
	\end{align*}
	We will now handle the term $\frac{1}{p}\tr {\boldsymbol{\Sigma}_1}{\bf H}_0\boldsymbol{\Sigma}_1{\bf H}_0$. Again, we start by replacing $\boldsymbol{\Sigma}_1$ by $\widehat{\boldsymbol{\Sigma}}_1$. In doing so, we obtain: 
	\begin{align*}
	&\frac{1}{p}\tr \widehat{\boldsymbol{\Sigma}}_1{\bf H}_0\widehat{\boldsymbol{\Sigma}}_1{\bf H}_0=\frac{1}{p}\sum_{j=2}^{n_1-1}\sum_{k=2}^{n_1-1} \frac{\widetilde{\bf y}_{1,j}\widetilde{\bf y}_{1,j}^{T}}{n_1-1}{\bf H}_0\frac{\widetilde{\bf y}_{1,k}\widetilde{\bf y}_{1,k}}{n_1-1}{\bf H}_0\\
	&=\frac{1}{p}\sum_{j=2}^{n_1-1}\left(\frac{\widetilde{\bf y}_{1,j}^{T}{\bf H}_0\widetilde{\bf y}_{1,j}}{n_1-1}\right)^2+\frac{1}{p}\sum_{j=2}\sum_{k\neq j} \frac{\widetilde{\bf y}_{1,j}\widetilde{\bf y}_{1,j}^{T}}{n_1-1}{\bf H}_0\frac{\widetilde{\bf y}_{1,k}\widetilde{\bf y}_{1,k}}{n_1-1}{\bf H}_0\\
	&=\frac{n_i}{p}\left(\frac{1}{n_i-1}\tr \boldsymbol{\Sigma}_1{\bf H}_0\right)^2 +\frac{1}{p}\tr \boldsymbol{\Sigma}_1{\bf H}_0\boldsymbol{\Sigma}_1{\bf H}_0+o(1)
	\end{align*}
	It remains now to handle the term $\frac{1}{p}\tr \boldsymbol{\Sigma}_1{\bf H}_1\boldsymbol{\Sigma}_1{\bf H}_0$. Using the same reasoning, we have:
	\begin{align*}
	&\frac{1}{p}\tr \widehat{\boldsymbol{\Sigma}}_1{\bf H}_1\widehat{\boldsymbol{\Sigma}}_1{\bf H}_0=\frac{1}{p}\sum_{j=2}^{n_1-1}\sum_{k=2}^{n_1-1}\frac{\widetilde{\bf y}_{1,j}\widetilde{\bf y}_{1,j}^{T}}{n_1-1}{\bf H}_1\frac{\widetilde{\bf y}_{1,k}\widetilde{\bf y}_{1,k}^{T}}{n_1-1}{\bf H}_0\\
	&=\frac{1}{p}\sum_{j=2}^{n_1-1}\frac{1}{(n_i-1)^2}\widetilde{\bf y}_{1,j}^{T}{\bf H}_1\widetilde{\bf y}_{1,j}\widetilde{\bf y}_{1,j}^{T}{\bf H}_0\widetilde{\bf y}_{1,j}  \\ & +\frac{1}{p}\sum_{j=2}^{n_1-1}\sum_{k\neq j} \frac{1}{(n_1-1)^2}\tr \widetilde{\bf y}_{1,j}\widetilde{\bf y}_{1,j}^{T}{\bf H}_1\widetilde{\bf y}_{1,k}\widetilde{\bf y}_{1,k}^{T}{\bf H}_0
	\end{align*}
	Using the inversion Lemma along with the trace Lemma, we ultimately find:
	\begin{align*}
	\frac{1}{p}\tr \widehat{\boldsymbol{\Sigma}}_1{\bf H}_1\widehat{\boldsymbol{\Sigma}}_1{\bf H}_0 & =\frac{n_i}{p}\frac{1}{n_i}\tr \boldsymbol{\Sigma}_1{\bf H}_0 \frac{\frac{1}{n_i-1}\tr \boldsymbol{\Sigma}_1{\bf H}_1}{1+\frac{\gamma}{n_1-1}\tr \boldsymbol{\Sigma}_1{\bf H}_1} \\ & +\frac{\frac{1}{p}\tr \boldsymbol{\Sigma}_1{\bf H}_1\boldsymbol{\Sigma}_1{\bf H}_0}{(1+\frac{\gamma}{n_1-1}\tr \boldsymbol{\Sigma}_1{\bf H}_1)^2}+o(1).
	\end{align*}
	Now, we will put things together. 
	We have the following
	\begin{align*}
	& \frac{1}{p}\tr \boldsymbol{\Sigma}_1{\bf H}_1\boldsymbol{\Sigma}_1{\bf H}_1  =(1+\frac{\gamma}{n_1-1}\tr \boldsymbol{\Sigma}_1{\bf H}_1)^4\frac{1}{p}\tr \widehat{\boldsymbol{\Sigma}}_1{\bf H}_1\widehat{\boldsymbol{\Sigma}}_1{\bf H}_1 \\ &-\frac{n_1}{p}\left(\frac{1}{n_1}\tr \boldsymbol{\Sigma}_1{\bf H}_1\right)^2\left(1+\frac{\gamma}{n_1-1}\tr \boldsymbol{\Sigma}_1{\bf H}_1\right)^2+o(1).
	\end{align*}
	\begin{align*}
	\frac{1}{p}\tr \boldsymbol{\Sigma}_1{\bf H}_0\boldsymbol{\Sigma}_1{\bf H}_0 & =\frac{1}{p}\tr\widehat{\boldsymbol{\Sigma}}_1{\bf H}_0\widehat{\boldsymbol{\Sigma}}_1{\bf H}_0-\frac{n_1}{p}\left(\frac{1}{n_1}\tr \boldsymbol{\Sigma}_1{\bf H}_0\right)^2 \\ & +o(1).
	\end{align*}
	and
	\begin{align*}
	& \frac{1}{p}\tr \boldsymbol{\Sigma}_1{\bf H}_1\boldsymbol{\Sigma}_1{\bf H}_0 =(1+\frac{\gamma}{n_1-1}\tr\boldsymbol{\Sigma}_1{\bf H}_1)^2\frac{1}{p}\widehat{\boldsymbol{\Sigma}}_1{\bf H}_1\widehat{\boldsymbol{\Sigma}}_1{\bf H}_0 \\ &-\frac{n_1}{p}\frac{1}{n_1}\tr \boldsymbol{\Sigma}_1{\bf H}_0\frac{1}{n_1}\boldsymbol{\Sigma}_1{\bf H}_1(1+\frac{\gamma}{n_1-1}\tr \boldsymbol{\Sigma}_1{\bf H}_1)+o(1).
	\end{align*}
	A consistent estimator of $\frac{1}{p}\tr {\bf B}_1^2$ is thus given by
	\begin{align*}
	&	\frac{1}{p}\tr {\bf B}_1^2=\left(1+\frac{\gamma}{n_1-1}\tr \boldsymbol{\Sigma}_1{\bf H}_1\right)^4\frac{1}{p}\tr \widehat{\boldsymbol{\Sigma}}_1{\bf H}_1\widehat{\boldsymbol{\Sigma}}_1{\bf H}_1 \\ & -\frac{n_1}{p}\left(\frac{1}{n_1}\tr \boldsymbol{\Sigma}_1{\bf H}_1\right)^2\left(1+\frac{\gamma}{n_1-1}\tr \boldsymbol{\Sigma}_1{\bf H}_1\right)^2\\
	& +\frac{1}{p}\tr \widehat{\boldsymbol{\Sigma}}_1{\bf H}_0\widehat{\boldsymbol{\Sigma}}_1{\bf H}_0-\frac{n_1}{p}\left(\frac{1}{n_1}\tr \boldsymbol{\Sigma}_1{\bf H}_0\right)^2  \\ & -2\left(1+\frac{\gamma}{n_1-1}\tr \boldsymbol{\Sigma}_1{\bf H}_1\right)^2\frac{1}{p}\tr \widehat{\boldsymbol{\Sigma}}_1{\bf H}_1\widehat{\boldsymbol{\Sigma}}_1{\bf H}_0\\ &+\frac{2n_1}{p}\frac{1}{n_1}\tr \boldsymbol{\Sigma}_1{\bf H}_0\frac{1}{n_1}\tr \boldsymbol{\Sigma}_1{\bf H}_1\left(1+\frac{\gamma}{n_1-1}\tr \boldsymbol{\Sigma}_1{\bf H}_1\right).
	%(1+\frac{\gamma}{n_1}\tr \boldsymbol{\Sigma}_1{\bf H}_1)^2\frac{1}{p}\widehat{\boldsymbol{\Sigma}_1}{\bf H}_1\widehat{\boldsymbol{\Sigma}}_1{\bf H}_1-2\left(1+\frac{\gamma}{n_1}\tr \boldsymbol{\Sigma}_1{\bf H}_1\right)\frac{1}{p}\tr \widehat{\boldsymbol{\Sigma}}_1{\bf H}_1\widehat{\boldsymbol{\Sigma}}_1{\bf H}_0+\frac{1}{p}\tr \widehat{\boldsymbol{\Sigma}}_1{\bf H}_0\widehat{\boldsymbol{\Sigma}}_1{\bf H}_0\\
	%&-\frac{n_1}{p}\left(\frac{1}{n_1}\tr \boldsymbol{\Sigma}_1{\bf H}_1\right)^2+\frac{n_1}{p}\frac{1}{n_1}\tr \boldsymbol{\Sigma}_1{\bf H}_0\frac{1}{n_i}\tr \boldsymbol{\Sigma}_1{\bf H}_1
	\end{align*}
	We replace $\frac{1}{n_1}\tr \boldsymbol{\Sigma}_1{\bf H}_1$ and $\frac{1}{n_1}\tr \boldsymbol{\Sigma}_1{\bf H}_0$ by their respective consistent estimates $\widehat{\delta}_1$ and $\frac{1}{n_1} \tr \widehat{\mathbf{\Sigma}}_1 \mathbf{H}_0$ to get the consistent estimate for $\frac{1}{p}\tr {\bf B}_1^2$. 
	By this, we achieve the proof of the theorem.
\end{appendices}
%
%	% Note: in this sample, the section number is hard-coded in. Following
%	% proper LaTeX conventions, it should properly be coded as a reference:
%
%	%In this appendix we prove the following theorem from
%	%Section~\ref{sec:textree-generalization}:
%
%	In this appendix we prove the following theorem from
%	Section~6.2:
%
%	\noindent
%	{\bf Theorem} {\it Let $u,v,w$ be discrete variables such that $v, w$ do
%		not co-occur with $u$ (i.e., $u\neq0\;\Rightarrow \;v=w=0$ in a given
%		dataset $\dataset$). Let $N_{v0},N_{w0}$ be the number of data points for
%		which $v=0, w=0$ respectively, and let $I_{uv},I_{uw}$ be the
%		respective empirical mutual information values based on the sample
%		$\dataset$. Then
%		\[
%		N_{v0} \;>\; N_{w0}\;\;\Rightarrow\;\;I_{uv} \;\leq\;I_{uw}
%		\]
%		with equality only if $u$ is identically 0.} \hfill\BlackBox
%
%	\noindent
%	{\bf Proof}. We use the notation:
%	\[
%	P_v(i) \;=\;\frac{N_v^i}{N},\;\;\;i \neq 0;\;\;\;
%	P_{v0}\;\equiv\;P_v(0)\; = \;1 - \sum_{i\neq 0}P_v(i).
%	\]
%	These values represent the (empirical) probabilities of $v$
%	taking value $i\neq 0$ and 0 respectively.  Entropies will be denoted
%	by $H$. We aim to show that $\fracpartial{I_{uv}}{P_{v0}} < 0$....\\
%
%	{\noindent \em Remainder omitted in this sample. See http://www.jmlr.org/papers/ for full paper.}

%\vskip 0.2in
%\bibliography{References.bib}

\end{document}